\documentclass[11pt]{article}
\usepackage{amsmath, amssymb, amsthm, soul}
\usepackage[numbers, square]{natbib}
\usepackage[left=1in, bottom=1in, right=1in, top=1in]{geometry}
\usepackage[citecolor=blue, urlcolor=blue]{hyperref}
\usepackage{dsfont}
\usepackage{xcolor}
\usepackage{natbib}
\usepackage{multicol}
\usepackage{multirow}
\usepackage{graphicx}
\usepackage{subcaption}
\usepackage{caption}
\usepackage{comment}
\usepackage{wrapfig}
\usepackage{mathtools}
\usepackage{booktabs}
\usepackage{enumitem}
\usepackage{bbm}
\DeclarePairedDelimiterX{\norm}[1]{\lVert}{\rVert}{#1}
\usepackage{subcaption}

\definecolor{applegreen}{rgb}{0.55, 0.71, 0.0}

\usepackage{wrapfig} 
\usepackage[font=small,skip=0pt]{caption}
\newtheorem{theorem}{Theorem}
\newtheorem{lemma}{Lemma}
\newtheorem{corollary}{Corollary}

\newtheorem{proposition}{Proposition}
\usepackage{authblk}

\title{Uncertainty Quantification of Next Generation Reservoir Computing with Applications to Memory-Driven Dynamical Systems}

\author[1]{Livia Popa\thanks{Corresponding author: \texttt{lp472@cornell.edu}}}
\author[1]{Sumanta Basu}
\author[1]{Martin T. Wells}

\affil[1]{Department of Statistics and Data Science, Cornell University, New York, USA}

\date{ }

\begin{document}

\maketitle

\begin{abstract}
Nonlinear dynamical systems with memory arise across science and engineering, yet uncertainty quantification for efficient forecasting methods such as Next Generation Reservoir Computing (NGRC) remains underdeveloped. We study Bayesian ridge and conformal prediction intervals for NGRC and characterize when their uncertainty estimates agree or differ. In low dimensions, their asymptotic widths are governed by different summaries of the residual distribution, so agreement depends on residual shape rather than dimensionality alone. In high dimensions, regularization introduces a further tradeoff between estimation variance, shrinkage bias, and posterior uncertainty, leading to an explicit transition between regimes where Bayesian intervals are wider or narrower than conformal intervals. We extend these results to quadratic NGRC feature maps and give sufficient conditions for transferring the analysis to temporally dependent forecast windows. Simulations and real-data experiments support the theoretical predictions and illustrate how residual distribution, regularization, dimensionality, and distribution shift affect interval calibration and efficiency. These results provide a principled framework for choosing and interpreting uncertainty quantification methods in reservoir-based forecasting.

\end{abstract}

\newpage

\section{Introduction}\label{sec:intro}

Forecasting uncertainty for nonlinear systems with memory is important in many scientific domains, including neuroscience \cite{Aram2015IntracorticalConnectivity}, epidemiology \cite{Pakkanen2023IncidencePrevalence}, materials engineering \cite{Patlashenko2001VolterraTimestepping}, and finance \cite{HanWong2021VolterraHeston}. In these settings, future behavior can depend on nonlinear interactions among past states and accumulated history. Accurate point forecasts alone are therefore insufficient when predictions are used for scientific interpretation, risk assessment, or decision-making; reliable predictive uncertainty is also needed \cite{GneitingBalabdaouiRaftery2007}.

Next generation reservoir computing (NGRC) provides a computationally efficient framework for forecasting nonlinear dynamical systems with memory. NGRC constructs a deterministic delay-coordinate representation of the observed process, augments it with nonlinear features, and estimates only a linear readout \cite{nextgeneration,GrigoryevaTingOrtega2025}. This structure is particularly convenient for uncertainty quantification because uncertainty can be introduced at the readout layer while the nonlinear feature map remains fixed. At the same time, increasing the number of delays and nonlinear interactions can cause the feature dimension to grow rapidly relative to the available sample size, while polynomial feature construction induces substantial dependence among coordinates. Thus, even though only the readout is trained, NGRC can naturally enter a proportional high-dimensional regime in which standard low-dimensional uncertainty approximations need not apply. Recent work has developed conformal methods for reservoir-based forecasting, including Reservoir Conformal Prediction (ResCP) \cite{Neglia2026ResCP}, but the relationship between Bayesian and conformal uncertainty quantification for NGRC remains incompletely understood. In particular, it is unclear when the two procedures should produce comparable intervals, what mechanisms determine differences in their widths and coverage, and how regularization and increasing feature dimension alter this comparison.

We study Bayesian ridge and split-conformal prediction intervals constructed from a common NGRC readout. First, under stationary ergodic sampling and fixed feature dimension, we show that Bayesian and conformal widths converge to different functionals of the limiting residual distribution. The Bayesian width is determined by the root-mean-square residual scale, whereas the conformal width is determined by an absolute-residual quantile. This formulation allows conditional-mean misspecification and singular feature covariance, showing that fixed-dimensional agreement is a residual-distributional matching property rather than an automatic consequence of low dimensionality.

Second, in proportional dimension, we derive a bias--variance--posterior decomposition separating Bayesian posterior uncertainty from frequentist estimation variance and ridge prediction bias. For isotropic Gaussian designs, the decomposition yields explicit signal-, regularization-, noise-, and aspect-ratio-dependent limits for interval width and Bayesian marginal coverage. In particular, the relative width of Bayesian and conformal intervals can change across a signal-dependent phase boundary, so a narrower Bayesian interval may reflect undercoverage rather than greater efficiency at equal coverage. These results connect NGRC uncertainty quantification with high-dimensional ridge regression and random-matrix theory \cite{DobribanWager2018,MarchenkoPastur1967}.

Third, we extend the high-dimensional comparison to an explicit quadratic NGRC feature model whose polynomial coordinates are dependent. Although the centered quadratic feature representation is isotropic, its coordinates are not independent, so the Gaussian-design argument does not apply directly. We show that quadratic-form concentration controls the spectral behavior of the dependent feature matrix, while sufficient diffuseness of the quadratic prediction direction permits a Gaussian approximation for scalar prediction errors. Together, these properties recover the finite-trace Bayesian and conformal width comparison for an explicit nonlinear NGRC feature model. Random effects provide one sufficient mechanism for diffuseness, while a deterministic-readout extension identifies diffuseness of the ridge-transformed quadratic bias direction as the relevant condition. This condition is a population-level theoretical criterion rather than a directly observable diagnostic for a fitted forecasting model. We therefore interpret the quadratic-feature theory as an analytical benchmark for NGRC uncertainty rather than a complete calibration theory for deployed NGRC systems \cite{Yaskov2016}.

We also study the effect of temporal dependence. The main high-dimensional calculations use independent fitting, calibration, and test feature vectors. We give a sufficient coupling argument under absolute regularity showing that probability statements from this independent-window benchmark can be transferred to sufficiently separated NGRC forecast windows \cite{ChernozhukovWuthrichZhu2018,xu2023conformal}. The separation requirement applies to the retained forecast windows themselves, not only to the boundaries between fitting, calibration, and test blocks, so the result should be viewed as a theoretical bridge for thinned, sufficiently separated forecast origins rather than as a direct approximation for ordinary overlapping NGRC windows.

We conduct simulation experiments to test these mechanisms directly. The experiments examine fixed-dimensional residual-shape effects, the signal-dependent high-dimensional phase transition, quadratic-feature diffuseness and temporal transfer, and robustness under an observation-noise variance change. An additional controlled-data-budget study examines the fitting--calibration tradeoff created by reserving observations for conformal calibration. Real-data experiments then compare Bayesian ridge, split conformal prediction, adaptive conformal prediction, time-weighted conformal prediction, and ResCP using a common NGRC forecasting model across environmental, energy, air-quality, and financial time series. We additionally compare NGRC, ARIMA, recurrent neural networks, and Transformer forecasts under a common ResCP uncertainty procedure to examine how point-forecasting architecture, residual behavior, calibration, and computational cost interact.
\section{Background}\label{sec:background}

We review the literature most closely related to our analysis: next generation reservoir computing, Bayesian and conformal uncertainty quantification for linear readouts, high-dimensional ridge regression and random-matrix theory, and conformal prediction for temporally dependent data.

\subsection{Next Generation Reservoir Computing}

Reservoir computing (RC) is a recurrent-learning framework in which a fixed dynamical representation is combined with a trained linear readout. Recent work has established a close connection between reservoir computing and nonlinear vector autoregressive (NVAR) models \cite{Bollt_2021,nextgeneration}. Next generation reservoir computing (NGRC) exploits this connection by replacing the recurrent reservoir with an explicit delay-coordinate feature map constructed from lagged observations and nonlinear transformations. Only the output readout is estimated, yielding a computationally efficient forecasting framework for nonlinear dynamical systems with memory \cite{nextgeneration}.

The explicit NGRC representation is also convenient for statistical analysis. Once the nonlinear features have been constructed, estimation reduces to a regularized linear regression problem. The feature dimension can nevertheless grow rapidly with the number of delays and the degree of the nonlinear expansion, while polynomial transformations of a common delay vector induce dependence among feature coordinates. Recent work has further studied NGRC through high-dimensional and kernel representations, including settings in which large polynomial feature spaces are handled implicitly \cite{GrigoryevaTingOrtega2025}. In this paper, we work with the finite-dimensional delay-coordinate representation defined in Section~\ref{sec:method} and study uncertainty quantification for its linear readout.

\subsection{Bayesian and Conformal Uncertainty Quantification}
\label{subsec:uq_background}

Because the NGRC feature map is fixed before fitting the readout, uncertainty quantification can be introduced directly at the output layer. Bayesian ridge regression provides a probabilistic treatment of the readout using a Gaussian likelihood and shrinkage prior \cite{mackay1992bayesian,tipping2001sparse}. Conformal prediction provides a complementary framework in which prediction sets are calibrated using empirical conformity scores rather than a fully specified parametric model for the forecast-error distribution \cite{angelopoulos2023conformal,shafer2008tutorial}. Our primary conformal benchmark is split conformal prediction, which estimates the forecasting model on a fitting sample and uses a separate calibration sample to determine the prediction interval \cite{lei2018distribution,papadopoulos2002inductive}.

Standard conformal prediction provides marginal coverage under exchangeability, which does not in general imply validity conditional on the covariates. Recent work has therefore considered weaker forms of conditional uncertainty quantification based on coarsened conditioning events or lower-dimensional summaries. For example, \citet{YangEtAl2026PICPI} introduce prediction-interval-conditional prediction intervals (PICPIs), which condition on the fitted prediction itself and provide locally calibrated uncertainty statements over prediction strata. Their inferential target differs from the response prediction intervals studied here, but the framework provides a complementary perspective on the distinction between marginal and locally conditional uncertainty quantification.

The relationship between Bayesian and conformal ridge procedures has also been studied previously. In particular, \citet{BurnaevVovk2014} analyze the efficiency of full conformalized ridge regression relative to Bayesian ridge prediction under a standard probabilistic model. Our setting differs in two important respects. First, we compare a Gaussian readout interval with split conformal prediction built around a common ridge predictor rather than full conformalized ridge regression. Second, our fixed-dimensional analysis characterizes the limiting discrepancy through the population prediction-error distribution without requiring the conditional mean to be exactly linear in the NGRC features or the population feature covariance to be nonsingular. Thus, the fixed-dimensional result identifies when Bayesian and split-conformal NGRC intervals agree under possible readout misspecification, rather than treating asymptotic agreement under the standard Bayesian model as a new phenomenon.

\subsection{High-Dimensional Ridge Regression and Nonlinear Features}

NGRC feature maps can become high dimensional because increasing the number of delays and adding polynomial interactions rapidly expands the readout dimension. When the feature dimension is non-negligible relative to sample size, classical fixed-dimensional approximations can fail. High-dimensional ridge regression has been studied extensively under proportional asymptotics, where $p/n\rightarrow\gamma>0$ and prediction risk admits deterministic limits derived using random-matrix theory \cite{DobribanWager2018,MarchenkoPastur1967}. In this regime, estimation variance and regularization bias can both remain nonvanishing.

Our high-dimensional analysis builds on this literature but focuses on the uncertainty represented by Bayesian and conformal procedures constructed around the same ridge predictor. A key distinction is that Bayesian predictive variance contains observation noise and posterior parameter uncertainty, whereas the prediction errors used for conformal calibration contain observation noise, frequentist estimation variance, and ridge prediction bias. Consequently, there is no universal ordering between Bayesian and conformal interval widths in proportional dimension: the prediction-bias contribution must be included in the comparison.

High-dimensional conformal prediction has also received recent attention. 
Full conformal procedures in proportional regimes have been studied using high-dimensional asymptotics and related computational methods
\cite{ClarteZdeborova2025,GibbsCandes2025HighDimensional}. 
Our focus is different: we study the discrepancy between Bayesian and split-conformal intervals sharing a common ridge center, and identify how posterior uncertainty, estimation variance, and ridge bias determine their relative widths and Bayesian marginal calibration.

A further complication for NGRC is that polynomial features are not composed of independent coordinates. Quadratic expansions contain squared lag variables and cross-products formed from the same underlying state vector, producing substantial dependence among feature coordinates. Whitening may make such coordinates uncorrelated, but it does not make them independent. Random-matrix limits can nevertheless hold for isotropic random vectors with dependent coordinates under suitable quadratic-form concentration conditions \cite{Yaskov2016}. This motivates our analysis of an explicit quadratic NGRC feature model, for which the required spectral concentration can be established directly. The resulting analysis also shows that high dimensionality alone is not sufficient for Gaussian prediction-error approximations: the relevant quadratic prediction direction must be sufficiently diffuse.

\subsection{Conformal Prediction for Temporally Dependent Data}

The standard finite-sample split-conformal guarantee relies on exchangeability, an assumption that is generally violated in time-series forecasting. A growing literature therefore studies conformal prediction under temporal dependence, distribution shift, and sequential sampling \cite{ChernozhukovWuthrichZhu2018,xu2023conformal,barber2023conformal}. For example, \citet{ChernozhukovWuthrichZhu2018} develop conformal inference methods for dependent data using dependence-aware permutations, while \citet{xu2023conformal} study sequential predictive conformal inference for time series.

Adaptive and weighted procedures provide additional mechanisms for responding to temporal variation. Adaptive conformal prediction updates the effective miscoverage level according to recent coverage errors \cite{gibbs2021adaptive,zaffran2022adaptive}, while weighted conformal procedures place greater emphasis on calibration observations that are more relevant to the current prediction problem \cite{barber2023conformal}. In our empirical analysis, we compare ordinary split conformal prediction with adaptive and time-weighted variants to assess robustness when forecast-error behavior changes over time. These procedures are empirical comparators; the main theoretical Bayesian--conformal comparison concerns Bayesian ridge and ordinary split conformal prediction constructed from a common NGRC readout.

Recent work has also developed conformal uncertainty quantification specifically for reservoir-based forecasting. Reservoir Conformal Prediction (ResCP) \cite{Neglia2026ResCP} represents residual dynamics in a reservoir state space and uses state-dependent weighting of past conformity scores to adapt prediction intervals to local temporal behavior. Rather than introducing a new conformity-score weighting mechanism, we study uncertainty quantification for the NGRC forecasting model itself and characterize how Bayesian and split-conformal intervals compare across fixed- and proportional-dimensional regimes. ResCP is used as an external adaptive benchmark in our simulation and real-data studies.

Our theoretical treatment of temporal dependence is complementary to these adaptive conformal approaches. The high-dimensional width calculations are first developed under independent fitting, calibration, and test feature vectors. We then use absolute-regularity arguments to give sufficient conditions under which probability statements from this independent-window benchmark transfer to sufficiently separated forecast origins of a dependent process. This result concerns dependence across retained forecasting windows and should not be interpreted as a finite-sample conformal guarantee for ordinary overlapping time-series windows.
\section{Method}\label{sec:method}

We consider direct forecasting of a scalar time series using an NGRC feature
representation followed by a regularized linear readout. Bayesian and conformal
uncertainty quantification are then applied to the same fitted readout. Throughout
this section, we fix a forecast horizon $H$ and suppress the horizon subscript
when no ambiguity arises.

\subsection{NGRC Forecasting}

Let $\{x_t\}$ denote a uniformly sampled scalar time series and define the
$H$-step-ahead response as $y_t=x_{t+H}$. For a delay length $k$ and lag spacing
$s$, define the delay vector
\[
z_t=(x_t,x_{t-s},\ldots,x_{t-(k-1)s})^\top\in\mathbb R^k.
\]
We use the standard quadratic NGRC feature map
\begin{equation}
r_t=
\left(
1,\,
z_t^\top,\,
\{z_{t,i}z_{t,j}:i\le j\}
\right)^\top
\in\mathbb R^p,
\qquad
p=1+k+\frac{k(k+1)}{2}.
\label{eq:ngrc_features}
\end{equation}
Thus, the representation contains an intercept, the lagged observations, and
all unique quadratic interactions among them. More general feature maps can be
accommodated by the same readout and uncertainty-quantification procedures.

Let $\mathcal I_{\mathrm{fit}}$ denote the fit indices, with
$n=|\mathcal I_{\mathrm{fit}}|$, and let $R\in\mathbb R^{n\times p}$ and
$y\in\mathbb R^n$ denote the corresponding design matrix and response vector.
We estimate the readout by ridge regression,
\begin{equation}
\widehat w
=
\arg\min_{w\in\mathbb R^p}
\left\{
\frac{1}{n}\|y-Rw\|_2^2+\lambda\|w\|_2^2
\right\}
=
\left(\frac{R^\top R}{n}+\lambda I_p\right)^{-1}
\frac{R^\top y}{n}.
\label{eq:ridge_solution}
\end{equation}
For a new feature vector $r_*$, the forecast point is
$\widehat y_*=r_*^\top\widehat w$.

We treat ridge regression as a prediction rule and do not require the
conditional mean of $y_t$ to be exactly linear in $r_t$. Stronger probabilistic
assumptions are introduced only where required by the uncertainty model or the
theoretical results in Section~\ref{sec:theory}.

\subsection{Uncertainty Quantification at the Readout Layer}
\label{subsec:uq_methods}

We compare Bayesian ridge and split conformal prediction using the same NGRC
feature map and fitted ridge readout, so differences in interval behavior are
not confounded with differences in the point predictor. Time-weighted conformal
prediction, adaptive conformal prediction, and Reservoir Conformal Prediction
(ResCP) are additionally considered as empirical benchmarks for temporally
dependent and time-varying forecast errors. The main theoretical comparison in
Section~\ref{sec:theory} concerns Bayesian ridge and ordinary split conformal
prediction.

\subsubsection{Bayesian Ridge}
\label{subsec:bayes_ridge}

For Bayesian uncertainty quantification, we use the Gaussian readout model
$y\mid R,w\sim\mathcal N(Rw,\sigma^2I_n)$ and prior
$w\sim\mathcal N(0,\sigma^2(n\lambda)^{-1}I_p)$. This scaling makes the posterior
mean equal to the ridge estimator in \eqref{eq:ridge_solution}, with posterior
covariance
\(
\Sigma_w
=
\sigma^2(R^\top R+n\lambda I_p)^{-1}.
\)
For a future feature vector $r_*$, define
$v_B(r_*)=\sigma^2+r_*^\top\Sigma_w r_*$. The $(1-\alpha)$ Gaussian predictive
interval is
\begin{equation}
C_B(r_*)
=
\left[
r_*^\top\widehat w-z_{1-\alpha/2}\sqrt{v_B(r_*)},
\;
r_*^\top\widehat w+z_{1-\alpha/2}\sqrt{v_B(r_*)}
\right].
\label{eq:bayes_interval}
\end{equation}

We refer to \eqref{eq:bayes_interval} as a Gaussian readout interval when the
conditional Gaussian model is not assumed to be correctly specified. In the
proportional high-dimensional theory, $\sigma$ is treated as known or replaced
by a separately justified consistent estimator.

\subsubsection{Split Conformal Prediction}

Let $\mathcal I_{\mathrm{cal}}$ denote a calibration set disjoint from the fit
set, with $m=|\mathcal I_{\mathrm{cal}}|$. Using the readout estimated only on
$\mathcal I_{\mathrm{fit}}$, define
$S_i=|y_i-r_i^\top\widehat w|$ for $i\in\mathcal I_{\mathrm{cal}}$, and let
$S_{(1)}\le\cdots\le S_{(m)}$ denote the ordered scores. With
$k_m=\lceil(m+1)(1-\alpha)\rceil$ and the convention $S_{(m+1)}=+\infty$, define
$\widehat q_{1-\alpha}=S_{(k_m)}$. The split-conformal interval is
\begin{equation}
C_C(r_*)
=
\left[
r_*^\top\widehat w-\widehat q_{1-\alpha},
\;
r_*^\top\widehat w+\widehat q_{1-\alpha}
\right].
\label{eq:split_interval}
\end{equation}
Under exchangeability, the corrected quantile gives the usual finite-sample
marginal coverage guarantee
\cite{lei2018distribution,papadopoulos2002inductive}. This guarantee does not
directly extend to general time-series data
\cite{ChernozhukovWuthrichZhu2018,xu2023conformal}.

\subsubsection{Time-Weighted Conformal Prediction}
\label{subsec:weighted_conformal}

To respond to changes in forecast-error behavior, let
$S_1,\ldots,S_{M_t}$ denote the conformity scores available before prediction
time $t$ and assign weights
$\omega_{j,t} = \frac{\rho^{M_t-j}}{\sum_{\ell=1}^{M_t}\rho^{M_t-\ell}}, \;\; 0<\rho\le1.$
The weighted empirical quantile is
\[
\widehat q^{\,w}_{t,1-\alpha}
=
\inf\left\{
q:
\sum_{j=1}^{M_t}\omega_{j,t}\mathbf 1\{S_j\le q\}
\ge1-\alpha
\right\},
\]
and the corresponding interval is
$C_{\mathrm{TW}}(r_t)=
[r_t^\top\widehat w-\widehat q^{\,w}_{t,1-\alpha},
 r_t^\top\widehat w+\widehat q^{\,w}_{t,1-\alpha}]$.
When $\rho=1$, all available scores receive equal weight, while $\rho<1$
emphasizes more recent forecast errors \cite{barber2023conformal}. We evaluate
this procedure empirically rather than include it in the main width theory.

\subsubsection{Adaptive Conformal Prediction}
\label{subsec:adaptive_conformal}

Adaptive conformal prediction updates the effective miscoverage level in
response to recent coverage errors
\cite{gibbs2021adaptive,zaffran2022adaptive}. Let $\alpha_t$ denote the level
at time $t$ and
$\mathrm{err}_t=\mathbf 1\{y_t\notin C_t(1-\alpha_t)\}$. The update is
\begin{equation}
\alpha_{t+1}
=
\Pi_{[0,1]}
\left[
\alpha_t+\eta(\alpha-\mathrm{err}_t)
\right],
\label{eq:adaptive_update}
\end{equation}
where $\eta>0$ is a learning rate. At prediction time $t$, the conformal
quantile is computed from the currently available residual history using level
$1-\alpha_t$. Once $y_t$ becomes observable, its residual is added to the
history and the miscoverage level is updated according to
\eqref{eq:adaptive_update}. Thus, no current or future response is used to
construct its own prediction interval. We include this procedure as an
empirical comparator for time-varying forecast errors.

\subsubsection{Reservoir Conformal Prediction}
\label{subsec:rescp}

Reservoir Conformal Prediction (ResCP) \cite{Neglia2026ResCP} provides an
adaptive conformal benchmark designed specifically for time-series forecasting.
ResCP maps the recent forecasting history into a reservoir state and uses
similarity between the current state and past reservoir states to construct
state-dependent weights on previously observed forecast residuals. Prediction
intervals are then formed from weighted residual quantiles, allowing the
uncertainty estimate to adapt to local temporal behavior.

In our experiments, ResCP is treated as an external
uncertainty-quantification procedure rather than as part of the Bayesian--conformal width theory developed below. The reservoir configuration is fixed in advance, while the sampling temperature and
residual-history policy are selected using only pre-test observations and are frozen before test evaluation. During testing, reservoir states and
residual histories are updated sequentially, and a forecast residual becomes available for calibration only after its corresponding outcome has been
observed. Detailed tuning ranges, history policies, and implementation choices are given in the experimental details.
\section{Theory}
\label{sec:theory}

We study how Bayesian and split-conformal NGRC intervals differ as the complexity
of the readout changes relative to sample size. Throughout this section, the
forecast horizon is fixed and suppressed from the notation. The Bayesian and
conformal procedures use the same ridge point predictor, so the comparison
isolates differences in uncertainty quantification rather than differences in
point prediction.

The assumptions differ across the regimes considered below. The
fixed-dimensional result permits readout misspecification and requires only
stationarity, ergodicity, and moment and quantile conditions. The proportional
high-dimensional results instead assume an exact linear readout with Gaussian
noise together with stronger conditions on the feature distribution. Thus, the
high-dimensional results should be interpreted as analytical benchmarks rather
than as misspecification-robust extensions of the fixed-dimensional theory.

The analysis proceeds from a general benchmark to the nonlinear feature
structure specific to NGRC. We first establish a fixed-dimensional comparison
without requiring the chosen feature representation to define an exact
conditional linear model. We then derive a finite-trace decomposition showing
that, in proportional dimension, the Bayesian--conformal discrepancy is
governed by posterior uncertainty, frequentist estimation variance, and ridge
prediction bias.

The main NGRC-specific result concerns quadratic polynomial features. Unlike an
ordinary Gaussian design, quadratic NGRC coordinates are not independent. For a
Gaussian delay vector $G$, coordinates such as $G_jG_k$ and
$(G_j^2-1)/\sqrt{2}$ are dependent even though the centered, normalized feature
vector is isotropic. Thus standard arguments based on independent feature
coordinates cannot be applied directly. We show that the quadratic feature map
nevertheless satisfies sufficient quadratic-form concentration for proportional
random-matrix limits and, under a diffuseness condition on the quadratic
prediction direction, admits a Gaussian approximation for forecast errors.
These properties recover the finite-trace Bayesian and conformal width formulas
despite the dependence among polynomial feature coordinates.

We then specialize the general comparison to the isotropic proportional regime,
where random-matrix limits yield an explicit signal-dependent phase boundary.
Finally, we state a coupling result showing when independent-window probability
statements can be transferred to sufficiently separated observations from a
temporally dependent process.

\paragraph{Coverage terminology.}
Throughout the asymptotic theory, ``coverage'' refers to marginal frequentist
coverage for a new test observation unless otherwise stated. A Bayesian
credible interval has an exact posterior-predictive interpretation under the
specified Bayesian probability model, but this does not in general imply exact
frequentist coverage for a fixed readout. Standard split conformal prediction
has a finite-sample marginal coverage guarantee under exchangeability. The
stationary-ergodic results below instead concern asymptotic marginal coverage,
and the dependent-window extension incurs an explicit coupling error. We do not
claim feature-conditional or sequential conditional coverage unless it is
stated explicitly.

\subsection{Fixed-Dimensional Bayesian--Conformal Comparison}
\label{subsec:fixed_theory}

Suppose $(r_t,y_t)_{t\in\mathbb Z}$ is strictly stationary and ergodic, with
$r_t\in\mathbb R^p$ and fixed $p$. Assume
\(
\mathbb E\|r_0\|^2+\mathbb E y_0^2<\infty,
\)
and define
\(
\Sigma=\mathbb E[r_0r_0^\top],
\;\;
w^\dagger=\Sigma^\dagger\mathbb E[r_0y_0],
\;\;
e_t=y_t-r_t^\top w^\dagger,
\;\;
\tau^2=\mathbb E[e_0^2]\in(0,\infty),
\)
where $\Sigma^\dagger$ denotes the Moore--Penrose inverse. Let
$\mathcal S=\operatorname{range}(\Sigma)$. The vector $w^\dagger$ is the
minimum-norm population least-squares projection; we do not assume that
$\mathbb E[y_t\mid r_t]=r_t^\top w^\dagger$ or that $\Sigma$ is nonsingular.
This projection formulation is natural for NGRC, where a finite delay and
polynomial feature map generally approximate rather than exactly specify the
underlying dynamics.

Also define
\(
S_n=\frac{1}{n}R_n^\top R_n,
\;\;
\widehat w_n
=
(S_n+\lambda_n I_p)^{-1}\frac{R_n^\top y_n}{n},
\)
where $\lambda_n>0$ and $\lambda_n\to0$. Define the fitted residual scale
\(
\widehat\tau_n^2
=
\frac{1}{n}
\sum_{t=1}^n
(y_t-r_t^\top\widehat w_n)^2.
\)

\begin{lemma}[Projection and residual-scale consistency]
\label{lem:projection_consistency}
Under the preceding conditions,
\(
\widehat w_n\longrightarrow w^\dagger
\;\; \text{and} \;\;
\widehat\tau_n^2\longrightarrow\tau^2
\;\; \text{almost surely}.
\)
The result does not require $\Sigma$ to be invertible.
\end{lemma}

The proof is given in Appendix~\ref{app:fixed_theory}. The use of the population
projection rather than a correctly specified conditional linear model is
important for NGRC, where a finite delay and polynomial feature map should
generally be viewed as an approximation to the underlying dynamics.

Let $z=z_{1-\alpha/2}$, and let $q_e$ denote the $(1-\alpha)$ quantile of
$|e_0|$. Assume that the distribution function of $|e_0|$ is continuous in a
neighborhood of $q_e$ and strictly crosses $1-\alpha$ there. For calibration,
let
\(
\mathcal I_n=\{a_n+1,\ldots,a_n+m_n\}
\)
be a deterministic contiguous block disjoint from the observations used to fit
$\widehat w_n$, where $a_n$ is deterministic and $m_n\to\infty$.

For the Gaussian readout construction, define
\[
W_{B,n}(r)
=
2z\widehat\tau_n
\sqrt{
1+
\frac{1}{n}
r^\top(S_n+\lambda_nI_p)^{-1}r
}.
\]
Let $W_{C,n}$ denote the width of the ordinary split-conformal interval based
on the absolute calibration residuals
$|y_i-r_i^\top\widehat w_n|$, $i\in\mathcal I_n$.

\begin{theorem}[Fixed-dimensional Bayesian--conformal width discrepancy]
\label{thm:fixed_width_discrepancy}
For every fixed $r\in\mathcal S$,
\(
W_{B,n}(r)
\xrightarrow{\mathbb P}
2z\tau,
\;\;
W_{C,n}
\xrightarrow{\mathbb P}
2q_e.
\)
Consequently,
\begin{equation}
W_{B,n}(r)-W_{C,n}
\xrightarrow{\mathbb P}
2(z\tau-q_e).
\label{eq:fixed_width_difference}
\end{equation}
In particular, the two widths are asymptotically equivalent if and only if
\(
q_e=z\tau.
\)
\end{theorem}

The proof is given in Appendix~\ref{app:fixed_theory}.
Equation~\eqref{eq:fixed_width_difference} clarifies that low dimensionality
alone does not imply Bayesian--conformal equivalence. The Gaussian readout
interval is determined asymptotically by the residual root-mean-square scale,
whereas split conformal prediction is determined by an absolute-residual
quantile. Gaussian residuals provide an important case in which these two
summaries agree. The result does not require an exact conditional linear model,
a nonsingular population feature covariance, or mixing of the calibration
sequence; stationarity and ergodicity are sufficient for the deterministic
contiguous calibration blocks considered here.

This result is related to earlier work on the asymptotic efficiency of
conformalized ridge regression \cite{BurnaevVovk2014}. Here, we characterize
the fixed-dimensional comparison for the NGRC readout without requiring an
exact conditional linear model or a nonsingular population feature covariance.

The same calculation also identifies an effective-dimensional extension. Let
\(
d_{\mathrm{eff},n}
=
\operatorname{tr}\{
S_n(S_n+\lambda_nI_p)^{-1}
\}.
\)
For the fitted observations,
\[
\frac{1}{n}\sum_{t=1}^n W_{B,n}(r_t)^2
=
4z^2\widehat\tau_n^2
\left(
1+\frac{d_{\mathrm{eff},n}}{n}
\right).
\]
Thus the Bayesian parameter-uncertainty contribution also vanishes whenever
$d_{\mathrm{eff},n}/n\to0$, provided the corresponding prediction and
calibration errors are well behaved. This observation motivates the
proportional regime, where the effective dimension need not be negligible.

\begin{corollary}[Fixed-dimensional marginal coverage]
\label{cor:fixed_coverage}
Suppose the conditions of Theorem~\ref{thm:fixed_width_discrepancy} hold, and
let $(r_*,y_*)$ be an independent draw from the stationary marginal
distribution. Write \(e_*=y_*-r_*^\top w^\dagger,\) and
\(F_{|e|}(x)=\mathbb P(|e_*|\le x).\) Assume additionally that $F_{|e|}$ is
continuous at $z\tau$. Then
\(\mathbb P\{y_*\in C_{B,n}(r_*)\}\longrightarrow F_{|e|}(z\tau),\)
and
\(\mathbb P\{y_*\in C_{C,n}(r_*)\}\longrightarrow1-\alpha.\)
In particular, when $q_e=z\tau$, the Bayesian interval also has asymptotic
marginal coverage $1-\alpha$.
\end{corollary}

Corollary~\ref{cor:fixed_coverage} concerns marginal coverage for an independent
draw from the stationary marginal distribution; it does not claim conditional
coverage given the current forecasting history. The Gaussian readout interval
is asymptotically calibrated according to where the Gaussian radius $z\tau$
falls in the actual absolute prediction-error distribution, whereas split
conformal prediction estimates the corresponding absolute-error quantile
directly. Thus, the same residual RMS--quantile mismatch that determines the
limiting width difference also determines asymptotic Bayesian marginal
miscalibration.

\subsection{High-Dimensional Bias, Variance, and Posterior Uncertainty}
\label{subsec:hd_decomposition}

We next consider a sequence of problems for which $p=p_n$ grows with the
training size $n$. Write
\(
S=\frac{1}{n}R^\top R,
\;\;
A=(S+\lambda I_p)^{-1},
\;\;
\widehat w=A\frac{R^\top y}{n},
\)
where $\lambda>0$ is fixed on the normalized covariance scale.

For the high-dimensional benchmark, suppose
\(
y=Rw_0+\varepsilon,
\;\;
\mathbb E[\varepsilon\mid R]=0,
\;\;
\operatorname{Cov}(\varepsilon\mid R)=\sigma^2I_n.
\)
Let $r_*$ be an independent test feature with
$\Sigma=\mathbb E[r_*r_*^\top]$, and let
$y_*=r_*^\top w_0+\varepsilon_*$, where $\varepsilon_*$ is independent
mean-zero test noise with variance $\sigma^2$.

Define $B_n=\lambda^2w_0^\top A\Sigma Aw_0$,
$V_n=\frac{\sigma^2}{n}\operatorname{tr}(\Sigma ASA)$, and
$P_n=\frac{\sigma^2}{n}\operatorname{tr}(\Sigma A)$.
Here $B_n$ is squared ridge prediction bias, $V_n$ is frequentist estimation
variance, and $P_n$ is feature-averaged Bayesian posterior parameter
uncertainty. The following comparison is an exact finite-sample algebraic
identity conditional on the design.

\begin{lemma}[Bias--variance--posterior decomposition]
\label{lem:bvp_decomposition}
The out-of-sample prediction error satisfies
\begin{equation}
\mathbb E[(y_*-r_*^\top\widehat w)^2\mid R]
=
\sigma^2+B_n+V_n,
\;\; \text{and} \;\;
P_n-V_n-B_n
=
\frac{\sigma^2\lambda}{n}\operatorname{tr}(\Sigma A^2)
-
\lambda^2w_0^\top A\Sigma Aw_0.
\label{eq:bvp_difference}
\end{equation}
\end{lemma}

The proof is given in Appendix~\ref{app:hd_theory}. The decomposition highlights
the central distinction between the two uncertainty procedures:
\(
\text{Bayesian predictive variance}
=
\text{noise}+\text{posterior uncertainty},
\)
whereas
\(
\text{frequentist prediction error}
=
\text{noise}+\text{estimation variance}+\text{ridge bias}.
\)
Therefore, the fact that posterior uncertainty exceeds frequentist estimation
variance does not by itself imply that the Bayesian interval is wider. The
prediction bias appearing in the conformal residual distribution can reverse
the ordering.

We next translate this decomposition into interval widths. Assume
$\varepsilon\mid R\sim N(0,\sigma^2I_n)$. Suppose the calibration and test
features are independent of the training sample and are i.i.d.
$N(0,\Sigma)$, with independent Gaussian noises. Assume $p/n$ is bounded,
$\|\Sigma\|_{\mathrm{op}}$ and $\|w_0\|$ are uniformly bounded, and the
calibration size $m\to\infty$.

For the conformal calculation, define the realized feature-averaged
prediction-error variance
\(
Q_n=(\widehat w-w_0)^\top\Sigma(\widehat w-w_0).
\)
Conditional on the fitted training sample, a fresh calibration or test residual
is Gaussian with variance $\sigma^2+Q_n$. Moreover,
\(\mathbb E(Q_n\mid R)=B_n+V_n,\)
and the argument in Appendix~\ref{app:hd_theory} gives
\(Q_n=B_n+V_n+O_{\mathbb P}(n^{-1/2}).\)
This concentration converts the conditional residual distribution into the
finite-trace conformal width approximation below.

\begin{theorem}[Finite-trace width comparison]
\label{thm:finite_trace_widths}
For the common ridge center and known noise scale $\sigma$,
\(
\frac{W_{B,n}(r_*)^2}{4z^2}
=
\sigma^2+P_n+O_{\mathbb P}(n^{-1/2}),
\;\; \text{whereas} \;\;
\frac{W_{C,n}^2}{4z^2}
=
\sigma^2+B_n+V_n+O_{\mathbb P}(n^{-1/2}+m^{-1/2}).
\)
Consequently,
\begin{equation}
\frac{W_{B,n}(r_*)^2-W_{C,n}^2}{4z^2}
=
\frac{\sigma^2\lambda}{n}
\operatorname{tr}(\Sigma A^2)
-
B_n
+
O_{\mathbb P}(n^{-1/2}+m^{-1/2}).
\label{eq:finite_trace_width_difference}
\end{equation}
\end{theorem}

The resulting interval comparison is asymptotic. Under the assumptions above,
the Bayesian and conformal squared widths admit finite-trace approximations,
proved in Appendix~\ref{app:hd_theory}. Their difference in
Equation~\eqref{eq:finite_trace_width_difference} includes ridge prediction bias
and therefore allows either procedure to have the larger asymptotic width. The
general finite-trace comparison depends on covariance-weighted traces and
signal alignment and does not require isotropy.

The high-dimensional Bayesian result treats $\sigma$ as known, or requires a
separately justified consistent estimator. In proportional dimension, the
training residual root mean square need not consistently estimate $\sigma$
because ridge shrinkage and high-dimensional leverage remain non-negligible;
the calculation is given in Appendix~\ref{app:hd_residual_scale}.

\subsection{Quadratic NGRC Features}
\label{subsec:quadratic_ngrc_theory}

The finite-trace decomposition in
Theorem~\ref{thm:finite_trace_widths} is a generic ridge result. Applying it to
NGRC requires additional work because polynomial features constructed from the
same delay vector are dependent. In particular, for a Gaussian delay vector
$G$, coordinates such as $G_jG_k$ and $G_j^2-1$ are not independent even after
centering and normalization. Whitening therefore makes the coordinates
uncorrelated but does not reduce the problem to the independent Gaussian
feature model considered in the preceding high-dimensional benchmark.

To isolate this issue, let $G\sim N(0,I_d)$ and define the centered quadratic
feature map
\begin{equation}
\phi_d(G)=
\begin{pmatrix}
G \\
\operatorname{svec}
\left(
\frac{GG^\top-I_d}{\sqrt{2}}
\right)
\end{pmatrix}
\in\mathbb R^p,
\qquad
p=\frac{d(d+3)}{2},
\label{eq:quadratic_feature_map}
\end{equation}
where $\operatorname{svec}$ is an isometry from symmetric matrices under the
Frobenius norm to Euclidean space. The quadratic coordinates consist of
$(G_j^2-1)/\sqrt{2}$ and $G_jG_k$, $j<k$. Although these coordinates are
dependent, they satisfy $\mathbb E[\phi_d(G)]=0$ and
$\mathbb E[\phi_d(G)\phi_d(G)^\top]=I_p$.

The map $\phi_d$ is a population-normalized theoretical benchmark rather than
the literal empirical NGRC representation in
Section~\ref{sec:method}, which includes an intercept and fit-sample
normalization of raw lag and quadratic-product coordinates. Extending the
theory to empirical normalization under dependent delays remains an open
problem.

Isotropy alone is not sufficient for the preceding Gaussian-design argument.
The extension requires quadratic-form concentration to control the spectral
behavior of the dependent feature matrix and sufficient diffuseness of the
quadratic prediction direction to obtain a Gaussian approximation for scalar
prediction errors. Quadratic-form concentration for isotropic random vectors
with dependent coordinates is sufficient for Marchenko--Pastur behavior under
proportional growth \cite{Yaskov2016}. For the feature map
\eqref{eq:quadratic_feature_map}, this concentration can be verified directly.
The corresponding technical results are given in
Appendix~\ref{app:quadratic_feature_proofs}.

To describe the second requirement, for any $v\in\mathbb R^p$, write
$v=(a,q)$, where $a\in\mathbb R^d$ and
$q\in\mathbb R^{d(d+1)/2}$, and define
\(
\mathcal Q_d(v)=\operatorname{svec}^{-1}(q).
\)
Then
\(
v^\top\phi_d(G)
=
a^\top G
+
\{G^\top\mathcal Q_d(v)G
-
\operatorname{tr}(\mathcal Q_d(v))\}/\sqrt{2}.
\)
Thus $\|\mathcal Q_d(v)\|_{\mathrm{op}}$ measures the largest individual
quadratic direction in the scalar functional generated by $v$.

Let
\(r=\Sigma_p^{1/2}\phi_d(G), \;\; \|\Sigma_p\|_{\mathrm{op}}\le L,\)
and assume
\(p/n\longrightarrow\gamma\in(0,\infty), \;\; m\longrightarrow\infty,\)
and \(\lambda>0.\)
The rows of the training design $R$, the calibration features, and the test
feature are independent copies of $r$, with independent Gaussian observation
noise of variance $\sigma^2$.

For the first benchmark, consider the random-effects readout
\begin{equation}
w_0
\sim
N\left(
0,\frac{s_p^2}{p}I_p
\right),
\qquad
y=r^\top w_0+\varepsilon,
\qquad
\varepsilon\sim N(0,\sigma^2),
\label{eq:quadratic_random_effects}
\end{equation}
where $s_p^2$ is uniformly bounded and the readout, feature vectors, and noises
are mutually independent.

Let $S=R^\top R/n$, $A=(S+\lambda I_p)^{-1}$, and define
$P_n=\frac{\sigma^2}{n}\operatorname{tr}(\Sigma_pA)$ and
$T_n=\frac{s_p^2\lambda^2}{p}\operatorname{tr}(\Sigma_pA^2)
+\frac{\sigma^2}{n}\operatorname{tr}(\Sigma_pASA)$.

\begin{theorem}[Quadratic-NGRC width approximation]
\label{thm:quadratic_ngrc_width}
Under the preceding conditions,

$$
\frac{W_{B,n}(r_*)^2}{4z^2}
=
\sigma^2+P_n+O_{\mathbb P}(n^{-1/4}),
\qquad
\frac{W_{C,n}^2}{4z^2}
=
\sigma^2+T_n+
O_{\mathbb P}(n^{-1/4}+m^{-1/2}).
$$

Consequently,
\begin{equation*}
\frac{W_{B,n}(r_*)^2-W_{C,n}^2}{4z^2}
=
\left(
\frac{\sigma^2\lambda}{n}
-
\frac{s_p^2\lambda^2}{p}
\right)
\operatorname{tr}(\Sigma_pA^2)
+
O_{\mathbb P}
\left(
n^{-1/4}+m^{-1/2}
\right).
\end{equation*}
When $\Sigma_p=I_p$ and $s_p^2\to s^2$, the same
Marchenko--Pastur width limits and phase relation as in
Corollary~\ref{cor:hd_phase} below hold despite the dependence among the
quadratic feature coordinates.
\end{theorem}

The proof is given in Appendix~\ref{app:quadratic_feature_proofs}. The
$n^{-1/4}$ terms above describe approximation by the finite-sample trace
quantities and are not claimed as rates for convergence of those traces to
their Marchenko--Pastur limits.

Theorem~\ref{thm:quadratic_ngrc_width} provides the main link between the
generic high-dimensional ridge comparison and quadratic NGRC. Although the
features generated from a common delay vector are dependent, quadratic-form
concentration controls the spectral behavior of the design, while diffuseness
of the prediction direction controls the distribution of calibration and test
residuals. Together these properties recover the finite-trace Bayesian and
conformal width comparison for an explicit nonlinear NGRC feature model.

The random-effects assumption in
\eqref{eq:quadratic_random_effects} is a sufficient mechanism for obtaining a
diffuse fitted prediction direction, but randomness of the underlying NGRC
readout is not required. For a deterministic readout, define
\(B_n=\lambda^2w_0^\top A\Sigma_pAw_0\),
\(V_n=\frac{\sigma^2}{n}\operatorname{tr}(\Sigma_pASA),\)
and
\(
D_n
=
\mathcal Q_d
\left(
\lambda\Sigma_p^{1/2}Aw_0
\right).
\)

\begin{corollary}[Deterministic quadratic-NGRC readout]
\label{cor:quadratic_deterministic}
Suppose the feature, noise, proportional-growth, and regularization conditions
of Theorem~\ref{thm:quadratic_ngrc_width} hold, but let $w_0=w_{0,p}$ be
deterministic and satisfy $\sup_p\|w_{0,p}\|<\infty$. Assume that, for some
deterministic sequence $\rho_n\downarrow0$,
\begin{equation}
\|D_n\|_{\mathrm{op}}
=
O_{\mathbb P}(\rho_n).
\label{eq:deterministic_diffuseness}
\end{equation}
Then
\[
\frac{W_{B,n}(r_*)^2}{4z^2}=\sigma^2+P_n+O_{\mathbb P}(n^{-1/4}), \;\; \text{whereas} \;\; \frac{W_{C,n}^2}{4z^2}=\sigma^2+B_n+V_n+O_{\mathbb P}\!\left(n^{-1/4}+\rho_n+m^{-1/2}\right).
\]
Consequently,
\[
\frac{W_{B,n}(r_*)^2-W_{C,n}^2}{4z^2}
=
\frac{\sigma^2\lambda}{n}
\operatorname{tr}(\Sigma_pA^2)
-
B_n
+
O_{\mathbb P}
\left(
n^{-1/4}+\rho_n+m^{-1/2}
\right).
\]
In particular, if $\|D_n\|_{\mathrm{op}}=O_{\mathbb P}(n^{-1/4})$, then the
deterministic-readout extension has the same approximation order as the
random-effects result.
\end{corollary}

The proof is given in Appendix~\ref{app:quadratic_feature_proofs}. The key
point is that the random component of the fitted ridge error already has largest
quadratic direction of order $O_{\mathbb P}(n^{-1/4})$, while
\eqref{eq:deterministic_diffuseness} directly controls the deterministic
ridge-bias component. Lemma~\ref{lem:quadratic_gaussian_approx} then bounds the
Gaussian-approximation error by the operator norm of the combined quadratic
direction.

Corollary~\ref{cor:quadratic_deterministic} clarifies the role of the
random-effects model. Random effects are a sufficient mechanism for
diffuseness, but randomness of the underlying NGRC readout is not required.
The essential condition is that no individual quadratic direction carries a
non-negligible fraction of the fitted prediction error. Importantly,
\eqref{eq:deterministic_diffuseness} concerns the quadratic component of the
ridge-transformed bias direction, not the raw signal
$\mathcal Q_d(w_0)$. The sample resolvent $A$ and feature covariance
$\Sigma_p$ therefore enter the diffuseness requirement.

This is a population-level sufficient condition rather than a directly
observable diagnostic for a fitted NGRC model, since it depends on the unknown
readout $w_0$ and population feature covariance. Its purpose is to identify the
mechanism required by the deterministic-readout extension---diffuseness of the
ridge-transformed quadratic bias direction---rather than to provide a
data-driven model-selection criterion. We therefore interpret this result as an
analytical benchmark for NGRC uncertainty rather than a complete calibration
theory for deployed forecasting systems. The condition also does not require
the total squared ridge bias $B_n$ to vanish; the bias may remain first order
while being distributed across many individually small quadratic directions.

\paragraph{Why quadratic diffuseness is needed.}
Consider $f_d(G)=(G_1^2-1)/\sqrt{2}$. For every $d$,
$\operatorname{Var}\{f_d(G)\}=1$, but its distribution is a centered and
scaled $\chi_1^2$ distribution and does not approach a Gaussian distribution
as $d\to\infty$. The associated quadratic matrix is $B_d=e_1e_1^\top$, and
$\|B_d\|_{\mathrm{op}}=1$, so all quadratic variation remains concentrated in
a single direction.

By contrast, consider the diffuse quadratic functional

$$
\widetilde f_d(G)
=
\frac{1}{\sqrt d}
\sum_{j=1}^d
\frac{G_j^2-1}{\sqrt{2}}
=
\frac{\|G\|^2-d}{\sqrt{2d}}.
$$

Its associated quadratic matrix is $\widetilde B_d=I_d/\sqrt d$, so
$\|\widetilde B_d\|_{\mathrm{op}}=d^{-1/2}\longrightarrow0$, and
$\operatorname{tr}(\widetilde B_d^2)=1$. Thus its total variance remains of
constant order while its largest individual quadratic direction vanishes. The
central limit theorem for the centered chi-square statistic yields
$\widetilde f_d(G)\xrightarrow{d}N(0,1)$.

The two examples show what
\eqref{eq:deterministic_diffuseness} rules out. High dimensionality alone does
not imply Gaussian prediction errors if the signal remains concentrated in a
small number of quadratic directions.

\subsection{Signal-Dependent Width Ordering in Proportional Dimension}
\label{subsec:hd_phase}

The finite-trace formulas above are the more general comparison. They show that
Bayesian and conformal widths differ because posterior parameter uncertainty
need not equal the combination of frequentist estimation variance and ridge
prediction bias. This mechanism does not require isotropy and, for general
feature covariance, depends on covariance-weighted traces and on the alignment
of the signal with the feature eigenspaces.

The isotropic proportional-growth model gives an explicit asymptotic
specialization of these finite-trace quantities. In addition to the conditions
of Theorem~\ref{thm:finite_trace_widths}, suppose the rows of the training
design are i.i.d. $N(0,I_p)$ and are independent of the training noises,
calibration sample, and test pair. Assume
\(p/n\longrightarrow\gamma\in(0,\infty)\),
\(\|w_0\|^2\longrightarrow s^2,\)
and $\lambda>0$ is fixed.

Define

$$
I_1(\gamma,\lambda)
=
\int\frac{1}{x+\lambda}\,
dF_\gamma^{\mathrm{MP}}(x),
\qquad
I_2(\gamma,\lambda)
=
\int\frac{1}{(x+\lambda)^2}\,
dF_\gamma^{\mathrm{MP}}(x),
$$

and

$$
J(\gamma,\lambda)
=
I_1-\lambda I_2
=
\int
\frac{x}{(x+\lambda)^2}
\,dF_\gamma^{\mathrm{MP}}(x),
$$

where $F_\gamma^{\mathrm{MP}}$ denotes the Marchenko--Pastur law
\cite{MarchenkoPastur1967}. These resolvent quantities are standard in
high-dimensional ridge prediction theory \cite{DobribanWager2018}.

\begin{corollary}[Signal-dependent high-dimensional width ordering]
\label{cor:hd_phase}
Under the preceding conditions,

$$
\frac{W_{B,n}(r_*)^2}{4z^2}
\xrightarrow{\mathbb P}
\sigma^2(1+\gamma I_1),
\qquad
\frac{W_{C,n}^2}{4z^2}
\xrightarrow{\mathbb P}
\sigma^2(1+\gamma J)
+
s^2\lambda^2I_2.
$$

Therefore,
\begin{equation}
W_{B,n}(r_*)^2-W_{C,n}^2
\xrightarrow{\mathbb P}
4z^2\lambda I_2
\left(
\sigma^2\gamma-s^2\lambda
\right).
\label{eq:hd_phase}
\end{equation}
Consequently,

$$
\begin{cases}
s^2\lambda<\sigma^2\gamma,
& W_B>W_C,\\[1mm]
s^2\lambda=\sigma^2\gamma,
& W_B=W_C,\\[1mm]
s^2\lambda>\sigma^2\gamma,
& W_B<W_C
\end{cases}
\qquad
\text{asymptotically}.
$$

\end{corollary}

The phase boundary is therefore a closed-form specialization of the more
general finite-trace discrepancy rather than the source of the
high-dimensional phenomenon itself. In proportional dimension, posterior
uncertainty, estimation variance, and ridge bias all remain first order, and
their competition allows either procedure to produce the wider interval. Under
the isotropic benchmark, this comparison reduces to the scalar condition
$s^2\lambda=\sigma^2\gamma$.

The equality condition also has a Bayesian interpretation. Under the prior used
in Section~\ref{subsec:bayes_ridge},
\(w\sim N\left(0,\frac{\sigma^2}{n\lambda}I_p\right),\)
and
\(\mathbb E\|w\|^2=\frac{p\sigma^2}{n\lambda}
\longrightarrow\frac{\sigma^2\gamma}{\lambda}.\)
Thus $s^2\lambda=\sigma^2\gamma$ matches the limiting prior signal scale to
the limiting true signal strength. The centered benchmark $w_0=0$ corresponds
to only one side of the phase comparison and necessarily produces wider
Bayesian intervals. The prior-matching interpretation is specific to the
isotropic benchmark; the broader finite-trace comparison does not require
isotropy or this prior representation.

The same asymptotic quantities determine Bayesian marginal coverage. Let
\(v_B=\sigma^2(1+\gamma I_1)\) and
\(v_C=\sigma^2(1+\gamma J)+s^2\lambda^2I_2.\)
Then
\[
\mathbb P\{y_*\in C_{B,n}(r_*)\}
\longrightarrow
2\Phi\!\left(
z\sqrt{\frac{v_B}{v_C}}
\right)-1,
\]
whereas split conformal prediction retains its nominal marginal coverage in the
independent benchmark. A derivation is given in
Appendix~\ref{app:hd_coverage}. In particular, the phase relation also
describes the direction of Bayesian miscalibration. Hence,
\[
\begin{cases}
v_B>v_C,
& \displaystyle
\lim_{n\to\infty}
\mathbb P\!\left\{y_*\in C_{B,n}(r_*)\right\}
>
1-\alpha,\\[2mm]
v_B=v_C,
& \displaystyle
\lim_{n\to\infty}
\mathbb P\!\left\{y_*\in C_{B,n}(r_*)\right\}
=
1-\alpha,\\[2mm]
v_B<v_C,
& \displaystyle
\lim_{n\to\infty}
\mathbb P\!\left\{y_*\in C_{B,n}(r_*)\right\}
<
1-\alpha.
\end{cases}
\]

Equivalently, the weak-signal regime produces asymptotic Bayesian overcoverage,
the phase boundary produces nominal marginal coverage, and the strong-signal
regime produces asymptotic Bayesian undercoverage. Therefore, when the Bayesian
interval becomes narrower than the conformal interval in the strong-signal
regime, its smaller width reflects undercoverage rather than improved
efficiency at equal coverage. This fixed-$\lambda$ boundary is a specialization
of the general finite-trace discrepancy, not an intrinsic failure of Bayesian
inference in high dimension. Data-dependent regularization requires separate
analysis.

The phase boundary is also connected to optimal high-dimensional ridge
regularization. For the isotropic proportional benchmark, the limiting ridge
prediction risk is minimized at
\(
\lambda_*=\frac{\sigma^2\gamma}{s^2},
\)
which is the standard optimal-ridge scaling in this setting
\cite{DobribanWager2018}. At this value,
\(
s^2\lambda_*=\sigma^2\gamma,
\)
so the limiting Bayesian and conformal widths coincide. The corresponding
limiting-risk calculation is given in Appendix~\ref{app:optimal_ridge}. We use
this connection as an interpretation of the width comparison.

\paragraph{Non-isotropic feature covariances.}
The isotropic result gives the clearest closed-form phase boundary, but the
underlying finite-trace decomposition is more general. For non-isotropic
features, the relevant posterior and frequentist estimation-variance terms are
the covariance-weighted traces
$\operatorname{tr}(\Sigma A)/n$ and
$\operatorname{tr}(\Sigma ASA)/n$, whose difference is
$\lambda\operatorname{tr}(\Sigma A^2)/n$.

These quantities cannot generally be replaced by unweighted integrals against
the empirical spectral distribution of $S$. Moreover, the ridge-bias term
$w_0^\top A\Sigma Aw_0$ depends on the alignment of the signal with the
eigenspaces of $\Sigma$. Consequently, the feature spectrum alone does not
determine the Bayesian--conformal width ordering outside the isotropic setting.
The general finite-trace comparison, rather than the scalar boundary
$s^2\lambda=\sigma^2\gamma$, is therefore the appropriate statement for
anisotropic designs.

\paragraph{Relation to the fixed-dimensional regime.}
The proportional-growth result does not reduce automatically to
Theorem~\ref{thm:fixed_width_discrepancy} by taking $\gamma\downarrow0$. The
high-dimensional analysis holds the normalized ridge penalty $\lambda>0$
fixed. Under this scaling, a nonvanishing shrinkage bias may persist even as
the aspect ratio decreases. Recovering the fixed-dimensional result therefore
additionally requires $\lambda\longrightarrow0$ on the normalized covariance
scale, or more generally requires the resulting prediction bias to vanish. The
limits $p/n\to0$ and $\lambda\to0$ therefore represent distinct asymptotic
operations.

\paragraph{Noise-scale estimation.}
In fixed dimension, the fitted residual root mean square is consistent under
the conditions of Lemma~\ref{lem:projection_consistency}. In proportional
dimension, this need not hold because ridge shrinkage and high-dimensional
leverage remain non-negligible. A direct calculation of the proportional
training-residual scale is given in
Appendix~\ref{app:hd_residual_scale}. Accordingly, the Bayesian
high-dimensional results treat $\sigma$ as known, or require a separately
justified consistent estimator.

\paragraph{Scope of the dependent-window extension.}
The high-dimensional results above are stated first for independent training,
calibration, and test feature vectors. NGRC features constructed from a time
series are generally dependent across forecast origins.
Appendix~\ref{app:dependent_transfer} provides a sufficient transfer result
under absolute regularity. If a feature--response pair at origin $t$ depends on
observations in $[t-L_n,t+H]$, with $L_n=(k_n-1)s_n$, and retained forecast
origins $t_1<\cdots<t_N$ satisfy
$t_{j+1}-t_j\ge L_n+H+h_n$ with $h_n\ge1$, then
Proposition~\ref{prop:beta_transfer} couples the retained windows to
independent copies with the same marginal distribution, with probability of at
least one coupling failure bounded by $(N-1)\beta(h_n)$. Consequently, an
event with independent-window failure probability at most $\delta$ has
dependent-window failure probability at most
$\delta+(N-1)\beta(h_n)$.

For split conformal prediction, whenever the independent copies satisfy the
usual assumptions for the standard marginal coverage guarantee, this yields

$$
\mathbb P\{y_*\in C_C(r_*)\}
\ge
1-\alpha-(N-1)\beta(h_n).
$$

The separation condition applies to every pair of successive retained forecast
windows, not merely to the boundaries between fitting, calibration, and test
blocks. Thus, the proposition describes a thinned sequence of sufficiently
separated NGRC forecast origins. It does not provide an independence
approximation for the ordinary overlapping delay windows used at every
consecutive forecast origin.

The separation requirement can also be costly in the proportional
quadratic-NGRC regime. With a full quadratic lag map, $p\asymp k_n^2$. If
$p\asymp n$, then $k_n\asymp n^{1/2}$. For fixed lag spacing,
$L_n=(k_n-1)s_n\asymp n^{1/2}$. Retaining $N\asymp n$ windows while enforcing
$t_{j+1}-t_j\ge L_n+H+h_n$ therefore requires an underlying time series of
length at least order $NL_n\asymp n^{3/2}$, even before accounting for any
additional mixing gap $h_n$. The proposition should therefore be viewed as a
sufficient theoretical bridge for heavily separated windows rather than as a
direct asymptotic description of ordinary overlapping NGRC forecasts.

The result is also only a transfer principle. It transfers probability
statements from an independent-window model to sufficiently separated
dependent windows, but it does not establish Gaussian feature marginals, the
linear readout model, quadratic-feature concentration, or a
Marchenko--Pastur law for an arbitrary dependent NGRC process. Those
conditions must be verified separately. In particular, preventing overlap
between adjacent delay windows is not sufficient for approximate independence;
the additional separation $h_n$ must be large relative to the decay of the
temporal dependence. For example, if
$\beta(h)\le Ce^{-ch}$, choosing
$h_n\ge c^{-1}\log\{C(N-1)/\eta_n\}$ makes the additional dependence error at
most $\eta_n$.

\paragraph{Summary of theoretical results.}
The theoretical results identify distinct mechanisms in the low- and
high-dimensional regimes. In fixed dimension, the Gaussian readout interval is
governed asymptotically by the root mean square of the population projection
error, whereas split conformal prediction is governed by an absolute-error
quantile. Their limiting widths therefore agree only when these two summaries
of the prediction-error distribution are appropriately matched.

In proportional dimension, posterior uncertainty, frequentist estimation
variance, and ridge prediction bias all remain first order. The finite-trace
decomposition shows that the Bayesian--conformal discrepancy is determined by
the difference between posterior parameter uncertainty and the combination of
sampling variance and shrinkage bias, allowing either procedure to produce the
wider interval.

The quadratic-NGRC result is the main nonlinear extension of this comparison.
Quadratic NGRC coordinates generated from a common delay vector are dependent
even after normalization, so the standard independent-coordinate Gaussian
argument does not apply. Nevertheless, quadratic-form concentration and a
Gaussian approximation for sufficiently diffuse quadratic prediction
directions recover the same finite-trace uncertainty comparison. The
deterministic-readout extension further shows that a random-effects model for
the true readout is not essential. The relevant structural requirement is
diffuseness of the ridge-transformed quadratic bias direction, and the total
prediction bias itself need not vanish.

The isotropic proportional benchmark then converts the finite-trace quantities
into explicit Marchenko--Pastur limits. This yields the signal-dependent
relations $s^2\lambda<\sigma^2\gamma$,
$s^2\lambda=\sigma^2\gamma$, and
$s^2\lambda>\sigma^2\gamma$, which determine the asymptotic ordering of the
Bayesian and conformal widths and the corresponding direction of Bayesian
marginal miscalibration. Thus, the phase boundary is a consequence of the
broader posterior-variance, estimation-variance, and bias decomposition rather
than the source of the high-dimensional phenomenon itself.

\paragraph{Limitations.}
The high-dimensional theory is developed primarily under an independent-window
benchmark. The absolute-regularity result in
Appendix~\ref{app:dependent_transfer} transfers independent-window probability
statements to sufficiently separated dependent windows with an explicit
coupling error, but it does not establish the full quadratic-feature or
random-matrix theory for arbitrary dependent time series.

The quadratic-NGRC analysis also uses an explicit Gaussian latent delay vector
to make the dependence among polynomial coordinates analytically tractable.
The deterministic-readout extension removes the requirement that the
underlying forecasting signal itself be a random draw, but it still requires
the quadratic component of the ridge-transformed prediction-bias direction to
become asymptotically diffuse. Concentrated quadratic directions, such as the
single-coordinate example above, need not admit the Gaussian approximation
used to obtain the conformal width formula.

Finally, the proportional results assume a known noise scale or a separately
justified consistent estimator for the Bayesian interval. This assumption is
not interchangeable with using the training residual variance in high
dimension. Split conformal prediction avoids this particular noise-scale
estimation requirement by calibrating directly on held-out prediction
residuals.

We empirically evaluate these theoretical approximations in
Section~\ref{sec:simulations}, including settings with temporal dependence,
quadratic features, and non-negligible feature-to-sample ratios. The real-data
analyses do not assume that the independent-window theory holds exactly.
Calibration and interval efficiency are instead evaluated directly on
chronologically ordered test observations.
\section{Simulations}
\label{sec:simulations}

We conduct five experiments examining the theoretical predictions and
finite-sample behavior of the uncertainty-quantification procedures.
Experiment~1 studies the residual-shape mechanism in fixed dimension.
Experiment~2 examines the high-dimensional width and coverage phase transition.
Experiment~3 studies dependent quadratic NGRC features, coefficient
diffuseness, and temporally dependent windows. Experiment~4 considers
nonlinear-memory forecasting under constant observation-noise variance and an
observation-noise variance change. Experiment~5 studies the fitting--calibration
tradeoff under controlled data budgets.

Unless otherwise stated, the nominal miscoverage level is $\alpha=0.05$.
Within each configuration of Experiments~1--4, the uncertainty-quantification
procedures use the same fitted ridge point predictor. Experiment~5 explicitly
allows the Bayesian and split-conformal point predictors to use different
fitting samples because data allocation is the object of study. Error bars
represent empirical means plus or minus $1.96$ standard errors. Experiments~1
and~2 are finite-sample diagnostics of the asymptotic mechanisms rather than
estimates of asymptotic approximation rates.

\subsection{Experiment 1: Residual Shape and Fixed-Dimensional Agreement}
\label{subsec:sim_fixed_residual_shape}

Theorem~\ref{thm:fixed_width_discrepancy} shows that, in fixed dimension,
Bayesian--conformal agreement depends on the residual distribution through
\(
W_{B,n}-W_{C,n}
\xrightarrow{\mathbb P}
2(z\tau-q_e),
\)
where $\tau^2=\mathbb E[e^2]$, $z=z_{1-\alpha/2}$, and $q_e$ is the
$(1-\alpha)$ quantile of $|e|$. Agreement therefore requires $q_e=z\tau$,
which may depend on the nominal level $\alpha$. We test this mechanism by
varying both residual shape and miscoverage level.

For each replication, we generate
\(
r_i\overset{\mathrm{i.i.d.}}{\sim}N(0,I_p)
\)
and
\(
y_i=r_i^\top w_0+e_i,
\)
with $p=5$, $w_0=(-0.4,-0.2,0,0.2,0.4)^\top$, and $\tau^2=1$.
We compare Gaussian $N(0,1)$, Laplace with scale $1/\sqrt{2}$,
standardized Student-$t_5$ errors $\sqrt{3/5}\,T_5$, and centered
exponential errors $E-1$, where $E\sim\operatorname{Exp}(1)$.

We consider
\(
\alpha\in\{0.01,0.05,0.10,0.20\}.
\)
For each configuration, we use $n_{\mathrm{fit}}=800$,
$n_{\mathrm{cal}}=800$, $n_{\mathrm{test}}=3000$, and ridge penalty
$\lambda=0.01$. Results are averaged over 300 independent replications. We
record the Bayesian--conformal width difference and the marginal coverage of
both procedures, comparing them with the theoretical width and Bayesian
coverage limits from Theorem~\ref{thm:fixed_width_discrepancy} and the
marginal-coverage limit from Corollary~\ref{cor:fixed_coverage}.

Figure~\ref{fig:sim_fixed_residual_shape} confirms the predicted
residual-shape mechanism. Gaussian residuals produce near-zero width
differences and nominal Bayesian coverage, whereas non-Gaussian residuals
produce $\alpha$-dependent width differences and corresponding Bayesian over-
or undercoverage. Split-conformal coverage remains near nominal across the
residual distributions. Thus, fixed-dimensional disagreement is governed by
residual shape rather than persistent parameter-estimation uncertainty.

\begin{figure}[t]
    \centering
    \includegraphics[width=\textwidth]{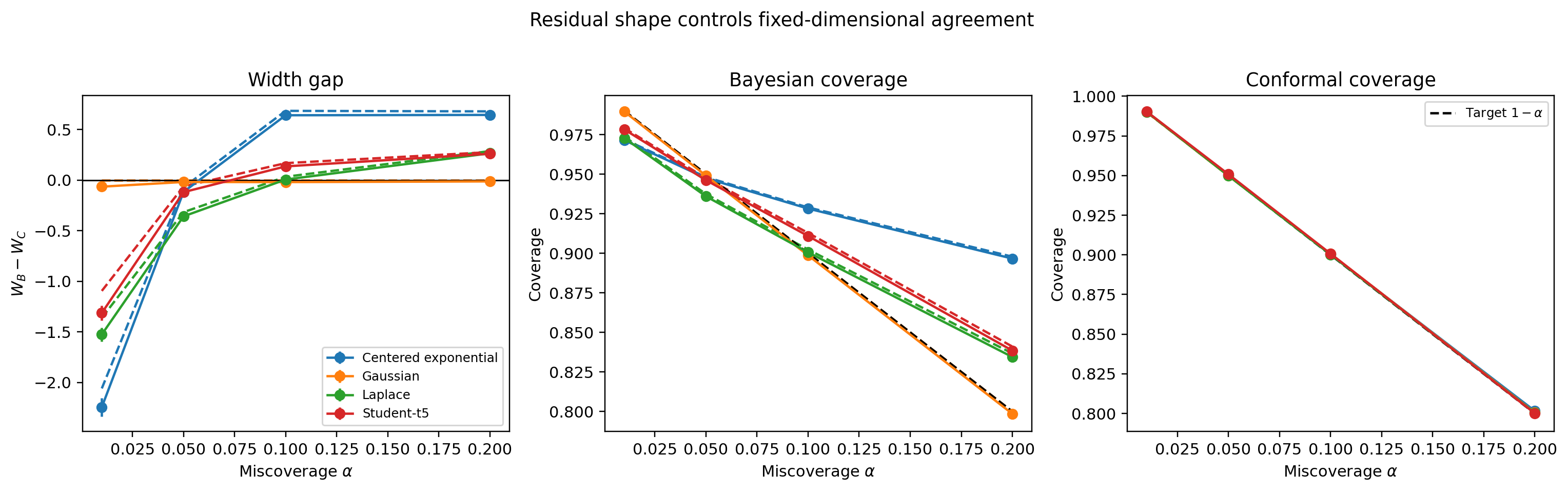}
    \caption{
    Fixed-dimensional residual-shape effects.
    Left: Bayesian--conformal width difference and theoretical limit.
    Middle: Bayesian coverage and its theoretical limit.
    Right: split-conformal coverage; dashed lines denote nominal coverage.
    }
    \label{fig:sim_fixed_residual_shape}
\end{figure}

\subsection{Experiment 2: High-Dimensional Phase Transition and Coverage}
\label{subsec:sim_hd_widths}

Corollary~\ref{cor:hd_phase} predicts a signal-dependent transition in the
Bayesian--conformal width ordering together with a corresponding change in
Bayesian marginal coverage. Unlike Experiment~1, the features and errors here
are Gaussian, so disagreement is driven by the interaction among signal
strength, regularization, and the feature-to-sample ratio.

We generate
\(
r_i\overset{\mathrm{i.i.d.}}{\sim}N(0,I_p),
\)
\(
y_i=r_i^\top w_0+\varepsilon_i,
\)
and
\(
\varepsilon_i\overset{\mathrm{i.i.d.}}{\sim}N(0,1).
\)
We take
\(
p=\operatorname{round}(\gamma n)
\)
with
\(
\gamma\in\{0.25,0.50,0.75,1.00,1.25,1.50\},
\)
and vary
\(
\lambda\in\{0.05,0.2,1\}
\)
and
\(
\mathrm{SNR}=s^2/\sigma^2\in\{0.25,1,4\},
\)
where $s^2=\|w_0\|^2$ and $\sigma^2=1$. Each configuration uses
$n=400$ fitting observations, $n_{\mathrm{cal}}=400$ independent calibration
observations, and $n_{\mathrm{test}}=1500$ independent test observations.
Results are averaged over 100 replications, and the Bayesian interval uses the
known noise variance. The realized ratio $p/n$ is used in the theoretical
calculations.

The predicted equal-width boundary is
\(
s^2/\sigma^2=\gamma/\lambda.
\)
We compare the empirical squared-width difference $W_B^2-W_C^2$ with its
finite-trace approximation and empirical Bayesian coverage with the
corresponding asymptotic limit.

Figure~\ref{fig:sim_hd_phase} shows the predicted high-dimensional mechanism.
The finite-trace approximation closely tracks the empirical squared-width
difference across aspect ratios, signal strengths, and ridge penalties,
including the predicted changes in width ordering. Bayesian coverage likewise
follows its theoretical limit. Configurations in which the Bayesian interval
becomes narrower lie on the undercoverage side of the phase boundary. Thus,
the strong-signal side of the transition represents reduced Bayesian width at
the cost of calibration rather than an efficiency gain at equal coverage. The
individual Bayesian and conformal finite-trace approximations are verified
separately in Appendix~\ref{app:sim_hd_width_verification}.

\begin{figure}[t]
    \centering
    \includegraphics[width=0.6\linewidth]{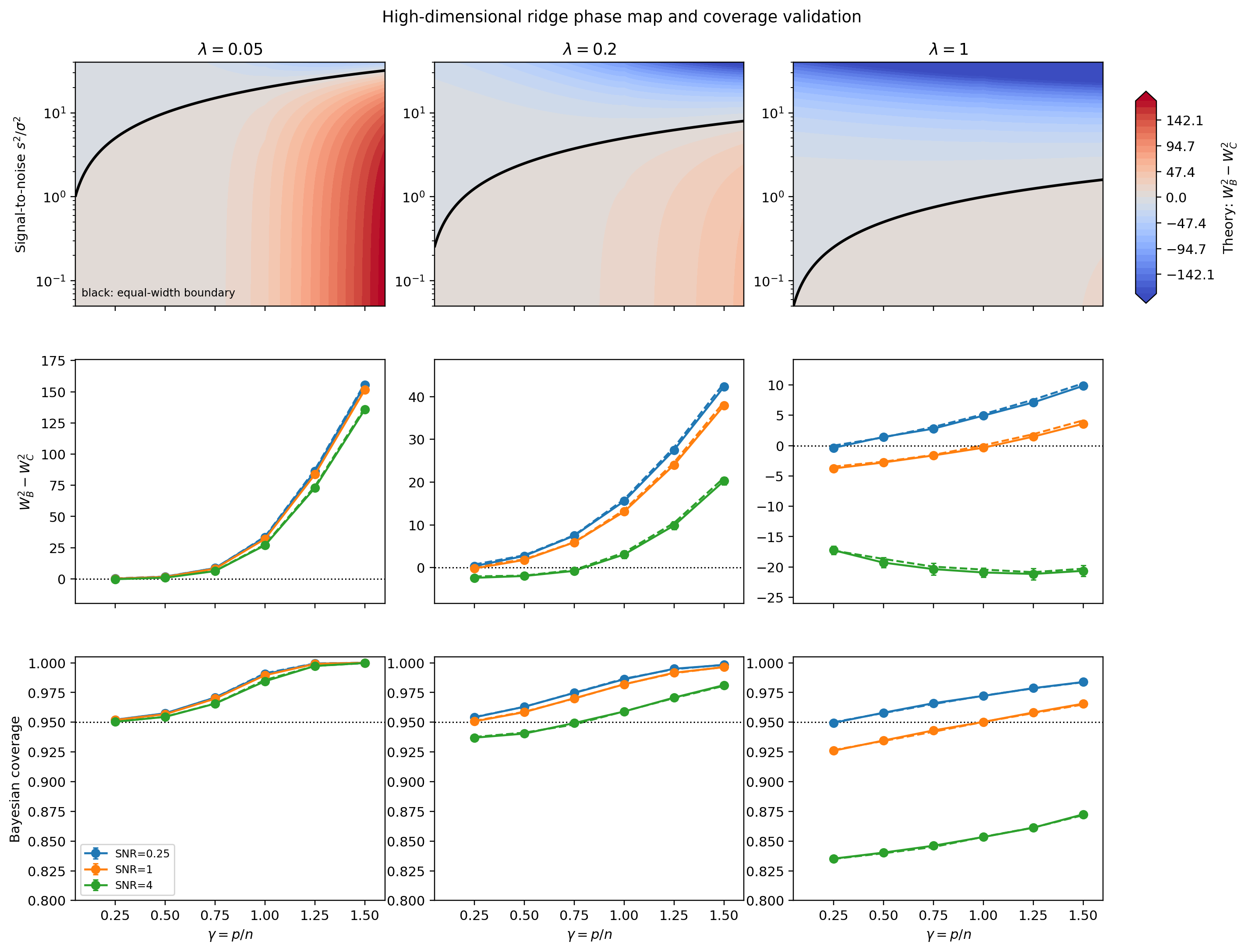}
    \caption{
    High-dimensional width and coverage phase transition.
    Top: Marchenko--Pastur limits for $W_B^2-W_C^2$ and the equal-width
    boundary. Middle: empirical squared-width differences and finite-trace
    approximations. Bottom: empirical and theoretical Bayesian coverage;
    dotted lines mark 95\% coverage.
    }
    \label{fig:sim_hd_phase}
\end{figure}

\subsection{Experiment 3: Quadratic NGRC Features, Diffuseness, and Temporal Transfer}
\label{subsec:sim_temporal_dependence}

Theorem~\ref{thm:quadratic_ngrc_width} extends the high-dimensional
Bayesian--conformal comparison to quadratic NGRC features with dependent
polynomial coordinates, while
Corollary~\ref{cor:quadratic_deterministic} identifies diffuseness of the
ridge-transformed quadratic bias direction as sufficient for the corresponding
Gaussian approximation. We examine these mechanisms and their extension to
temporally dependent feature windows.

For this experiment, we use the quadratic-only feature map
\[
\psi_d(G)
=
\operatorname{svec}
\left(
\frac{GG^\top-I_d}{\sqrt{2}}
\right)
\in\mathbb R^{p_q},
\qquad
p_q=\frac{d(d+1)}{2},
\qquad
G\sim N(0,I_d).
\]
The centered quadratic signal is generated as
$G^\top M_dG-\operatorname{tr}(M_d)$, with
\(
M_d^{\mathrm{conc}}=\sqrt{2}\,e_1e_1^\top
\)
and
\(
M_d^{\mathrm{diff}}=\sqrt{2/d}\,I_d.
\)

Under the normalized theoretical convention
\(
w_0^\top\psi_d(G)
=
\{G^\top\mathcal Q_d(w_0)G-\operatorname{tr}(\mathcal Q_d(w_0))\}/\sqrt{2},
\)
the corresponding coefficient matrices satisfy
\(
\mathcal Q_d(w_0)=\sqrt{2}\,M_d.
\)
Thus,
\(
\mathcal Q_d^{\mathrm{conc}}(w_0)=2e_1e_1^\top
\)
and
\(
\mathcal Q_d^{\mathrm{diff}}(w_0)=2I_d/\sqrt d.
\)
Consequently, $\|w_0\|^2=\|\mathcal Q_d(w_0)\|_F^2=4$ for both signal
configurations, while the generating matrices have equal squared Frobenius
norm $\|M_d\|_F^2=2$.

Experiment~3 uses the quadratic-only submap $\psi_d$ rather than the full
linear-plus-quadratic map in
Theorem~\ref{thm:quadratic_ngrc_width}. Thus, the simulated feature dimension is
$p_q=d(d+1)/2$, whereas the full theoretical map has dimension
$d(d+3)/2$. The same quadratic-form concentration and residual-approximation
arguments apply after the linear block is omitted. The concentrated--diffuse
deterministic construction is therefore intended primarily to diagnose the
mechanism in Corollary~\ref{cor:quadratic_deterministic}, rather than to
simulate the random-effects assumption in
Theorem~\ref{thm:quadratic_ngrc_width}.

We use $d\in\{8,12,16,20,24\}$,
$n_{\mathrm{fit}}=n_{\mathrm{cal}}=2p_q$, giving
$p_q/n_{\mathrm{fit}}=0.5$. Each configuration uses
$n_{\mathrm{test}}=2000$, $\lambda=0.5$, $\sigma=0.25$, fixed signal
magnitude $2$, and 150 independent replications at nominal 95\% coverage.

\paragraph{Diffuseness and finite-trace approximation.}
In each replication, we record
\(
\|D_n\|_{\mathrm{op}}
=
\left\|
\mathcal Q_d
\left(
\lambda\Sigma_p^{1/2}Aw_0
\right)
\right\|_{\mathrm{op}},
\)
the ridge-transformed diffuseness diagnostic appearing in
Corollary~\ref{cor:quadratic_deterministic}. We also record the KS distance
between standardized prediction residuals and a standard Gaussian distribution,
together with the error of the finite-trace conformal-width approximation.

As shown in the top row of Figure~\ref{fig:sim_quadratic_dependence}, the
ridge-transformed direction becomes increasingly diffuse under the diffuse
construction and its residual distribution becomes increasingly Gaussian as
$d$ grows. The concentrated construction retains a substantially larger
operator norm and a more pronounced residual-normality discrepancy. At the
same time, the finite-trace width approximation can remain accurate even when
the complete residual distribution is not close to Gaussian. These
finite-dimensional diagnostics are consistent with the mechanism in
Corollary~\ref{cor:quadratic_deterministic}; they are not presented as proof
that its asymptotic diffuseness condition holds.

\paragraph{Feature dependence and temporal transfer.}
At $d=24$, we compare independent Gaussian features, quadratic Gaussian
features, and quadratic lag features generated from a stationary Gaussian
AR(1) process with $\rho=0.8$. Concentration has little effect on residual
normality for Gaussian features but produces a much larger KS discrepancy for
quadratic features, while split-conformal coverage remains near nominal across
the designs. This shows that the relevant Gaussian-approximation issue arises
from the interaction of nonlinear feature dependence and concentration of the
quadratic prediction direction, rather than high dimension alone.

We next study how increasing forecast-origin separation moves the dependent
design toward its same-marginal independent-window benchmark. We use temporal
quadratic features with $\rho=0.8$, $d=20$, $p_q=210$, and
$n_{\mathrm{fit}}=n_{\mathrm{cal}}=420$, and vary forecast-origin separation
over $1,4,16,$ and $64$. For comparison, the independent-window benchmark uses
Gaussian lag windows with the same within-window AR(1) covariance, so the
marginal feature-response distribution is preserved while dependence across
forecast origins is removed.

The lower-right panel of Figure~\ref{fig:sim_quadratic_dependence} shows that
the KS discrepancy between the split-conformal width distribution and its
same-marginal independent-window benchmark decreases with separation, while
the marginal residual-normality diagnostic changes comparatively little.
Increasing separation therefore primarily reduces dependence across retained
windows rather than changing the nonlinear residual-shape mechanism within a
window. The smaller separations include overlapping windows and are empirical
stress tests rather than configurations covered by
Proposition~\ref{prop:beta_transfer}; the proposition becomes applicable only
once the retained windows are sufficiently separated. Detailed
absolute-regularity bounds and additional temporal-transfer diagnostics are
reported in Appendix~\ref{app:sim_beta_penalties}.

\begin{figure}[t]
    \centering
    \includegraphics[width=0.7\linewidth]{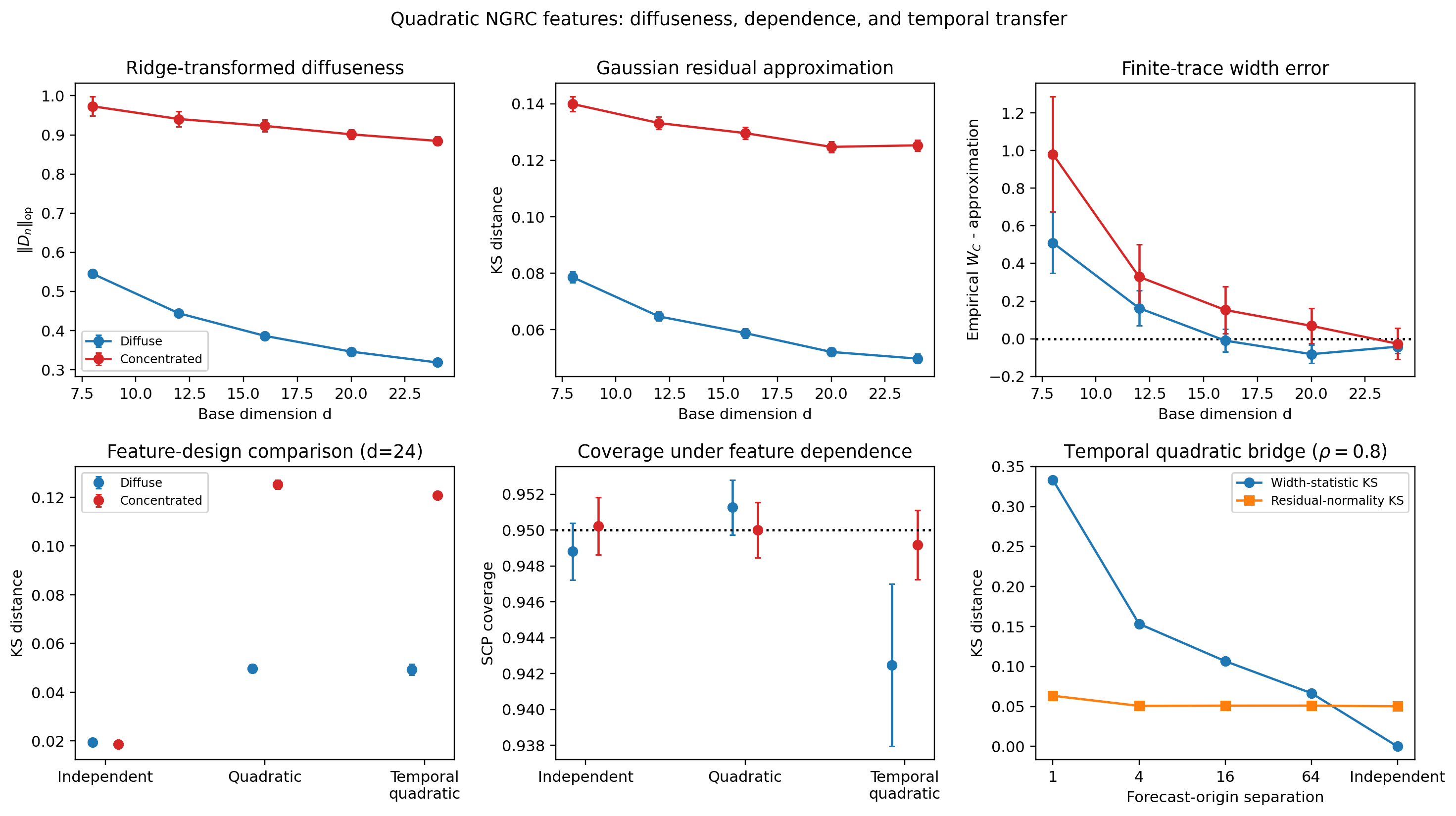}
    \caption{
    Quadratic NGRC diffuseness and temporal transfer.
    Top: ridge-transformed diffuseness, residual Gaussian approximation, and
    finite-trace width error.
    Bottom: feature-design comparisons and convergence toward the
    independent-window benchmark with increasing forecast-origin separation.
    }
    \label{fig:sim_quadratic_dependence}
\end{figure}

\subsection{Experiment 4: Online Robustness in a Nonlinear Volterra System}
\label{subsec:sim_nonlinear_volterra}

We study one-step-ahead forecasting in a nonlinear Volterra
integro-differential system with stochastic observations. The goal is to
compare how uncertainty procedures respond when the observation-noise variance
changes after deployment while the latent trajectory and fitted point
predictor remain fixed. The latent trajectory satisfies
\(
x'(t)
=
\sin(t)-0.1x(t)
+
0.5\tanh\{M(t)\},
\)
with
\(
x(0)=0,
\)
where
\[
M(t)
=
\int_0^t
e^{-0.2(t-s)}
\cos\{2\pi(t-s)\}
x(s)\,ds.
\]
The exponentially damped oscillatory kernel induces fading memory, and the
hyperbolic-tangent transformation makes the dependence on the memory integral
nonlinear. We integrate the equivalent three-dimensional state-space system
using forward Euler with internal step size $h=0.005$ and retain every 20th
simulated state, giving an observation spacing of $\Delta t=0.1$. This
internal step satisfies the Euler stability bound for the oscillatory memory
states; the stability calculation and complete recursion are given in
Appendix~\ref{app:volterra_numerics}. The first 500 retained states are
discarded as burn-in.

Observations satisfy
\(
y_i=x(t_i)+\varepsilon_i
\)
under the constant-variance condition,
\(
\varepsilon_i\sim N(0,0.15^2)
\)
throughout evaluation. Under the variance-change condition,
\[
\varepsilon_i\sim
\begin{cases}
N(0,0.15^2), & i<160,\\
N(0,0.35^2), & i\ge160,
\end{cases}
\]
so the observation-noise standard deviation increases from $0.15$ to $0.35$
at test step 160. We refer to test observations 160 through 799 in this
condition as the \emph{high-noise period}. Within each replication, the
constant-variance and variance-change settings use the same latent trajectory
and standardized noise innovations, differing only in the scale applied after
the change point.

We construct NGRC features consisting of an intercept, 20 linear lag
coordinates, and all upper-triangular quadratic products of those lags,
giving
\(
1+20+\frac{20(21)}{2}=231
\)
features in total. The ridge predictor minimizes
\(
\frac{1}{n}\|y-Xw\|^2+\lambda\|w\|^2,
\)
with $\lambda=0.01$, and remains fixed throughout testing. Each replication
uses 1600 fitting observations, 800 calibration observations, and 800 test
observations.

We compare frozen-scale Bayesian ridge, frozen-quantile split conformal
prediction (SCP), rolling SCP, an updated Gaussian interval, adaptive
conformal prediction, time-weighted conformal prediction, and two Reservoir
Conformal Prediction (ResCP) history policies. Rolling SCP and Updated Gaussian
use the 100 most recently observed forecast residuals; adaptive conformal uses
learning rate $\eta=0.01$; time-weighted conformal uses decay factor $0.98$;
and the common ResCP reservoir and sampling settings are tuned once on an
independent constant-variance pilot sequence and then fixed across evaluation
replications. The all-history ResCP retains every revealed residual, whereas
the shift-aware implementation maintains a fixed FIFO calibration set of
$N=800$ residuals and uses the same linear recency weighting as the
all-history implementation. All online procedures update only with residuals
whose outcomes have already been observed. Exact update rules for the
diagnostic baselines are given in
Appendix~\ref{app:online_uq_updates}.

Results are averaged over 100 paired replications.
Figure~\ref{fig:sim_volterra_variance_change} shows the response during the
high-noise period through rolling coverage and mean interval width.
Table~\ref{tab:sim_volterra_response} reports constant-regime coverage and
width, coverage over the first 50 high-noise observations, coverage and mean
interval width over the complete high-noise period, and recovery diagnostics.

Recovery delay is measured from the noise increase at test step 160. It is the
number of subsequent test observations until coverage over the trailing window
of 100 test observations first enters $[0.93,0.97]$ and remains in that
interval for every remaining test observation. Recovery fraction is the number
of replications satisfying this sustained-recovery criterion. This is a
stringent retrospective stability criterion rather than a direct estimate of
an intrinsic adaptation time: sampling variation can cause an otherwise
well-calibrated procedure to leave the band, and recovery delay is defined only
among replications that eventually satisfy the criterion. We therefore
interpret recovery delay jointly with the corresponding recovery fraction.

\begin{figure*}[t]
    \centering
    \includegraphics[width=0.6\textwidth]{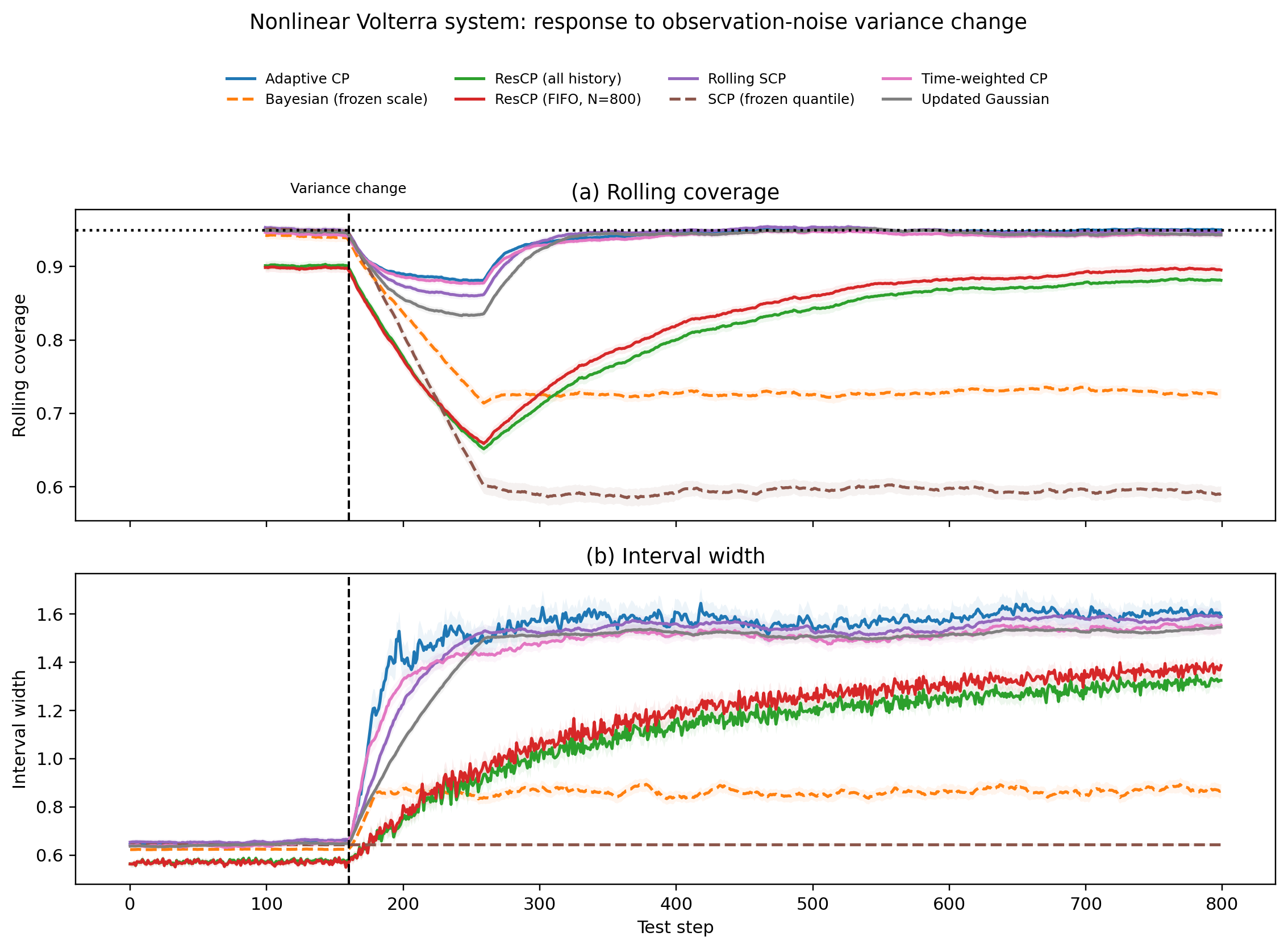}
    \caption{
    Online response to an observation-noise variance change in the nonlinear
    Volterra system. The vertical dashed line marks the increase in
    observation-noise standard deviation from $0.15$ to $0.35$ at test step
    160, and the horizontal dotted line denotes nominal 95\% coverage.
    Panel (a) reports trailing-window coverage and panel (b) reports mean
    interval width. Curves are averages over 100 paired replications and
    shaded regions are 95\% Monte Carlo confidence intervals.
    }
    \label{fig:sim_volterra_variance_change}
\end{figure*}

\begin{table*}[t]
\centering
\caption{
Performance under constant observation-noise variance and after the
noise increase in the nonlinear Volterra system. First-50 high-noise
coverage is calculated over test observations 160--209. Full high-noise
coverage and high-noise width use test observations 160--799. Recovery
fraction is the number of replications satisfying the sustained-recovery
criterion. Observations to recovery counts the number of test observations
from the noise increase until sustained recovery and is reported as mean
$\pm$ standard error among replications that recover. All other entries are
means $\pm$ standard errors over 100 paired replications.
}
\label{tab:sim_volterra_response}
\scriptsize
\setlength{\tabcolsep}{3pt}
\renewcommand{\arraystretch}{1.08}
\resizebox{\textwidth}{!}{%
\begin{tabular}{lccccccc}
\toprule
Method
& \shortstack{Constant\\coverage}
& \shortstack{Constant\\width}
& \shortstack{First-50 high-noise\\coverage}
& \shortstack{Full high-noise\\coverage}
& \shortstack{High-noise\\width}
& \shortstack{Recovery\\fraction}
& \shortstack{Observations from noise\\increase to recovery} \\
\midrule
Bayesian (frozen scale)
& $0.941 \pm 0.001$
& $0.624 \pm 0.001$
& $0.702 \pm 0.005$
& $0.726 \pm 0.001$
& $0.858 \pm 0.003$
& $0/100$
& -- \\
SCP (frozen quantile)
& $0.949 \pm 0.001$
& $0.645 \pm 0.002$
& $0.606 \pm 0.008$
& $0.595 \pm 0.003$
& $0.645 \pm 0.002$
& $0/100$
& -- \\
Rolling SCP
& $0.950 \pm 0.000$
& $0.656 \pm 0.002$
& $0.794 \pm 0.004$
& $0.935 \pm 0.001$
& $1.502 \pm 0.006$
& $62/100$
& $592.6 \pm 5.1$ \\
Updated Gaussian
& $0.947 \pm 0.001$
& $0.643 \pm 0.002$
& $0.756 \pm 0.005$
& $0.929 \pm 0.001$
& $1.458 \pm 0.005$
& $58/100$
& $587.2 \pm 5.5$ \\
Adaptive CP
& $0.949 \pm 0.000$
& $0.650 \pm 0.002$
& $0.832 \pm 0.003$
& $0.938 \pm 0.000$
& $1.544 \pm 0.007$
& $80/100$
& $546.6 \pm 9.8$ \\
Time-weighted CP
& $0.945 \pm 0.000$
& $0.642 \pm 0.002$
& $0.828 \pm 0.004$
& $0.933 \pm 0.000$
& $1.476 \pm 0.006$
& $76/100$
& $574.4 \pm 7.6$ \\
ResCP (all history)
& $0.911 \pm 0.001$
& $0.583 \pm 0.002$
& $0.602 \pm 0.006$
& $0.815 \pm 0.001$
& $1.127 \pm 0.005$
& $11/100$
& $610.3 \pm 13.7$ \\
ResCP (shift-aware FIFO, $N=800$)
& $0.901 \pm 0.001$
& $0.571 \pm 0.002$
& $0.602 \pm 0.006$
& $0.830 \pm 0.001$
& $1.177 \pm 0.005$
& $17/100$
& $603.3 \pm 9.8$ \\
\bottomrule
\end{tabular}
}
\end{table*}

Because the change occurs at test index 160, an absolute recovery index can be
obtained by adding 160 to the reported delay. For example, Adaptive CP's mean
delay of 546.6 corresponds to mean test index 706.6 among recovering
replications, while all-history ResCP's mean delay of 610.3 corresponds to
index 770.3.

Figure~\ref{fig:sim_volterra_variance_change} and
Table~\ref{tab:sim_volterra_response} show substantial differences in
robustness to the observation-noise increase. Under constant variance, most
methods remain close to nominal coverage, although both ResCP variants
undercover in this configuration.

After the variance increase, frozen SCP deteriorates most sharply. Its
calibration quantile remains fixed at the pre-change residual scale, so its
interval width does not respond to the larger forecast errors and full
high-noise coverage falls to $0.595$. Frozen-scale Bayesian ridge also
undercovers persistently. Its interval can vary through the predictive-leverage
term, but its fitted noise scale remains fixed at the pre-change value, so it
does not fully reflect the larger post-change observation variance. Neither
frozen method satisfies the sustained-recovery criterion in any replication.

Both ResCP variants also exhibit substantial initial undercoverage. FIFO
updating gives a modest improvement over the all-history implementation: full
high-noise coverage rises from $0.815$ to $0.830$, and the recovery fraction
rises from $11/100$ to $17/100$, at the cost of increasing high-noise width
from $1.127$ to $1.177$. Thus, restricting the residual history improves
adaptation to the variance change in this experiment, but the improvement is
partial.

In contrast, adaptive CP, rolling SCP, time-weighted CP, and the updated
Gaussian interval widen as recently observed forecast errors reflect the new
noise level and recover much closer to nominal coverage. Adaptive CP has the
largest recovery fraction and the shortest conditional mean recovery delay
among the methods considered, although these recovery statistics should be
interpreted jointly because the delay is defined only among replications that
satisfy the sustained-recovery criterion. The updated Gaussian interval is
particularly informative: once the Gaussian procedure is allowed to replace
its frozen fitting-period scale with a sequential estimate based on recent
forecast errors, its behavior becomes similar to that of the adaptive
conformal procedures. Thus, the principal mechanism in this experiment is not
conformality alone, but access to recently observed forecast errors and
sequential updating of the uncertainty estimate.

\subsection{Experiment 5: Fitting--Calibration Tradeoff Under Controlled Data Budgets}
\label{subsec:sim_data_budget}

Finally, we examine how reserving observations for conformal calibration
affects forecasting performance under a controlled data budget. We consider
\(
y_i=x_i^\top\beta_0+\varepsilon_i
\)
with
\(
x_i\sim N(0,I_p)
\)
and
\(
\varepsilon_i\sim N(0,\sigma^2).
\)
We use a total pre-test budget $B=400$, $\sigma^2=0.01$, and
$\|\beta_0\|^2=1$. We compare a low-dimensional setting with $p=10$
($p/B=0.025$) and a high-dimensional setting with $p=300$
($p/B=0.75$). Each replication uses an independent test set of 2,000
observations.

Throughout this experiment, ridge regression uses the same normalized convention as in the main analysis, with $\lambda=2.5\times10^{-4}$. The normalized penalty is held fixed as the fitting allocation changes, so differences across $\rho$ reflect the
fitting--calibration tradeoff rather than an implicit change in regularization strength.

\paragraph{Equal total budget.}
Bayesian ridge uses all $B=400$ observations for fitting, whereas split
conformal prediction allocates $\rho B$ observations to fitting and
$(1-\rho)B$ observations to calibration, with
\(
\rho\in\{0.5,0.6,0.7,0.8\}.
\)
Within each replication, all allocation fractions use the same pre-test sample
and test set. Results are averaged over 200 paired replications, and we report
coverage, average interval width, Winkler score, and test RMSE.

Figure~\ref{fig:sim_budget_allocation} shows that the allocation cost is small
in low dimension but becomes substantial when the fitting problem is high
dimensional. Increasing $\rho$ improves split-conformal prediction by
providing more observations for fitting, but across the considered allocation
range its interval width, Winkler score, and RMSE remain larger than those of
Bayesian ridge while both procedures maintain near-nominal coverage. Thus,
under the correctly specified stationary model used here, reserving
observations for conformal calibration can carry a meaningful efficiency cost
when the feature dimension is large relative to the available fitting sample.

\begin{figure*}[t]
    \centering
    \includegraphics[width=\textwidth]{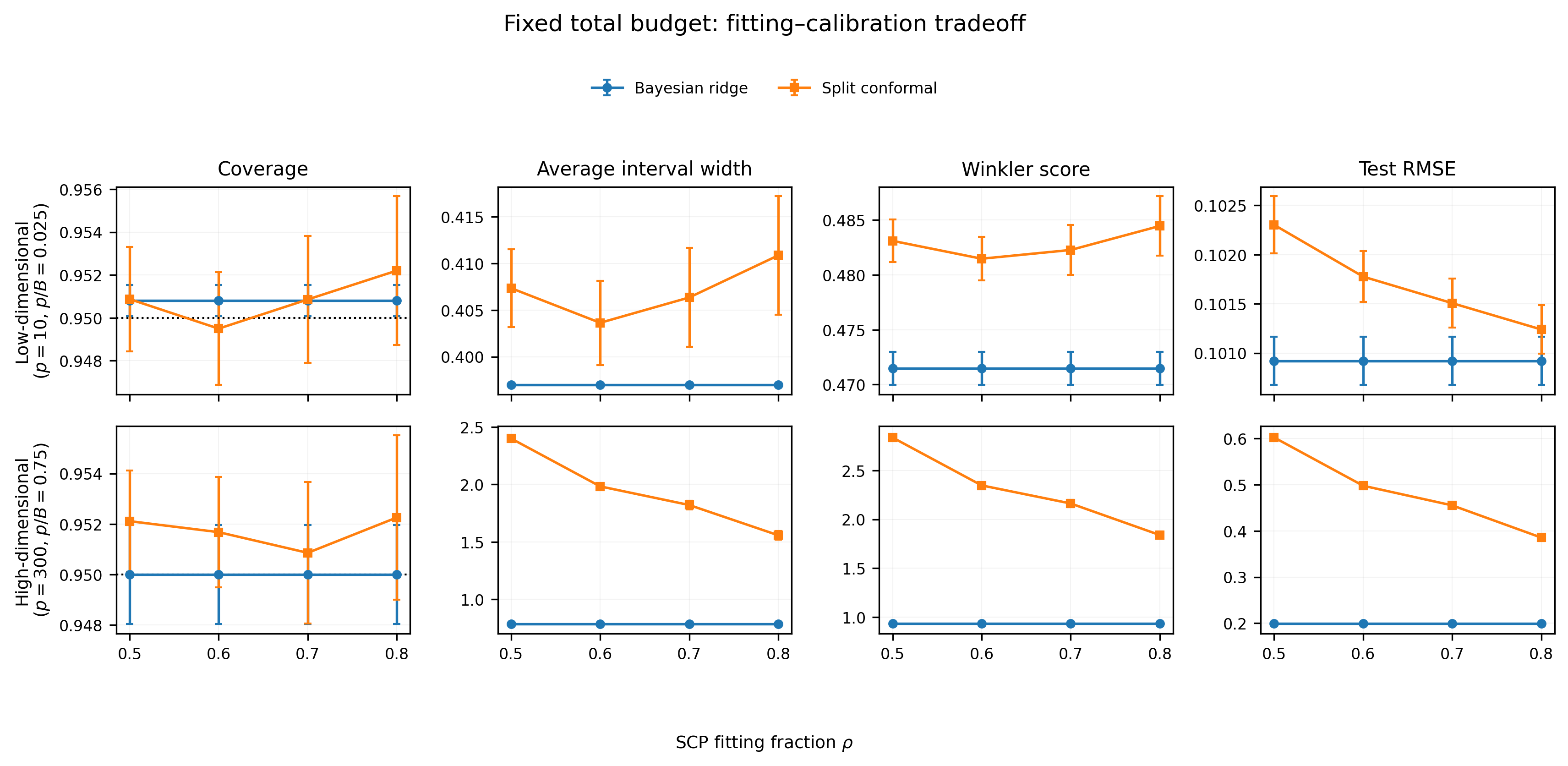}
    \caption{
    Fitting--calibration tradeoff under the fixed total budget $B=400$.
    Bayesian ridge uses all observations for fitting, while split conformal
    varies the fitting fraction $\rho$. Columns show coverage, interval width,
    Winkler score, and test RMSE for $p=10$ and $p=300$. Results are means
    over 200 paired replications with 95\% confidence intervals.}
    \label{fig:sim_budget_allocation}
\end{figure*}
\section{Real Data Analysis}
\label{sec:real_data}

This section illustrates how residual shape, finite-sample regularization,
temporal changes in forecast-error scale, and online adaptation interact in
practical forecasting. Across the four applications below,
$p/n_{\rm fit}$ is between 0.024 and 0.040 and
$d_{\rm eff}/n_{\rm fit}$ is between 0.004 and 0.029, placing these
applications much closer to a low- or effective-low-dimensional regime than
to the proportional-growth setting studied in Section~\ref{sec:theory}.

We study LA ozone, Solar power, Beijing PM$_{10}$, and Exchange returns. The
series contain 16,071 daily, 52,560 ten-minute, 35,064 hourly, and 7,587 daily
observations, respectively. Our primary uncertainty-quantification comparison
holds the NGRC point forecast fixed. A secondary, explicitly exploratory
comparison holds the uncertainty procedure fixed and varies the point
forecaster among NGRC, ARIMA, a GRU, and a Transformer.

The primary NGRC specification uses lagged observations and quadratic delay
features, with seasonal Fourier terms included for Solar and Beijing
PM$_{10}$. Dataset-specific preprocessing, lag choices, feature construction,
and tuning details are given in Appendix~\ref{app:real_details}.

\paragraph{Chronological protocol and leakage controls.}
Our primary fixed-origin analysis uses a chronological 40\% fitting,
40\% calibration, and 20\% test split. As a sensitivity analysis, we also
use three shifted 40\% fit/20\% calibration/20\% test windows with starting
points at 0\%, 10\%, and 20\% of the available feature rows.

Scaling parameters and ridge penalties are estimated using only the fitting
block of each evaluation window. Missing observations are filled only by
carrying the most recent past observation forward, and an unavailable leading
prefix is discarded. Fourier features are deterministic functions of the time
index and contain no parameters estimated from future observations. A row is
assigned to a split according to its target time, and a direct $H$-step
feature vector contains observations only through time $t-H$ for a target at
time $t$. For sequential procedures, a forecast residual enters the online
state only after its target has been observed, producing an $H$-step feedback
delay.

\subsection{Connection to Theoretical Results}
\label{subsec:real_diagnostics}

The empirical applications are closer to the fixed- or
effective-low-dimensional setting than to the proportional-growth regime.
They do not, however, satisfy the assumptions of
Theorem~\ref{thm:fixed_width_discrepancy} literally. In particular, the
empirical ridge penalties are selected from the fitting data and need not
follow the vanishing-regularization sequence assumed in the theorem.
Moreover, the fitting, calibration, and test periods exhibit changes in their
forecast-error distributions. We therefore use the fixed-dimensional theory
as a qualitative guide for interpreting residual shape, residual scale,
regularization, and interval behavior rather than as a literal asymptotic
model for these datasets.

The theoretical ridge calculations penalize every readout coordinate, whereas
the empirical NGRC fits leave the intercept unpenalized. Accordingly, define
\(D=\operatorname{diag}(0,1,\ldots,1)\) and
\(
d_{\rm eff}
=
\operatorname{tr}\!\left\{
S(S+\lambda D)^{-1}
\right\},
\)
with
\(S=R^\top R/n_{\rm fit}.\)
This quantity is the trace of the smoother corresponding to the actual
empirical ridge fit. For each dataset,
Table~\ref{tab:real_theory_diagnostics} reports $p/n_{\rm fit}$, the
regularization-adjusted ratio $d_{\rm eff}/n_{\rm fit}$, residual RMS scales
in the fitting, calibration, and test blocks, the calibration residual
quantile $\widehat q_{1-\alpha,\rm cal}$, and
\(
\rho_q=
\widehat q_{1-\alpha,\rm cal}/
\{z_{1-\alpha/2}\widehat\tau_{\rm fit}\}.
\)
The last quantity compares the empirical conformal residual radius with the
Gaussian radius based on the fitting-period RMS scale. Values near one
indicate similar radii, values above one favor a wider SCP residual radius,
and values below one favor a wider Bayesian residual radius. Predictive
leverage and finite-sample corrections still affect the final Bayesian
interval widths, so $\rho_q$ is a diagnostic rather than an identity for the
interval-width ratio.

\begin{table*}[t]
\centering
\small
\caption{Theory-facing diagnostics for the one-step, 95\% fixed-origin
analysis. RMS scales and residual quantiles are on the standardized-response
scale.}
\label{tab:real_theory_diagnostics}
\resizebox{\textwidth}{!}{%
\begin{tabular}{lrrrrrrrr}
\toprule
Dataset
& $n_{\rm fit}$
& $p/n_{\rm fit}$
& $d_{\rm eff}/n_{\rm fit}$
& $\widehat{\tau}_{\rm fit}$
& $\widehat{\tau}_{\rm cal}$
& $\widehat{\tau}_{\rm test}$
& $\widehat{q}_{.95,\rm cal}$
& $\widehat{q}/(z\widehat{\tau}_{\rm fit})$ \\
\midrule
LA ozone
& 6420
& 0.0394
& 0.0287
& 0.521
& 0.392
& 0.360
& 0.787
& 0.771 \\
Solar
& 21009
& 0.0338
& 0.0038
& 0.029
& 0.026
& 0.031
& 0.057
& 1.001 \\
Beijing PM$_{10}$
& 14016
& 0.0238
& 0.0116
& 0.231
& 0.248
& 0.182
& 0.388
& 0.858 \\
Exchange returns
& 3029
& 0.0396
& 0.0038
& 0.993
& 1.445
& 1.284
& 2.704
& 1.389 \\
\bottomrule
\end{tabular}%
}
\end{table*}

The diagnostic combines two effects. Exactly,
\(
\rho_q=
\{\widehat q_{1-\alpha,\rm cal}/
(z_{1-\alpha/2}\widehat\tau_{\rm cal})\}
\{\widehat\tau_{\rm cal}/\widehat\tau_{\rm fit}\},
\)
so the first factor describes the shape of the calibration residual
distribution relative to a Gaussian RMS benchmark, while the second measures a
fit-to-calibration change in residual scale. The corresponding values are
reported in Table~\ref{tab:real_diagnostic_decomposition}.

\begin{table}[t]
\centering
\small
\caption{Decomposition of the residual diagnostic at 95\% nominal coverage.}
\label{tab:real_diagnostic_decomposition}
\begin{tabular}{lrrr}
\toprule
Dataset & Shape factor & Scale factor & $\rho_q$ \\
\midrule
LA ozone & 1.024 & 0.752 & 0.771 \\
Solar & 1.119 & 0.897 & 1.001 \\
Beijing PM$_{10}$ & 0.798 & 1.074 & 0.858 \\
Exchange returns & 0.955 & 1.455 & 1.389 \\
\bottomrule
\end{tabular}
\end{table}

LA ozone and Beijing have $\rho_q=0.771$ and $0.858$, respectively, so the
fit-block Gaussian RMS scale is large relative to the calibration residual
quantile and the Bayesian intervals are correspondingly wider. Solar has
$\rho_q=1.001$, indicating nearly equal Gaussian and conformal residual radii,
while Exchange returns has $\rho_q=1.389$, for which SCP is wider. The
differences among $\widehat\tau_{\rm fit}$,
$\widehat\tau_{\rm cal}$, and $\widehat\tau_{\rm test}$ also provide direct
evidence of temporal changes in the forecast-error scale.

The decomposition sharpens this interpretation. For LA ozone, the
within-calibration shape factor is close to one while the calibration RMS is
substantially smaller than the fitting RMS, so the Bayesian--SCP difference is
driven mainly by a between-block scale change. For Exchange returns, the shape
factor is again close to one, whereas the calibration RMS is about 46\% larger
than the fitting RMS, making temporal scale change the dominant contribution
to $\rho_q>1$. Beijing PM$_{10}$ shows a more substantial within-block shape
effect, while Solar exhibits offsetting shape and scale contributions. Thus,
the empirical results are consistent with the broader residual-distribution
interpretation developed in Section~\ref{subsec:fixed_theory}, but they do not
constitute a clean empirical isolation of the stationary residual-shape
mechanism in Theorem~\ref{thm:fixed_width_discrepancy}.

\subsection{NGRC with Alternative Uncertainty Procedures}
\label{subsec:real_ngrc_uq}

We compare Bayesian ridge, symmetric split conformal prediction (SCP), an
equal-tail asymmetric SCP baseline based on signed residuals, adaptive CP,
time-weighted CP, and ResCP. All six procedures use the same fitted NGRC point
forecast within a dataset and evaluation origin. The asymmetric baseline
isolates the effect of signed-residual asymmetry without reservoir
localization, temporal weighting, or adaptive calibration.

\begin{table*}[t]
\centering
\small
\caption{Fixed-origin one-step NGRC uncertainty-quantification results for
the 40\% fit/40\% calibration/20\% test split at 95\% nominal coverage.
All methods within a dataset use the same NGRC point forecasts, so RMSE is
reported once per dataset. The signed coverage error is
$\Delta_{\rm cov}=100(\widehat{\mathrm{Cov}}-0.95)$ percentage points,
computed from unrounded empirical coverage.
UQ time is the wall-clock time for method-specific calibration, tuning,
sequential updating, and interval construction; it excludes the shared NGRC
fit. Boldface marks the coverage closest to 0.95 and the lowest Winkler score
within each dataset.}
\label{tab:real_main}
\resizebox{\textwidth}{!}{%
\begin{tabular}{llrrrrrr}
\toprule
Dataset & Method & RMSE & Coverage & $\Delta_{\rm cov}$ (pp)
& Width & Winkler & UQ time (s) \\
\midrule

\multirow{6}{*}{LA ozone}
& Bayesian ridge
& \multirow{6}{*}{0.360}
& 0.993 & 4.25 & 2.047 & 2.083 & 0.015 \\
& SCP
& & 0.963 & 1.29 & 1.573 & 1.795 & $<0.001$ \\
& Asymmetric SCP
& & 0.963 & 1.29 & 1.574 & 1.791 & 0.001 \\
& Adaptive CP
& & \textbf{0.954} & 0.36 & 1.485 & 1.784 & 6.523 \\
& Time-weighted CP
& & \textbf{0.954} & 0.36 & 1.477 & \textbf{1.782} & 8.677 \\
& ResCP
& & 0.929 & -2.13 & 1.369 & 1.797 & 18.831 \\
\midrule

\multirow{6}{*}{Solar}
& Bayesian ridge
& \multirow{6}{*}{0.031}
& 0.941 & -0.94 & 0.115 & 0.205 & 0.115 \\
& SCP
& & 0.940 & -0.98 & 0.115 & 0.205 & 0.001 \\
& Asymmetric SCP
& & 0.933 & -1.70 & 0.112 & 0.209 & 0.002 \\
& Adaptive CP
& & \textbf{0.943} & -0.66 & 0.109 & 0.167 & 73.168 \\
& Time-weighted CP
& & 0.857 & -9.31 & 0.113 & 0.176 & 99.132 \\
& ResCP
& & 0.961 & 1.08 & 0.058 & \textbf{0.075} & 256.402 \\
\midrule

\multirow{6}{*}{Beijing PM$_{10}$}
& Bayesian ridge
& \multirow{6}{*}{0.182}
& 0.971 & 2.10 & 0.909 & 1.174 & 0.025 \\
& SCP
& & 0.960 & 1.00 & 0.776 & 1.138 & 0.001 \\
& Asymmetric SCP
& & 0.960 & 1.05 & 0.775 & 1.138 & 0.001 \\
& Adaptive CP
& & \textbf{0.947} & -0.32 & 0.726 & 1.048 & 32.816 \\
& Time-weighted CP
& & 0.936 & -1.36 & 0.677 & 1.030 & 45.449 \\
& ResCP
& & 0.954 & 0.43 & 0.656 & \textbf{0.933} & 125.954 \\
\midrule

\multirow{6}{*}{Exchange returns}
& Bayesian ridge
& \multirow{6}{*}{1.284}
& 0.946 & -0.41 & 3.981 & 6.468 & 0.002 \\
& SCP
& & 0.983 & 3.35 & 5.407 & 6.997 & $<0.001$ \\
& Asymmetric SCP
& & 0.983 & 3.35 & 5.391 & 6.999 & $<0.001$ \\
& Adaptive CP
& & \textbf{0.954} & 0.38 & 4.015 & 6.396 & 1.420 \\
& Time-weighted CP
& & 0.939 & -1.14 & 4.281 & \textbf{6.242} & 1.846 \\
& ResCP
& & 0.934 & -1.60 & 3.793 & 6.640 & 4.890 \\
\bottomrule
\end{tabular}%
}
\end{table*}

Table~\ref{tab:real_main} shows substantial dataset dependence in both
calibration and interval efficiency. Adaptive CP remains comparatively close
to nominal coverage across all four applications. ResCP achieves particularly
low Winkler scores for Solar and Beijing PM$_{10}$, whereas its narrower ozone
and Exchange intervals are accompanied by coverage shortfalls. Symmetric and
asymmetric SCP are nearly indistinguishable, indicating that signed-residual
asymmetry alone does not explain the larger differences between SCP and ResCP.
The Bayesian--SCP width ordering also agrees qualitatively with the
residual-scale diagnostics in
Table~\ref{tab:real_theory_diagnostics}. These comparisons are descriptive:
because the empirical penalties are data-selected and residual distributions
change across chronological blocks, we do not interpret them as direct
validation of Theorem~\ref{thm:fixed_width_discrepancy}.

Computational costs differ substantially across procedures. Bayesian ridge,
ordinary SCP, and asymmetric SCP add little overhead after fitting the common
NGRC predictor. Adaptive CP, time-weighted CP, and ResCP require sequential
processing and, where applicable, pre-test hyperparameter selection and are
therefore substantially more expensive. The runtimes in
Table~\ref{tab:real_main} measure only the uncertainty-quantification stage.
They should not be compared directly with the end-to-end runtimes in
Table~\ref{tab:real_rescp}, which also include preprocessing, point-model
fitting, forecasting, ResCP tuning, and interval construction.

\paragraph{ResCP comparison.}
ResCP differs from ordinary SCP through its reservoir-state representation,
state-dependent localization and weighting of past residual information, and
sequential updating. The near agreement between symmetric and asymmetric SCP
suggests that signed-residual asymmetry alone does not explain the larger
differences observed for ResCP. This diagnostic does not separately identify
the contributions of reservoir localization, weighting, and sequential
adaptation.

\paragraph{Repeated chronological robustness.}
To assess sensitivity to the particular fixed-origin test period, we repeat
the one-step comparison over three shifted chronological evaluation windows.
Because these windows overlap and use a shorter calibration block than the
primary analysis, the resulting standard deviations quantify sensitivity to
the forecast period rather than sampling uncertainty and should not be
compared row-for-row with Table~\ref{tab:real_main}. Adaptive CP remains
comparatively stable across forecast periods, whereas the Exchange-return
results exhibit substantially larger between-period variation, highlighting
the practical importance of temporal shift. Complete method-by-dataset results
are reported in Appendix~\ref{app:real_rolling}.

\paragraph{Forecast-horizon robustness.}
We additionally evaluate direct horizons $H\in\{1,3,7\}$ using the same
availability-safe delayed-update protocol. A forecast residual is incorporated
into a sequential method only after the corresponding target has become
observable. Adaptive CP remains comparatively stable across horizons, while
Solar and Exchange returns show more pronounced deterioration for some methods
as the forecast horizon increases. Complete coverage, width, Winkler-score,
and runtime results are reported in Appendix~\ref{app:real_horizon}.

\paragraph{Calibration across nominal levels.}
We also evaluate empirical calibration at nominal coverage levels
$80\%$, $85\%$, $90\%$, $95\%$, and $97.5\%$. No procedure is uniformly
calibrated across all datasets and nominal levels; the complete calibration
curves are reported in Appendix~\ref{app:real_calibration}.

\paragraph{Joint coverage and efficiency.}
Figure~\ref{fig:real_efficiency} plots coverage error against both mean
interval width and Winkler score for the fixed-origin analysis. Points to the
left of zero undercover, so small interval width for such methods should not
be interpreted as greater efficiency at comparable coverage. In particular,
ResCP's relatively narrow ozone and Exchange intervals accompany coverage
shortfalls, whereas its Solar and Beijing intervals combine relatively small
width with coverage closer to the target. The Winkler score provides a
complementary summary by penalizing both interval width and misses.

\begin{figure*}[t]
  \centering
  \includegraphics[width=\textwidth]{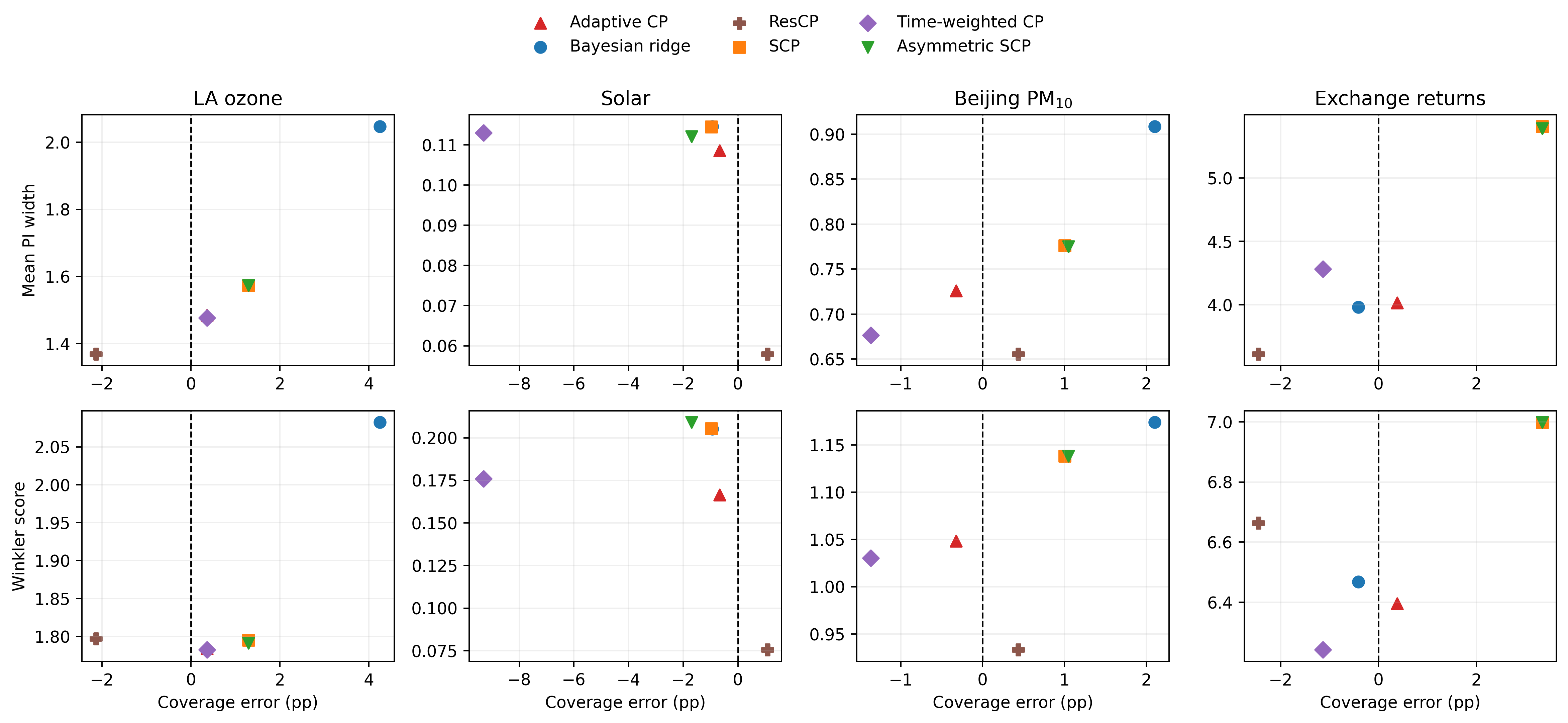}
  \caption{Coverage error versus mean interval width (top) and Winkler score
  (bottom) at 95\% nominal coverage. The vertical line marks exact marginal
  coverage. All methods use the same NGRC point forecasts.}
  \label{fig:real_efficiency}
\end{figure*}

\subsection{Forecasting Models with ResCP}
\label{subsec:real_rescp_models}

As a secondary exploratory analysis, we compare NGRC, ARIMA, GRU, and
Transformer point forecasters under a common ResCP uncertainty layer. The
models use the same target transformation, chronological evaluation blocks,
and lag or history budget where their architectures permit it. Because the
model classes and optimization problems cannot be made identical, this
comparison should be interpreted as an end-to-end pipeline comparison rather
than as an isolated test of forecasting architecture. In particular, the GRU
and Transformer follow the MAE training objective used in the ResCP benchmark
and the same chronological early-stopping budget, so RMSE is a descriptive
evaluation metric rather than an objective optimized equally by every model.

\begin{table*}[t]
\centering
\small
\caption{Comparison of point-forecasting architectures under the common
official ResCP uncertainty procedure for one-step-ahead forecasting at 95\%
nominal coverage. Runtime includes preprocessing, point-model fitting,
forecasting, ResCP tuning, and interval construction.}
\label{tab:real_rescp}
\resizebox{\textwidth}{!}{%
\begin{tabular}{llrrrrr}
\toprule
Dataset & Forecasting pipeline & RMSE & Coverage & PI width & Winkler
& Total runtime (s) \\
\midrule
\multirow{4}{*}{LA ozone}
& ARIMA + ResCP & 0.359 & 0.941 & 1.439 & 1.834 & 26.190 \\
& GRU + ResCP & 0.365 & 0.956 & 1.524 & 1.810 & 60.960 \\
& Transformer + ResCP & 0.364 & 0.935 & 1.391 & 1.841 & 141.942 \\
& NGRC + ResCP & 0.360 & 0.929 & 1.369 & 1.797 & 19.390 \\
\midrule
\multirow{4}{*}{Solar}
& ARIMA + ResCP & 0.021 & 0.904 & 0.038 & 0.069 & 156.337 \\
& GRU + ResCP & 0.020 & 0.894 & 0.041 & 0.070 & 442.007 \\
& Transformer + ResCP & 0.029 & 0.907 & 0.041 & 0.067 & 1102.612 \\
& NGRC + ResCP & 0.031 & 0.961 & 0.058 & 0.075 & 263.522 \\
\midrule
\multirow{4}{*}{Beijing PM$_{10}$}
& ARIMA + ResCP & 0.178 & 0.957 & 0.651 & 0.904 & 157.428 \\
& GRU + ResCP & 0.184 & 0.954 & 0.646 & 0.924 & 160.148 \\
& Transformer + ResCP & 0.204 & 0.956 & 0.671 & 0.919 & 351.799 \\
& NGRC + ResCP & 0.182 & 0.954 & 0.656 & 0.933 & 125.263 \\
\midrule
\multirow{4}{*}{Exchange returns}
& ARIMA + ResCP & 1.276 & 0.941 & 3.804 & 6.515 & 10.705 \\
& GRU + ResCP & 1.275 & 0.934 & 3.721 & 6.727 & 9.687 \\
& Transformer + ResCP & 1.276 & 0.921 & 3.625 & 6.684 & 24.113 \\
& NGRC + ResCP & 1.284 & 0.934 & 3.793 & 6.640 & 5.211 \\
\bottomrule
\end{tabular}%
}
\end{table*}

Table~\ref{tab:real_rescp} shows that changing the forecasting architecture
does not produce a uniform improvement in uncertainty performance under the
common ResCP procedure. For LA ozone and Exchange returns, RMSE is similar
across the four models, yet coverage and interval width differ, showing that
similar point-forecast accuracy need not imply similar uncertainty
performance. For Solar, ARIMA and GRU achieve lower RMSE than NGRC but
substantially lower coverage, whereas for Beijing PM$_{10}$ the four
architectures produce relatively similar coverage and interval widths.

Computational cost also differs substantially across architectures. The
Transformer pipelines are considerably more expensive on several datasets
without consistently improving coverage or Winkler score. We therefore do not
interpret these results as showing that NGRC is generally superior to
recurrent or Transformer models. Rather, point-forecast error, residual
behavior, local calibration, and computational cost interact, and the relative
performance of the complete pipelines depends on both the dataset and the
evaluation metric.

Overall, the real-data analyses illustrate several of the mechanisms
identified by the theory without assuming that the theoretical
independent-window models hold exactly. Residual shape and changes in residual
scale are informative for the Bayesian--SCP width ordering, while adaptive
recalibration can be useful when forecast-error behavior changes over time.
The forecasting-model comparison further shows that similar point-forecast
accuracy can lead to different uncertainty performance. At the same time,
narrow intervals must be evaluated jointly with their coverage and interval
score. Calibration and interval efficiency are therefore assessed directly on
chronologically ordered test observations rather than inferred from the
asymptotic theory.
\section{Conclusion}\label{sec:conclusion}

We present a statistical uncertainty quantification framework for next generation reservoir computing in nonlinear time series with memory. NGRC provides a computationally efficient forecasting backbone because the nonlinear delay representation is fixed and only the output readout is trained. Building uncertainty quantification on top of this readout makes it possible to compare model-based and conformal intervals within a common forecasting architecture.

Our theoretical results characterize when the Bayesian and conformal uncertainty intervals agree and what governs their differences. In the fixed-dimensional regime, the Gaussian readout interval and the split-conformal interval built on the same NGRC ridge readout converge to widths $2z\tau$ and $2q_e$, determined by the root mean square and the $(1-\alpha)$ quantile of the absolute population residual, without requiring the conditional mean to be linear in the NGRC features or the feature covariance to be invertible. The two procedures are asymptotically equivalent if and only if $q_e=z\tau$, a residual-distributional matching condition satisfied by Gaussian errors but not in general. Low dimensionality alone only makes the Bayesian parameter-uncertainty contribution vanish, leaving the comparison to the shape of the residual distribution. In proportional regimes, the bias--variance--posterior decomposition shows that the Bayesian width carries observation noise plus posterior parameter uncertainty, whereas the conformal width reflects observation noise, frequentist estimation variance, and ridge prediction bias. Because the bias term depends on the signal, there is no universal ordering between the two widths. For isotropic designs, the Bayesian interval is wider when $s^2\lambda<\sigma^2\gamma$, narrower when $s^2\lambda>\sigma^2\gamma$, and the two widths coincide on the boundary $s^2\lambda=\sigma^2\gamma$, where the prior scale matches the signal scale and the ridge penalty equals its risk-optimal value $\lambda^{*}=\sigma^2\gamma/s^2$. When the Bayesian interval is narrower, the reduced width reflects asymptotic undercoverage rather than improved efficiency at equal coverage. The quadratic-feature analysis extends these width approximations to an explicit nonlinear NGRC benchmark with dependent polynomial coordinates, with the deterministic extension requiring diffuseness of the ridge-transformed quadratic bias direction rather than the raw signal itself. The coupling argument of Appendix~\ref{app:dependent_transfer} further shows that independent-window conclusions can be transferred approximately to sufficiently separated forecast origins of an absolutely regular series; it is a sufficient transfer principle rather than an approximation for ordinary overlapping NGRC windows.

From a practical perspective, the feature-to-sample ratio helps determine when the choice of uncertainty procedure has first-order consequences, but it does not by itself determine which interval is wider or better calibrated. When $p/n$ is small, the two procedures differ asymptotically through the residual distribution. With approximately Gaussian forecast errors, they can be interchangeable to first order, whereas non-Gaussian residuals can leave a persistent width and coverage discrepancy. When $p/n$ is non-negligible, the comparison depends jointly on signal strength, noise level, and regularization. In the isotropic proportional model, Bayesian and conformal widths coincide at the prior-matched, risk-optimally regularized boundary. Away from that boundary, the Bayesian interval is conservative in the weak-signal regime and anticonservative in the strong-signal regime. It also requires a reliable noise scale, which the training residuals need not provide in proportional dimension. Split conformal prediction, by calibrating directly on held-out forecast errors, retains asymptotic nominal marginal coverage in the independent benchmark throughout the phase diagram without requiring a noise-scale estimate. Residual diagnostics in low dimension and the signal-to-regularization balance in high dimension therefore complement the ratio $p/n$ in determining when the two procedures can be treated interchangeably.

The simulation experiments support the theoretical conclusions while highlighting several practical tradeoffs. In fixed dimensions, Bayesian and conformal interval widths can differ because of the residual distribution, while in high dimensions their ordering depends on the balance between signal strength, regularization, and the feature-to-sample-size ratio. The quadratic-feature experiments show that decreasing ridge-transformed operator norm is accompanied by improved residual Gaussianity, while increasing forecast-origin separation moves the
dependent-window width distribution toward its same-marginal independent-window benchmark. The smaller separations, including overlapping windows, should be viewed as empirical stress tests rather than configurations covered by the transfer result. In the Volterra experiment, methods that update their uncertainty estimates using recently observed forecast errors are more robust to the variance shift than frozen Bayesian or split-conformal intervals. The similar behavior of the updated Gaussian baseline and the adaptive conformal procedures indicates that sequential uncertainty updating is an important mechanism in this setting. The data-budget experiment further shows that, under a correctly specified stationary model, split conformal calibration can incur an efficiency cost when holding the total pre-test sample size fixed because part of the sample must be reserved for calibration.

The real-data experiments show substantial dataset dependence in uncertainty performance. Adaptive conformal prediction is comparatively stable and often remains close to nominal coverage across the real-data evaluations. ResCP can achieve lower Winkler scores or narrower intervals in some datasets, but these gains are not uniform and can be accompanied by undercoverage. The Bayesian--SCP width ordering is qualitatively consistent with the residual diagnostics, although those diagnostics combine residual shape with fit-to-calibration changes in residual scale and therefore should not be interpreted as direct validation of the stationary fixed-dimensional theory. Holding ResCP fixed across NGRC, ARIMA, recurrent neural networks, and Transformers also shows that similar point-forecast accuracy can lead to different uncertainty performance. Computational cost differs substantially across procedures and forecasting models: Bayesian ridge and split conformal add little overhead once the NGRC predictor is fitted, whereas adaptive, time-weighted, and ResCP procedures require substantially more computation because of sequential processing and tuning. Overall, calibration, interval efficiency, point accuracy, temporal adaptation, and computational cost must be considered jointly when comparing forecasting pipelines.

Several directions remain open. Future work could extend the theory to multivariate and multi-horizon NGRC forecasting, where uncertainty must account for cross-coordinate dependence and horizon-dependent error propagation. Another direction is to develop sharper finite-sample coverage guarantees for conformal NGRC intervals under temporal dependence, nonstationarity, and distribution shift. Finally, the empirical results suggest the need for principled, training-only procedures to select lag length, nonlinear degree, seasonal features, and memory summaries. Overall, NGRC combined with principled uncertainty quantification provides a flexible and computationally efficient framework for nonlinear forecasting with memory. More broadly, our analysis identifies when Bayesian and conformal intervals are approximately interchangeable and when residual shape, regularization, dimensionality, signal strength, or temporal shift make the choice of uncertainty procedure consequential.

\newpage 

\bibliographystyle{abbrvnat}
\bibliography{bib}

@article{nextgeneration,
  author  = {Gauthier, Daniel J. and Bollt, Erik M. and Griffith, Aaron and Barbosa, Wendson A. S.},
  title   = {Next generation reservoir computing},
  journal = {Nature Communications},
  volume  = {12},
  pages   = {5564},
  year    = {2021},
  doi     = {10.1038/s41467-021-25801-2}
}

@article{Bollt_2021,
   title={On explaining the surprising success of reservoir computing forecaster of chaos? The universal machine learning dynamical system with contrast to VAR and DMD},
   volume={31},
   ISSN={1089-7682},
   url={http://dx.doi.org/10.1063/5.0024890},
   DOI={10.1063/5.0024890},
   number={1},
   journal={Chaos: An Interdisciplinary Journal of Nonlinear Science},
   publisher={AIP Publishing},
   author={Bollt, Erik},
   year={2021},
   month=jan }

@article{xu2023conformal,
  title={Conformal prediction for time series},
  author={Xu, Chen and Xie, Yao},
  journal={IEEE transactions on pattern analysis and machine intelligence},
  volume={45},
  number={10},
  pages={11575--11587},
  year={2023},
  publisher={IEEE}
}

@article{shafer2008tutorial,
  title={A tutorial on conformal prediction.},
  author={Shafer, Glenn and Vovk, Vladimir},
  journal={Journal of Machine Learning Research},
  volume={9},
  number={3},
  year={2008}
}

@inproceedings{zaffran2022adaptive,
  title={Adaptive conformal predictions for time series},
  author={Zaffran, Margaux and F{\'e}ron, Olivier and Goude, Yannig and Josse, Julie and Dieuleveut, Aymeric},
  booktitle={International Conference on Machine Learning},
  pages={25834--25866},
  year={2022},
  organization={PMLR}
}

@article{gibbs2021adaptive,
  title={Adaptive conformal inference under distribution shift},
  author={Gibbs, Isaac and Candes, Emmanuel},
  journal={Advances in Neural Information Processing Systems},
  volume={34},
  pages={1660--1672},
  year={2021}
}

@article{Aram2015IntracorticalConnectivity,
  author  = {Aram, P. and Freestone, D. R. and Cook, M. J. and Kadirkamanathan, V. and Grayden, D. B.},
  title   = {Model-based estimation of intra-cortical connectivity using electrophysiological data},
  journal = {NeuroImage},
  volume  = {118},
  pages   = {563--575},
  year    = {2015},
  month   = sep,
  doi     = {10.1016/j.neuroimage.2015.06.048}
}

@article{Pakkanen2023IncidencePrevalence,
  author  = {Pakkanen, Mikko S. and Miscouridou, Xenia and Penn, Matthew J.
             and Whittaker, Charles and Berah, Tresnia and Mishra, Swapnil
             and Mellan, Thomas A. and Bhatt, Samir},
  title   = {Unifying incidence and prevalence under a time-varying general
             branching process},
  journal = {Journal of Mathematical Biology},
  volume  = {87},
  number  = {2},
  pages   = {35},
  year    = {2023},
  doi     = {10.1007/s00285-023-01958-w}
}

@article{Patlashenko2001VolterraTimestepping,
  author  = {Patlashenko, Igor and Givoli, Dan and Barbone, Paul},
  title   = {Time-stepping schemes for systems of Volterra integro-differential equations},
  journal = {Computer Methods in Applied Mechanics and Engineering},
  volume  = {190},
  number  = {43--44},
  pages   = {5691--5718},
  year    = {2001},
  doi     = {10.1016/S0045-7825(01)00192-X}
}

@article{HanWong2021VolterraHeston,
  author  = {Han, Bingyan and Wong, Hoi Ying},
  title   = {Merton's portfolio problem under Volterra Heston model},
  journal = {Finance Research Letters},
  volume  = {39},
  year    = {2021},
  pages   = {101580},
  doi     = {10.1016/j.frl.2020.101580}
}

@inproceedings{Neglia2026ResCP,
  author    = {Roberto Neglia and Andrea Cini and Michael M. Bronstein and Filippo Maria Bianchi},
  title     = {ResCP: Reservoir Conformal Prediction for Time Series Forecasting},
  booktitle = {International Conference on Learning Representations (ICLR)},
  year      = {2026}
}

@article{GneitingBalabdaouiRaftery2007,
  author  = {Tilmann Gneiting and Fadoua Balabdaoui and Adrian E. Raftery},
  title   = {Probabilistic Forecasts, Calibration and Sharpness},
  journal = {Journal of the Royal Statistical Society: Series B (Statistical Methodology)},
  volume  = {69},
  number  = {2},
  pages   = {243--268},
  year    = {2007},
  doi     = {10.1111/j.1467-9868.2007.00587.x}
}

@article{GrigoryevaTingOrtega2025,
  author  = {Lyudmila Grigoryeva and Hannah Lim Jing Ting and Juan-Pablo Ortega},
  title   = {Infinite-dimensional next-generation reservoir computing},
  journal = {Physical Review E},
  volume  = {111},
  number  = {3},
  pages   = {035305},
  year    = {2025},
  doi     = {10.1103/PhysRevE.111.035305}
}

@article{DobribanWager2018,
  author  = {Edgar Dobriban and Stefan Wager},
  title   = {High-dimensional asymptotics of prediction: Ridge regression and classification},
  journal = {The Annals of Statistics},
  volume  = {46},
  number  = {1},
  pages   = {247--279},
  year    = {2018},
  doi     = {10.1214/17-AOS1549}
}

@article{MarchenkoPastur1967,
  author  = {V. A. Mar{\v c}enko and L. A. Pastur},
  title   = {Distribution of eigenvalues for some sets of random matrices},
  journal = {Mathematics of the USSR-Sbornik},
  volume  = {1},
  number  = {4},
  pages   = {457--483},
  year    = {1967},
  doi     = {10.1070/SM1967v001n04ABEH001994}
}

@article{mackay1992bayesian,
  author  = {MacKay, David J. C.},
  title   = {Bayesian Interpolation},
  journal = {Neural Computation},
  volume  = {4},
  number  = {3},
  pages   = {415--447},
  year    = {1992},
  doi     = {10.1162/neco.1992.4.3.415}
}

@article{tipping2001sparse,
  author  = {Tipping, Michael E.},
  title   = {Sparse Bayesian Learning and the Relevance Vector Machine},
  journal = {Journal of Machine Learning Research},
  volume  = {1},
  pages   = {211--244},
  year    = {2001}
}

@article{angelopoulos2023conformal,
  author  = {Angelopoulos, Anastasios N. and Bates, Stephen},
  title   = {Conformal Prediction: A Gentle Introduction},
  journal = {Foundations and Trends in Machine Learning},
  volume  = {16},
  number  = {4},
  pages   = {494--591},
  year    = {2023},
  doi     = {10.1561/2200000101}
}

@inproceedings{papadopoulos2002inductive,
  author    = {Papadopoulos, Harris and Proedrou, Kostas and Vovk, Vladimir and Gammerman, Alexander},
  title     = {Inductive Confidence Machines for Regression},
  booktitle = {Machine Learning: ECML 2002},
  series    = {Lecture Notes in Computer Science},
  volume    = {2430},
  pages     = {345--356},
  publisher = {Springer},
  year      = {2002},
  doi       = {10.1007/3-540-36755-1_29}
}

@article{lei2018distribution,
  author  = {Lei, Jing and G'Sell, Max and Rinaldo, Alessandro and Tibshirani, Ryan J. and Wasserman, Larry},
  title   = {Distribution-Free Predictive Inference for Regression},
  journal = {Journal of the American Statistical Association},
  volume  = {113},
  number  = {523},
  pages   = {1094--1111},
  year    = {2018},
  doi     = {10.1080/01621459.2017.1307116}
}

@article{barber2023conformal,
  author  = {Barber, Rina Foygel and Cand{\`e}s, Emmanuel J. and Ramdas, Aaditya and Tibshirani, Ryan J.},
  title   = {Conformal Prediction Beyond Exchangeability},
  journal = {The Annals of Statistics},
  volume  = {51},
  number  = {2},
  pages   = {816--845},
  year    = {2023},
  doi     = {10.1214/23-AOS2276}
}

@inproceedings{BurnaevVovk2014,
  title     = {Efficiency of Conformalized Ridge Regression},
  author    = {Burnaev, Evgeny and Vovk, Vladimir},
  booktitle = {Proceedings of the 27th Conference on Learning Theory},
  series    = {Proceedings of Machine Learning Research},
  volume    = {35},
  pages     = {605--622},
  year      = {2014},
  publisher = {PMLR}
}

@inproceedings{ChernozhukovWuthrichZhu2018,
  title     = {Exact and Robust Conformal Inference Methods for Predictive Machine Learning with Dependent Data},
  author    = {Chernozhukov, Victor and W{\"u}thrich, Kaspar and Zhu, Yinchu},
  booktitle = {Proceedings of the 31st Conference on Learning Theory},
  series    = {Proceedings of Machine Learning Research},
  volume    = {75},
  pages     = {732--749},
  year      = {2018},
  publisher = {PMLR}
}

@article{Yaskov2016,
  title   = {Necessary and Sufficient Conditions for the Marchenko--Pastur Theorem},
  author  = {Yaskov, Pavel},
  journal = {Electronic Communications in Probability},
  volume  = {21},
  number  = {73},
  pages   = {1--8},
  year    = {2016},
  doi     = {10.1214/16-ECP4748}
}

@book{Berbee1979,
  author    = {Berbee, H. C. P.},
  title     = {Random Walks with Stationary Increments and Renewal Theory},
  series    = {Mathematical Centre Tracts},
  volume    = {112},
  publisher = {Mathematisch Centrum},
  address   = {Amsterdam},
  year      = {1979}
}

@book{NourdinPeccati2012,
  author    = {Nourdin, Ivan and Peccati, Giovanni},
  title     = {Normal Approximations with Malliavin Calculus:
               From Stein's Method to Universality},
  publisher = {Cambridge University Press},
  year      = {2012}
}

@book{Durrett2019,
  author    = {Durrett, Rick},
  title     = {Probability: Theory and Examples},
  edition   = {5},
  series    = {Cambridge Series in Statistical and Probabilistic Mathematics},
  volume    = {49},
  publisher = {Cambridge University Press},
  year      = {2019},
  doi       = {10.1017/9781108591034}
}

@article{Massart1990,
  author  = {Massart, Pascal},
  title   = {The Tight Constant in the Dvoretzky--Kiefer--Wolfowitz Inequality},
  journal = {The Annals of Probability},
  volume  = {18},
  number  = {3},
  pages   = {1269--1283},
  year    = {1990},
  doi     = {10.1214/aop/1176990746}
}

@book{RencherSchaalje2008,
  author    = {Rencher, Alvin C. and Schaalje, G. Bruce},
  title     = {Linear Models in Statistics},
  edition   = {2},
  publisher = {John Wiley \& Sons},
  year      = {2008},
  isbn      = {9780471754985},
  doi       = {10.1002/9780470192610}
}

@book{BoucheronLugosiMassart2013,
  author    = {Boucheron, St{\'e}phane and Lugosi, G{\'a}bor and Massart, Pascal},
  title     = {Concentration Inequalities: A Nonasymptotic Theory of Independence},
  publisher = {Oxford University Press},
  year      = {2013},
  isbn      = {9780199535255},
  doi       = {10.1093/acprof:oso/9780199535255.001.0001}
}

@incollection{Vershynin2012,
  author    = {Vershynin, Roman},
  title     = {Introduction to the Non-Asymptotic Analysis of Random Matrices},
  booktitle = {Compressed Sensing: Theory and Applications},
  editor    = {Eldar, Yonina C. and Kutyniok, Gitta},
  pages     = {210--268},
  publisher = {Cambridge University Press},
  year      = {2012},
  doi       = {10.1017/CBO9780511794308.006}
}

@unpublished{vanHandel2016,
  author = {van Handel, Ramon},
  title  = {Probability in High Dimension},
  note   = {APC 550 Lecture Notes, Princeton University},
  year   = {2016},
  url    = {https://web.math.princeton.edu/~rvan/APC550.pdf}
}

@misc{YangEtAl2026PICPI,
  title        = {PICPIs: Prediction-Interval-Conditional Prediction Intervals},
  author       = {Yang, Xuelin and Huang, Baihe and Hou, Yilong and Imbens, Guido and Jordan, Michael I.},
  year         = {2026},
  note         = {Manuscript, forthcoming on arXiv}
}

@inproceedings{ClarteZdeborova2025,
  title     = {Building Conformal Prediction Intervals with Approximate Message Passing},
  author    = {Clart{\'e}, Lucas and Zdeborov{\'a}, Lenka},
  booktitle = {Proceedings of the Forty-First Conference on Uncertainty in Artificial Intelligence},
  series    = {Proceedings of Machine Learning Research},
  volume    = {286},
  pages     = {798--820},
  year      = {2025},
  publisher = {PMLR}
}

@article{GibbsCandes2025HighDimensional,
  title   = {Characterizing the Training-Conditional Coverage of Full Conformal Inference in High Dimensions},
  author  = {Gibbs, Isaac and Cand{\`e}s, Emmanuel J.},
  journal = {arXiv preprint arXiv:2502.20579},
  year    = {2025},
  eprint  = {2502.20579},
  archivePrefix = {arXiv},
  primaryClass  = {stat.ML}
}

\newpage
\appendix
\appendix

\section{Discretization and Connection to ARMA-Type Models}
\label{app:arma_connection}

Memory-driven continuous-time systems can induce discrete-time dynamics with
autoregressive structure and, under additional assumptions on the memory
operator, moving-average components. However, discretizing a memory term
expressed in past states does not produce a moving average of past
innovations.

Let $L$ denote the lag operator. Suppose that discretization of the local
dynamics and the memory term yields a relation of the form
\(
A(L)x_t=K(L)x_t+\varepsilon_t,
\)
where $A(L)$ is the polynomial associated with the local dynamics,
$K(L)$ is a linear state-memory operator and $\varepsilon_t$ denotes the
innovation term.

If the memory operator admits a rational representation
\(
K(L)= B(L) / C(L),
\)
with finite-order polynomials $B(L)$ and $C(L)$, then multiplying by 
$C(L)$ gives
\(
\{C(L)A(L)-B(L)\}x_t
=
C(L)\varepsilon_t.
\)
After normalization, this can be written in the standard ARMA-type form
\(
\Phi(L)x_t=\Theta(L)\varepsilon_t,
\)
where
\(
\Phi(L)=C(L)A(L)-B(L),
\;\;
\Theta(L)=C(L).
\)
Equivalently,
\(
x_t
=
\sum_{j=1}^{p}\phi_j x_{t-j}
+
\varepsilon_t
+
\sum_{j=1}^{q}\theta_j\varepsilon_{t-j},
\)
for suitable autoregressive and moving-average coefficients whenever
$\Phi(L)$ and $\Theta(L)$ have finite degree.

The distinction is important. A finite approximation to a convolution of past
states,
\(
K(L)x_t
\approx
\sum_{j=1}^{J}\kappa_j x_{t-j},
\)
generally contributes additional lagged state terms and therefore enlarges the
autoregressive component. It does not generate an MA$(q)$ term in
past innovations. The moving-average component arises only after the state-memory
operator is eliminated in a representation that transfers part of the dynamics
to the innovation side of the equation.

This connection provides a bridge between memory-driven dynamical systems and
classical discrete-time time-series models. In particular, systems with suitably
structured memory kernels may admit finite-order ARMA or VARMA representations
after discretization and elimination of the memory operator. More general memory
kernels may instead lead to higher-order or infinite-order
autoregressive representations. The NGRC framework does not require such a finite-order representation. Its delay-coordinate feature map directly models dependence on past observations and can incorporate nonlinear interactions among
those lags.

\section{Proofs and Technical Results}
\label{app:proofs}

Throughout the appendix, write
$z=z_{1-\alpha/2}$. All limits are taken along the asymptotic sequence
specified in the corresponding main-text result.

\subsection{Background on Standard Technical Tools}
\label{app:technical_tools}

For completeness, we briefly record references for several standard
probabilistic and random-matrix tools used in the proofs. These results are not
specific to the present setting; they are collected here to make the technical
arguments easier to follow and to provide references for readers seeking
additional background.

\paragraph{Ergodic theorem.}
If $(Z_t)_{t\in\mathbb Z}$ is strictly stationary and ergodic with
$\mathbb E|Z_0|<\infty$, Birkhoff's ergodic theorem gives
\(
\frac1n\sum_{t=1}^n Z_t
\longrightarrow
\mathbb E Z_0
\) almost surely. The result applies componentwise to integrable finite-dimensional vector- and
matrix-valued processes. We use this fact repeatedly in the fixed-dimensional
analysis to obtain convergence of empirical feature moments, feature--residual
cross moments, residual second moments, and calibration-block averages. See
\cite[Theorem~6.2.1]{Durrett2019} for background on ergodic theorems.

\paragraph{Gaussian quadratic forms.}
If $Z\sim N(0,C)$ and $M$ is symmetric, then
\(
\mathbb E(Z^\top MZ)=\operatorname{tr}(MC) \) and 
\(\operatorname{Var}(Z^\top MZ)
=
2\operatorname{tr}\{(MC)^2\}.
\)
Related formulas give, for deterministic $b$,
\(
\operatorname{Var}
\left(
Z^\top MZ+2b^\top MZ
\right)
=
2\operatorname{tr}\{(MC)^2\}
+
4b^\top MCMb.
\)
These standard Gaussian moment identities are used to control the fitted
prediction error and Bayesian test-feature uncertainty; see, for example,
\cite[Theorem~5.2a, Theorem~5.2c]{RencherSchaalje2008}.

\paragraph{Empirical-CDF concentration.}
For i.i.d.\ observations with distribution function $F$ and empirical
distribution function $\widehat F_m$, the
Dvoretzky--Kiefer--Wolfowitz--Massart inequality states that
\[
\mathbb P\left(
\sup_x|\widehat F_m(x)-F(x)|>\epsilon
\right)
\le
2e^{-2m\epsilon^2}.
\]
We use this inequality after conditioning on the fitted training sample, so
that the calibration residuals in the independent benchmark are conditionally
i.i.d. The sharp two-sided constant is due to \cite[Corollary~1]{Massart1990}.

\paragraph{Gaussian concentration and integration by parts.}
For $G\sim N(0,I_d)$ and a sufficiently regular function $f$, the Gaussian
Poincar\'e inequality gives
\(
\operatorname{Var}\{f(G)\}
\le
\mathbb E\|\nabla f(G)\|^2;
\)
see
\cite[Theorem~3.20]{BoucheronLugosiMassart2013}.
We use this inequality to establish quadratic-form concentration for the
explicit quadratic NGRC feature map. Gaussian integration by parts,
\(
\mathbb E\{G_j f(G)\}
=
\mathbb E\left\{
\frac{\partial f(G)}{\partial G_j}
\right\},
\)
is given more generally in
\cite[Lemma~6.10]{vanHandel2016}.
Its connection with Stein's identity is discussed in
\cite[Lemma~3.1.2]{NourdinPeccati2012}.

\paragraph{Covering nets.}
For $\epsilon\in(0,1)$, the Euclidean unit sphere in $\mathbb R^d$ admits an
$\epsilon$-net of cardinality at most
$(1+2/\epsilon)^d$ \cite[Lemma~5.2]{Vershynin2012}.
For a symmetric matrix $B$ and $\epsilon<1/2$,
\(
\|B\|_{\mathrm{op}}
\le
\frac{1}{1-2\epsilon}
\max_{u\in\mathcal N_\epsilon}|u^\top Bu|
\)
\cite[Lemma~5.4]{Vershynin2012}.
In particular, a $1/4$-net can be chosen with cardinality at most $9^d$ and
gives
\(
\|B\|_{\mathrm{op}}
\le
2\max_{u\in\mathcal N_{1/4}}|u^\top Bu|.
\)
We use this device to control the largest quadratic direction of the fitted
quadratic-NGRC prediction error.

\paragraph{Marchenko--Pastur limits.}
For an isotropic Gaussian design with $p/n\to\gamma\in(0,\infty)$, the
empirical spectral distribution of the sample covariance converges to the
Marchenko--Pastur law \cite[Theorem~1]{MarchenkoPastur1967}. Consequently, for fixed
$\lambda>0$,
\(
\frac1p\operatorname{tr}(S+\lambda I)^{-1}
\longrightarrow
\int\frac{1}{x+\lambda}\,dF_\gamma^{\mathrm{MP}}(x),
\)
with analogous limits for $(S+\lambda I)^{-2}$ and
$S(S+\lambda I)^{-2}$. These resolvent quantities are standard in
high-dimensional ridge-regression theory; see also
\cite[Lemma~2.2]{DobribanWager2018}.

For isotropic random vectors whose coordinates need not be independent,
quadratic-form concentration provides a more general route to the
Marchenko--Pastur law. In particular, Assumption~(A1) of
\cite{Yaskov2016} requires
\(
\frac{x_p^\top B_px_p-\operatorname{tr}(B_p)}{p}
\xrightarrow{\mathbb P}0
\)
for every sequence of uniformly bounded positive-semidefinite $B_p$.
Under proportional growth, \cite[Theorem~3.3]{Yaskov2016} shows that this
condition implies the Marchenko--Pastur property.

\paragraph{Coupling under absolute regularity.}
For an absolutely regular, or $\beta$-mixing, process, Berbee's coupling
principle allows a future observation block to be coupled to an independent
copy with the same marginal distribution, with failure probability controlled
by the relevant absolute-regularity coefficient. We use this result in
Appendix~\ref{app:dependent_transfer} to transfer independent-window
probability statements to sufficiently separated NGRC observation windows.
See \cite[Corollary~4.2.5]{Berbee1979}.

\subsection{Quantile Stability}
\label{app:quantile_stability}

The following auxiliary result is used in the fixed-dimensional conformal
argument. It shows that convergence of the calibration quantile requires only average, rather than uniform, convergence of the fitted conformity scores.

\begin{lemma}[Quantile stability under average score perturbations]
\label{lem:quantile_stability}
Let $S_1,\ldots,S_m$ be estimated calibration scores and
$S_1^0,\ldots,S_m^0$ reference scores, with empirical distribution functions
$\widehat F_m$ and $\widehat F_m^0$. Let $u=1-\alpha$,
$k_m=\lceil(m+1)u\rceil$. Write $S_{(1)}\le\cdots\le S_{(m)}$ for the order
statistics and adopt the convention $S_{(m+1)}=+\infty$. Define
$\widehat q_m=S_{(k_m)}$. Suppose that
$\widehat F_m^0(q)\to F(q)$ in probability at every continuity point of $F$,
\[
\frac1m\sum_{i=1}^m|S_i-S_i^0|\xrightarrow{\mathbb P}0,
\]
and that $F$ is continuous near its $u$-quantile $q_u$ and strictly crosses
$u$ there. Then
$\widehat q_m\xrightarrow{\mathbb P}q_u$.
\end{lemma}

\begin{proof} For $a>0$, define \( D_m(a) = \frac1m\sum_{i=1}^m \mathbf 1\{|S_i-S_i^0|>a\}. \) Outside the exceptional set $\{|S_i-S_i^0|>a\}$, the implications \(S_i^0\le q-a \Longrightarrow S_i\le q, \;\; S_i\le q \Longrightarrow S_i^0\le q+a \) hold. Summing the corresponding indicators and dividing by $m$ gives the deterministic bounds \(\widehat F_m^0(q-a)-D_m(a) \le \widehat F_m(q) \le \widehat F_m^0(q+a)+D_m(a). \) Moreover, Markov's inequality applied deterministically to the empirical average gives \( D_m(a) \le \frac{1}{am}\sum_{i=1}^m|S_i-S_i^0| \xrightarrow{\mathbb P}0. \) Fix $t>0$ sufficiently small. By continuity of $F$ near $q_u$ and the strict crossing assumption, we may choose $a\in(0,t)$ such that \( F(q_u-t+a)<u<F(q_u+t-a). \) Applying the preceding empirical-CDF bounds at $q_u-t$ and $q_u+t$ gives \( \widehat F_m(q_u-t) \le \widehat F_m^0(q_u-t+a)+D_m(a) \) and \(\widehat F_m(q_u+t) \ge \widehat F_m^0(q_u+t-a)-D_m(a). \) Since $\widehat F_m^0(q)\to F(q)$ in probability at the two displayed continuity points and $D_m(a)\to0$ in probability \( \widehat F_m(q_u-t)<u+o_{\mathbb P}(1), \;\; \widehat F_m(q_u+t)>u+o_{\mathbb P}(1). \) Also, \( \frac{k_m}{m}\longrightarrow u. \) Therefore, \[ \mathbb P\left\{ \widehat F_m(q_u-t) < \frac{k_m}{m} < \widehat F_m(q_u+t) \right\}\longrightarrow1. \] On this event, fewer than $k_m$ observations are at or below $q_u-t$, whereas at least $k_m$ observations are at or below $q_u+t$. By the definition of the $k_m$th order statistic, \( q_u-t<\widehat q_m\le q_u+t. \) Hence \( \mathbb P(|\widehat q_m-q_u|\le t)\longrightarrow1. \) Since $t>0$ is arbitrary, \( \widehat q_m\xrightarrow{\mathbb P}q_u. \) \end{proof}

\subsection{Fixed-Dimensional Results}
\label{app:fixed_theory}

\begin{proof}[Proof of Lemma~\ref{lem:projection_consistency}] Let \( \mathcal S=\operatorname{range}(\Sigma) =\ker(\Sigma)^\perp. \) If $v\in\ker(\Sigma)$, then \( 0=v^\top\Sigma v = \mathbb E(v^\top r_0)^2. \) Since $(v^\top r_0)^2\ge0$, it follows that $v^\top r_0=0$ almost surely. Thus $r_t\in\mathcal S$ almost surely for every $t$ by strict stationarity. We next verify that the population projection coefficient is well defined on $\mathcal S$. For any $v\in\ker(\Sigma)$, \( v^\top\mathbb E[r_0y_0] = \mathbb E[(v^\top r_0)y_0] = 0, \) and therefore \( \mathbb E[r_0y_0]\in\ker(\Sigma)^\perp=\mathcal S. \) By definition, \( w^\dagger = \Sigma^\dagger\mathbb E[r_0y_0]. \) The matrix $\Sigma\Sigma^\dagger$ is the orthogonal projector $P_{\mathcal S}$ onto $\mathcal S$, so \( \Sigma w^\dagger = \Sigma\Sigma^\dagger\mathbb E[r_0y_0] = P_{\mathcal S}\mathbb E[r_0y_0] = \mathbb E[r_0y_0]. \) Writing \( e_t=y_t-r_t^\top w^\dagger, \) we therefore have \( \mathbb E[r_0e_0] = \mathbb E[r_0y_0]-\Sigma w^\dagger = 0. \) 

Moreover, by Cauchy--Schwarz, \( \mathbb E\|r_0e_0\| \le \{\mathbb E\|r_0\|^2\}^{1/2} \{\mathbb E e_0^2\}^{1/2} <\infty. \) Because $(r_t,y_t)$ is stationary and ergodic, each coordinate of $r_tr_t^\top$, $r_te_t$, and $e_t^2$ is also stationary and ergodic. The ergodic theorem \cite[Theorem~6.2.1]{Durrett2019} therefore yields \( S_n:=\frac1nR_n^\top R_n\to\Sigma, \;\; \frac1nR_n^\top e_n\to0, \;\; \frac1n\sum_{t=1}^ne_t^2\to\tau^2 \) almost surely. Since the dimension is fixed, entrywise convergence of $S_n$ also implies \( \|S_n-\Sigma\|_{\mathrm{op}}\to0 \;\; \text{almost surely}. \) Every row of $R_n$ belongs to $\mathcal S$ almost surely, and hence $R_n^\top y_n\in\mathcal S$. The ridge matrix $S_n+\lambda_nI$ leaves both $\mathcal S$ and $\mathcal S^\perp$ invariant, so \( \widehat w_n = (S_n+\lambda_nI)^{-1}\frac1nR_n^\top y_n \in\mathcal S. \) Also $w^\dagger\in\mathcal S$. Define \( \Delta_n=\widehat w_n-w^\dagger. \) Using \( \frac1nR_n^\top y_n = S_nw^\dagger+\frac1nR_n^\top e_n, \) we obtain \[ \begin{aligned} \Delta_n &= (S_n+\lambda_nI)^{-1} \left( S_nw^\dagger+\frac1nR_n^\top e_n \right) -w^\dagger\\ &= (S_n+\lambda_nI)^{-1} \left( \frac1nR_n^\top e_n-\lambda_nw^\dagger \right). \end{aligned} \] Restricted to $\mathcal S$, the matrix $\Sigma$ is positive definite. Let \( c = \lambda_{\min}(\Sigma|_{\mathcal S})>0. \) Since $\|S_n-\Sigma\|_{\mathrm{op}}\to0$ almost surely, Weyl's inequality implies that, almost surely for all sufficiently large $n$, \( \lambda_{\min}(S_n|_{\mathcal S})\ge \frac c2. \) Because $\lambda_n\ge0$, \( \left\| (S_n+\lambda_nI)^{-1}|_{\mathcal S} \right\|_{\mathrm{op}} \le \frac{2}{c} \) eventually almost surely. Consequently, \[ \|\Delta_n\| \le \frac2c \left\| \frac1nR_n^\top e_n-\lambda_nw^\dagger \right\| \longrightarrow0 \] almost surely, since $n^{-1}R_n^\top e_n\to0$ and $\lambda_n\to0$. Hence \( \widehat w_n\to w^\dagger \;\; \text{almost surely}. \) Finally, \( y_n-R_n\widehat w_n = e_n-R_n\Delta_n, \) so \[ \begin{aligned} \widehat\tau_n^2 &= \frac1n\|e_n-R_n\Delta_n\|^2\\ &= \frac1n\sum_{t=1}^ne_t^2 - 2\Delta_n^\top\frac1nR_n^\top e_n + \Delta_n^\top S_n\Delta_n. \end{aligned} \] The first term converges almost surely to $\tau^2$. For the second term, \[ \left| \Delta_n^\top\frac1nR_n^\top e_n \right| \le \|\Delta_n\| \left\|\frac1nR_n^\top e_n\right\| \to0 \] almost surely. Since $S_n\to\Sigma$, the sequence $\|S_n\|_{\mathrm{op}}$ is almost surely bounded, and therefore \[ 0\le \Delta_n^\top S_n\Delta_n \le \|S_n\|_{\mathrm{op}}\|\Delta_n\|^2 \to0. \] Thus \( \widehat\tau_n^2\to\tau^2\) almost surely. \end{proof}

\begin{proof}[Proof of Theorem~\ref{thm:fixed_width_discrepancy}] Fix $r\in\mathcal S$. From the proof of Lemma~\ref{lem:projection_consistency}, the restriction of $(S_n+\lambda_nI)^{-1}$ to $\mathcal S$ has almost surely bounded operator norm for all sufficiently large $n$. Hence \[ 0 \le \frac1n r^\top(S_n+\lambda_nI)^{-1}r \le \frac{\|r\|^2}{n} \left\| (S_n+\lambda_nI)^{-1}|_{\mathcal S} \right\|_{\mathrm{op}} = O(n^{-1}) \] almost surely. Together with $\widehat\tau_n^2\to\tau^2$, this gives \[ W_{B,n}(r) = 2z\widehat\tau_n \sqrt{ 1+\frac1n r^\top(S_n+\lambda_nI)^{-1}r } \xrightarrow{\mathbb P} 2z\tau. \] We next consider the conformal interval. For $i\in\mathcal I_{\mathrm{cal},n}$, define \( S_i = |y_i-r_i^\top\widehat w_n|, \;\; S_i^0 = |e_i| = |y_i-r_i^\top w^\dagger|. \) The reverse triangle inequality gives \[ |S_i-S_i^0| \le \left| r_i^\top(\widehat w_n-w^\dagger) \right| \le \|r_i\|\, \|\widehat w_n-w^\dagger\|. \] Therefore \[ \frac1{m_n} \sum_{i\in\mathcal I_{\mathrm{cal},n}} |S_i-S_i^0| \le \|\widehat w_n-w^\dagger\| \frac1{m_n} \sum_{i\in\mathcal I_{\mathrm{cal},n}} \|r_i\|. \] Write the calibration block as \( \mathcal I_{\mathrm{cal},n} = \{a_n+1,\ldots,a_n+m_n\}. \) By strict stationarity, \\ \( \frac1{m_n} \sum_{i\in\mathcal I_{\mathrm{cal},n}}\|r_i\| \overset{d}{=} \frac1{m_n}\sum_{i=1}^{m_n}\|r_i\|. \) Since $m_n\to\infty$, the ergodic theorem \cite[Theorem~6.2.1]{Durrett2019} gives \( \frac1{m_n}\sum_{i=1}^{m_n}\|r_i\| \longrightarrow \mathbb E\|r_0\| \) almost surely. 

Hence the calibration-block average satisfies \( \frac1{m_n} \sum_{i\in\mathcal I_{\mathrm{cal},n}}\|r_i\| \xrightarrow{\mathbb P} \mathbb E\|r_0\|. \) Since $\widehat w_n-w^\dagger\to0$ in probability, \( \frac1{m_n} \sum_{i\in\mathcal I_{\mathrm{cal},n}} |S_i-S_i^0| = o_{\mathbb P}(1). \) It remains to verify convergence of the empirical distribution of the reference scores. Let \( F_e(q) = \mathbb P(|e_0|\le q). \) For every continuity point $q$ of $F_e$, the process \( \mathbf 1\{|e_t|\le q\} \) is stationary, ergodic, and bounded. 

Therefore \( \frac1{m_n} \sum_{i=1}^{m_n} \mathbf 1\{|e_i|\le q\} \longrightarrow F_e(q) \) almost surely. By stationarity, the corresponding calibration-block average has the same distribution, and hence \( \frac1{m_n} \sum_{i\in\mathcal I_{\mathrm{cal},n}} \mathbf 1\{|e_i|\le q\} \xrightarrow{\mathbb P} F_e(q). \) Thus the reference-score empirical CDF converges at every continuity point of $F_e$. Let $q_e$ denote the $(1-\alpha)$-quantile of $|e_0|$. By the assumed continuity and strict crossing of $F_e$ at $q_e$, all conditions of Lemma~\ref{lem:quantile_stability} are satisfied. Hence \( \widehat q_{1-\alpha} \xrightarrow{\mathbb P} q_e. \) The conformal width therefore satisfies \\ \( W_{C,n} = 2\widehat q_{1-\alpha} \xrightarrow{\mathbb P} 2q_e. \) Combining the Bayesian and conformal limits gives \[ W_{B,n}(r)-W_{C,n} \xrightarrow{\mathbb P} 2(z\tau-q_e). \] In particular, the two widths are asymptotically equivalent if and only if \( q_e=z\tau. \) \end{proof}

\begin{proof}[Proof of Corollary~\ref{cor:fixed_coverage}]
Let
\(
e_*=y_*-r_*^\top w^\dagger.
\)
Since $\widehat w_n\to w^\dagger$ in probability and
$\mathbb E\|r_*\|^2<\infty$,
\(
r_*^\top(\widehat w_n-w^\dagger)
=
o_{\mathbb P}(1).
\)
Therefore
\(
y_*-r_*^\top\widehat w_n
=
e_*+o_{\mathbb P}(1).
\)

Theorem~\ref{thm:fixed_width_discrepancy} and the same fixed-dimensional
leverage argument, now applied to the random test feature $r_*$, give
\(
W_{B,n}(r_*) / 2
\xrightarrow{\mathbb P}
z\tau.
\)
By continuity of $F_{|e|}$ at $z\tau$ and Slutsky's theorem,
\[
\mathbb P\{y_*\in C_{B,n}(r_*)\}
=
\mathbb P
\left\{
|y_*-r_*^\top\widehat w_n|
\le
\frac12W_{B,n}(r_*)
\right\}
\longrightarrow
F_{|e|}(z\tau).
\]
Similarly,
\(
W_{C,n}/2
\xrightarrow{\mathbb P}
q_e.
\)
The continuity and strict-crossing assumptions at $q_e$ imply
\(
F_{|e|}(q_e)=1-\alpha.
\)
Hence
\(
\mathbb P\{y_*\in C_{C,n}(r_*)\}
\longrightarrow
1-\alpha.
\)
If $q_e=z\tau$, the Bayesian limit is also $1-\alpha$.
\end{proof}

\subsection{High-Dimensional Finite-Trace Results}
\label{app:hd_theory}

\subsubsection{Proof of Lemma~\ref{lem:bvp_decomposition}}
\begin{proof}
Since \( A=(S+\lambda I)^{-1}, \;\; S=\frac1nR^\top R, \) the ridge estimator satisfies \( \widehat w = A\left( Sw_0+\frac1nR^\top\varepsilon \right). \) Using \( AS=I-\lambda A, \) we obtain \( \widehat w-w_0 = -\lambda Aw_0+\xi, \;\; \xi = \frac1nAR^\top\varepsilon. \) Conditionally on $R$, \( \mathbb E(\xi\mid R)=0 \) and \[ \begin{aligned} \operatorname{Cov}(\xi\mid R) &= \frac1{n^2} AR^\top \operatorname{Cov}(\varepsilon) RA\\ &= \frac{\sigma^2}{n^2} AR^\top RA\\ &= \frac{\sigma^2}{n}ASA. \end{aligned} \] Expanding the prediction-error quadratic form gives \[ \begin{aligned} (\widehat w-w_0)^\top \Sigma(\widehat w-w_0) &= \lambda^2w_0^\top A\Sigma Aw_0 - 2\lambda w_0^\top A\Sigma\xi + \xi^\top\Sigma\xi. \end{aligned} \] The conditional expectation of the cross term is zero: \( \mathbb E[ w_0^\top A\Sigma\xi \mid R] = w_0^\top A\Sigma\mathbb E(\xi\mid R) = 0. \) For the quadratic term, \[ \begin{aligned} \mathbb E[ \xi^\top\Sigma\xi \mid R] &= \operatorname{tr} \left\{ \Sigma\operatorname{Cov}(\xi\mid R) \right\}\\ &= \frac{\sigma^2}{n} \operatorname{tr}(\Sigma ASA). \end{aligned} \] Therefore \[ \begin{aligned} \mathbb E\left[ (\widehat w-w_0)^\top \Sigma(\widehat w-w_0) \,\middle|\,R \right] &= \lambda^2w_0^\top A\Sigma Aw_0 + \frac{\sigma^2}{n}\operatorname{tr}(\Sigma ASA)\\ &= B_n+V_n. \end{aligned} \] For an independent test observation \( y_*=r_*^\top w_0+\varepsilon_*, \) we have \( y_*-r_*^\top\widehat w = r_*^\top(w_0-\widehat w)+\varepsilon_*. \) Conditionally on $R$ and $\widehat w$, the two terms are centered and independent, with conditional second moments \( \mathbb E[ \{r_*^\top(w_0-\widehat w)\}^2 \mid R,\widehat w] = (\widehat w-w_0)^\top \Sigma(\widehat w-w_0) \) and \( \mathbb E(\varepsilon_*^2)=\sigma^2. \) Taking the remaining conditional expectation gives \( \mathbb E[ (y_*-r_*^\top\widehat w)^2 \mid R] = \sigma^2+B_n+V_n. \) Finally, since \( A(S+\lambda I)=I, \) we have \( AS=I-\lambda A, \) and, because $A$ and $S$ commute, \( ASA = A-\lambda A^2. \) Thus \[ \begin{aligned} P_n-V_n &= \frac{\sigma^2}{n} \operatorname{tr}(\Sigma A) - \frac{\sigma^2}{n} \operatorname{tr}(\Sigma ASA)\\ &= \frac{\sigma^2}{n} \operatorname{tr}\{\Sigma(A-ASA)\}\\ &= \frac{\sigma^2\lambda}{n} \operatorname{tr}(\Sigma A^2). \end{aligned} \] Subtracting $B_n$ yields \( P_n-V_n-B_n = \frac{\sigma^2\lambda}{n} \operatorname{tr}(\Sigma A^2) - \lambda^2w_0^\top A\Sigma Aw_0, \) which is the claimed bias--variance--posterior discrepancy. \end{proof}

\subsubsection{Proof of Theorem~\ref{thm:finite_trace_widths}}
\begin{proof} By the assumed uniform bounds, choose constants $\Gamma,L,M<\infty$ such that, along the asymptotic sequence, \(p/n\le\Gamma\), \( \|\Sigma\|_{\mathrm{op}}\le L, \) \( \|w_0\|\le M. \) Throughout this proof $\lambda>0$ is fixed. Set \( b=\lambda Aw_0,\) \( C=\frac{\sigma^2}{n}ASA,\) \( Q_n = (\widehat w-w_0)^\top \Sigma(\widehat w-w_0). \) 

From Lemma~\ref{lem:bvp_decomposition}, \(\widehat w-w_0=-b+\xi,\) \( \xi\mid R\sim N(0,C). \) Hence \( Q_n = b^\top\Sigma b - 2b^\top\Sigma\xi + \xi^\top\Sigma\xi, \) and \( \mathbb E(Q_n\mid R) = B_n+V_n. \) Define \( M_C = \Sigma^{1/2}C\Sigma^{1/2}. \) For a centered Gaussian vector, the standard quadratic-form identity gives \( \operatorname{Var}(\xi^\top\Sigma\xi\mid R) = 2\operatorname{tr}(M_C^2). \) The linear term satisfies \( \operatorname{Var} (2b^\top\Sigma\xi\mid R) = 4b^\top\Sigma C\Sigma b. \) Moreover, centered Gaussian odd moments vanish, so the covariance between the linear term and the centered quadratic term is zero. Therefore \( \operatorname{Var}(Q_n\mid R) = 2\operatorname{tr}(M_C^2) + 4b^\top\Sigma C\Sigma b. \) 

Since $S$ is positive semidefinite, \( \|A\|_{\mathrm{op}}\le 1/\lambda. \) Also, \( ASA = S(S+\lambda I)^{-2}, \) so \[ \|ASA\|_{\mathrm{op}} = \max_{x\in\operatorname{spec}(S)} \frac{x}{(x+\lambda)^2} \le \frac1{4\lambda}, \] where the scalar function $x/(x+\lambda)^2$ is maximized at $x=\lambda$. Consequently, \( \|C\|_{\mathrm{op}} \le \sigma^2 / 4n\lambda, \) and hence \( \|M_C\|_{\mathrm{op}} \le \|\Sigma\|_{\mathrm{op}}\|C\|_{\mathrm{op}} \le \sigma^2L / 4n\lambda. \) Similarly, \( \operatorname{tr}(M_C) \le p\|M_C\|_{\mathrm{op}} \le \sigma^2Lp / 4n\lambda. \) Because \( \|\lambda A\|_{\mathrm{op}}\le1, \) we have \( \|b\| = \|\lambda Aw_0\| \le \|w_0\| \le M, \) and therefore \( B_n = b^\top\Sigma b \le LM^2. \) 

For the first variance term, \[ 2\operatorname{tr}(M_C^2) \le 2p\|M_C\|_{\mathrm{op}}^2\\ \le \frac{\sigma^4L^2p} {8n^2\lambda^2}. \] For the second, \[ \begin{aligned} 4b^\top\Sigma C\Sigma b &= 4(\Sigma^{1/2}b)^\top M_C (\Sigma^{1/2}b)\\ &\le 4B_n\|M_C\|_{\mathrm{op}}\\ &\le \frac{\sigma^2L^2M^2}{n\lambda}. \end{aligned} \] Since $p/n$ is uniformly bounded, \(\operatorname{Var}(Q_n\mid R) = O(n^{-1}) \) uniformly along the sequence. Conditional Chebyshev's inequality therefore implies \( Q_n-(B_n+V_n) = O_{\mathbb P}(n^{-1/2}), \) or equivalently, \\ \( Q_n = B_n+V_n+O_{\mathbb P}(n^{-1/2}). \) 

We now study the conformal interval. Let $\mathcal D_{\mathrm{tr}}=(R,y)$ denote the training sample and condition on $\mathcal D_{\mathrm{tr}}$. Under the Gaussian design and sample-splitting assumptions, the calibration observations are independent of $\mathcal D_{\mathrm{tr}}$ and are independent across calibration indices. For a calibration observation, \( y_i-r_i^\top\widehat w = r_i^\top(w_0-\widehat w)+\varepsilon_i. \) Conditional on $\mathcal D_{\mathrm{tr}}$, the vector $w_0-\widehat w$ is fixed, while \( r_i\sim N(0,\Sigma),\) and \( \varepsilon_i\sim N(0,\sigma^2) \) are independent. Hence \( r_i^\top(w_0-\widehat w) \mid\mathcal D_{\mathrm{tr}} \sim N(0,Q_n), \) and therefore \( y_i-r_i^\top\widehat w \mid\mathcal D_{\mathrm{tr}} \sim N(0,\sigma^2+Q_n). \) Thus the conditional CDF of the absolute residual is \[ F_{Q_n}(q) = 2\Phi\left( \frac{q}{\sqrt{\sigma^2+Q_n}} \right)-1, \;\; q\ge0. \] Let $\widehat F_m$ denote the empirical CDF of the $m$ calibration absolute residuals. The Dvoretzky--Kiefer--Wolfowitz--Massart inequality \cite{Massart1990} gives, conditionally on $\mathcal D_{\mathrm{tr}}$, \[ \mathbb P\left( \sup_q |\widehat F_m(q)-F_{Q_n}(q)|>t \,\middle|\, \mathcal D_{\mathrm{tr}} \right) \le 2e^{-2mt^2}. \] Hence \( \sup_q |\widehat F_m(q)-F_{Q_n}(q)| = O_{\mathbb P}(m^{-1/2}). \) The $(1-\alpha)$-quantile of $F_{Q_n}$ is \( q_n^0 = z\sqrt{\sigma^2+Q_n},\) where \( z=z_{1-\alpha/2}. \) Its density is \[ f_{Q_n}(q) = \frac{2}{\sqrt{\sigma^2+Q_n}} \phi\left( \frac{q}{\sqrt{\sigma^2+Q_n}} \right), \;\; q>0. \] In particular, \( f_{Q_n}(q_n^0) = 2\phi(z) / \sqrt{\sigma^2+Q_n}. \) The preceding bounds imply \( Q_n=O_{\mathbb P}(1), \) because $B_n=O(1)$ and \( V_n = \operatorname{tr}(M_C) = O(1). \) Since $\sigma>0$, it follows that, with probability tending to one, $f_{Q_n}$ is bounded away from zero on a fixed neighborhood of $q_n^0$. Consequently the inverse-CDF map is locally Lipschitz there, and the empirical-CDF bound implies \( \widehat q_{1-\alpha} = z\sqrt{\sigma^2+Q_n} + O_{\mathbb P}(m^{-1/2}). \) 

Therefore, \[ \frac{W_{C,n}^2}{4z^2} = \frac{\widehat q_{1-\alpha}^2}{z^2} = \sigma^2+Q_n + O_{\mathbb P}(m^{-1/2}), \] where the last step uses $Q_n=O_{\mathbb P}(1)$. Substituting the expansion for $Q_n$ gives \[ \frac{W_{C,n}^2}{4z^2} = \sigma^2+B_n+V_n + O_{\mathbb P}(n^{-1/2}+m^{-1/2}). \] We next consider the Bayesian interval. Conditional on $R$, the independent test feature satisfies \( r_*\sim N(0,\Sigma), \) and hence \( \mathbb E( r_*^\top Ar_* \mid R) = \operatorname{tr}(\Sigma A). \) Therefore \( \mathbb E\left[ \frac{\sigma^2}{n}r_*^\top Ar_* \,\middle|\, R \right] = \frac{\sigma^2}{n} \operatorname{tr}(\Sigma A) = P_n. \) The Gaussian quadratic-form variance identity gives \[ \begin{aligned} \operatorname{Var}\left( \frac{\sigma^2}{n}r_*^\top Ar_* \,\middle|\, R \right) &= \frac{2\sigma^4}{n^2} \operatorname{tr} \left[ (\Sigma^{1/2}A\Sigma^{1/2})^2 \right]\\ &\le \frac{2\sigma^4p}{n^2} \|\Sigma^{1/2}A\Sigma^{1/2}\|_{\mathrm{op}}^2\\ &\le \frac{2\sigma^4L^2p}{n^2\lambda^2}\\ &= O(n^{-1}). \end{aligned} \] Thus \( \frac{\sigma^2}{n}r_*^\top Ar_* = P_n+O_{\mathbb P}(n^{-1/2}). \) Since the Bayesian predictive variance equals \( \sigma^2+ \frac{\sigma^2}{n}r_*^\top Ar_*, \) we obtain \[ \frac{W_{B,n}(r_*)^2}{4z^2} = \sigma^2+P_n + O_{\mathbb P}(n^{-1/2}). \] Finally, \[ \frac{W_{B,n}(r_*)^2-W_{C,n}^2}{4z^2} = P_n-V_n-B_n + O_{\mathbb P}(n^{-1/2}+m^{-1/2}), \] and Lemma~\ref{lem:bvp_decomposition} gives \[ P_n-V_n-B_n = \frac{\sigma^2\lambda}{n} \operatorname{tr}(\Sigma A^2) - \lambda^2w_0^\top A\Sigma Aw_0. \] \end{proof}

\subsubsection{Training residual scale in proportional dimension.}
\label{app:hd_residual_scale}
Let
\(
\mathsf H=\frac1nRAR^\top.
\) Since
$R\widehat w=\mathsf H y$,
\(
y-R\widehat w
=
(I_n-\mathsf H)Rw_0
+
(I_n-\mathsf H)\varepsilon.
\)
Using
$(I_n-\mathsf H)R=\lambda RA$ and conditioning on $R$ gives
\[
\mathbb E\left[
\frac{\|y-R\widehat w\|^2}{n}
\,\middle|\,R
\right]
=
\lambda^2w_0^\top ASA w_0
+
\frac{\sigma^2}{n}
\operatorname{tr}(I_n-\mathsf H)^2.
\]
The resulting in-sample residual scale generally contains nonvanishing shrinkage and leverage effects and therefore need not converge to \(\sigma^2\),
which is why the high-dimensional Bayesian results assume known $\sigma$ or a
separately justified consistent estimator.

\subsection{Proofs for the Quadratic NGRC Results}
\label{app:quadratic_feature_proofs}

The proof of Theorem~\ref{thm:quadratic_ngrc_width} uses two auxiliary results. The first verifies quadratic-form concentration despite dependence among the nonlinear feature coordinates; the second controls the Gaussian
approximation of a linear-plus-quadratic Gaussian functional.

\subsubsection{Quadratic-Form Concentration}
\begin{lemma}
\label{lem:quadratic_concentration}
Let $G\sim N(0,I_d)$,
\(
\phi_d(G)
=
\begin{pmatrix}
G\\
\operatorname{svec}\left(
\frac{GG^\top-I_d}{\sqrt2}
\right)
\end{pmatrix},\) and 
\(p=\frac{d(d+3)}2,
\)
where $\operatorname{svec}$ is an isometry from symmetric matrices under the
Frobenius norm to Euclidean vectors. Then
\(
\mathbb E\phi_d(G)=0,\) and
\(\mathbb E[\phi_d(G)\phi_d(G)^\top]=I_p.
\)
For every symmetric $B\in\mathbb R^{p\times p}$,
\[
\operatorname{Var}\{\phi_d(G)^\top B\phi_d(G)\}
\le
(4d^3+30d^2+38d)\|B\|_{\mathrm{op}}^2
\le72d^3\|B\|_{\mathrm{op}}^2,
\]
and consequently
\[
\mathbb E\left[
\left\{
\frac{
\phi_d(G)^\top B\phi_d(G)-\operatorname{tr}(B)
}{p}
\right\}^2
\right]
\le
\frac{288}{d}\|B\|_{\mathrm{op}}^2.
\]
Moreover,
\(
\mathbb E\|\phi_d(G)\|^4\le7p^2.
\)
\end{lemma}

\begin{proof}
The mean and covariance identities follow from Gaussian second- and
fourth-moment identities; all linear--quadratic cross moments vanish by
symmetry.

Let $S_G=\|G\|^2$. Since
$\|GG^\top-I_d\|_{\mathrm F}^2=S_G^2-2S_G+d$,
\(
\|\phi_d(G)\|^2
=
S_G+\frac12(S_G^2-2S_G+d)
=
\frac12(S_G^2+d).
\)
For $v\in\mathbb R^d$,
\(
D\phi_d(G)[v]
=
\begin{pmatrix}
v\\
\operatorname{svec}\left(
\frac{Gv^\top+vG^\top}{\sqrt2}
\right)
\end{pmatrix},
\)
and hence
\(
\|D\phi_d(G)[v]\|^2
=
\|v\|^2+\|G\|^2\|v\|^2+(G^\top v)^2
\le
(1+2S_G)\|v\|^2.
\)
For
$f(G)=\phi_d(G)^\top B\phi_d(G)$,
\(
\nabla f(G)
=
2D\phi_d(G)^\top B\phi_d(G).
\)
The Gaussian Poincar\'e inequality therefore gives
\cite{NourdinPeccati2012}
\[
\begin{aligned}
\operatorname{Var}\{f(G)\}
&\le
\mathbb E\|\nabla f(G)\|^2\\
&\le
4\|B\|_{\mathrm{op}}^2
\mathbb E[
(1+2S_G)\|\phi_d(G)\|^2]\\
&=
2\|B\|_{\mathrm{op}}^2
\mathbb E[(1+2S_G)(S_G^2+d)].
\end{aligned}
\]
Using
\(
\mathbb ES_G=d, \;\;
\mathbb ES_G^2=d(d+2),\) and
\(\mathbb ES_G^3=d(d+2)(d+4),
\)
the final expression equals
\(
(4d^3+30d^2+38d)\|B\|_{\mathrm{op}}^2,
\) which proves the variance bound. Because $p\ge d^2/2$, division by $p^2$
gives the normalized quadratic-form bound.

Finally,
$\mathbb ES_G^4=d(d+2)(d+4)(d+6)$, so
\[
\begin{aligned}
\mathbb E\|\phi_d(G)\|^4
&=
\frac14\mathbb E(S_G^2+d)^2\\
&=
\frac{d^4+14d^3+49d^2+48d}{4}
\le7p^2,
\end{aligned}
\]
where the final inequality follows from
\[
7p^2-
\frac{d^4+14d^3+49d^2+48d}{4}
=
\frac12d(d-1)(d+3)(3d+8)\ge0.
\]
\end{proof}

For independent rows distributed as $\phi_d(G)$, the preceding normalized
quadratic-form concentration verifies the condition used in generalized
Marchenko--Pastur results for isotropic vectors with dependent coordinates;
see \cite{Yaskov2016}.

\subsubsection{Gaussian Approximation for Quadratic Features}
\begin{lemma}
\label{lem:quadratic_gaussian_approx}
Let
\(
F
=
a^\top G
+
\frac{G^\top BG-\operatorname{tr}(B)}{\sqrt2}
+
\sigma Z,
\)
where $G\sim N(0,I_d)$, $Z\sim N(0,1)$ is independent of $G$, $B$ is
symmetric, and $\sigma>0$. Define
\(
V=\sigma^2+\|a\|^2+\operatorname{tr}(B^2).
\)
Then
\[
d_{\mathrm K}
\left(
\frac{F}{\sqrt V},N(0,1)
\right)
\le
\frac{
\sqrt{
\frac92\|Ba\|^2+
2\operatorname{tr}(B^4)
}
}{V}
\le
\frac{3}{\sqrt2\,\sigma}\|B\|_{\mathrm{op}}.
\]
\end{lemma}

\begin{proof}
Gaussian integration by parts gives, for suitable differentiable $f$,
\(
\mathbb E[Ff(F)]
=
\mathbb E[Tf'(F)],
\)
where
\(
T
=
\sigma^2+\|a\|^2
+
\frac{3}{\sqrt2}a^\top BG
+
G^\top B^2G.
\)
Indeed, the linear and noise terms give
\(
\mathbb E[(a^\top G+\sigma Z)f(F)]
=
\mathbb E[
\{\|a\|^2+\sqrt2a^\top BG+\sigma^2\}f'(F)],
\)
while the quadratic term gives
\[
\mathbb E\left[
\frac{G^\top BG-\operatorname{tr}(B)}{\sqrt2}f(F)
\right]
=
\mathbb E\left[
\left\{
\frac{a^\top BG}{\sqrt2}
+
G^\top B^2G
\right\}f'(F)
\right].
\]
Thus $\mathbb ET=V$. The linear and centered quadratic terms in $T$ are
uncorrelated, giving
\[
\operatorname{Var}(T)
=
\frac92\|Ba\|^2
+
2\operatorname{tr}(B^4).
\]
Applying the normal Stein equation for indicator test functions
\cite{NourdinPeccati2012} yields
\[
d_{\mathrm K}(F/\sqrt V,N(0,1))
\le
\mathbb E|1-T/V|
\le
\frac{\sqrt{\operatorname{Var}(T)}}{V}.
\]
Finally,
\(
\|Ba\|^2
\le
\|B\|_{\mathrm{op}}^2\|a\|^2,\) and 
\(
\operatorname{tr}(B^4)
\le
\|B\|_{\mathrm{op}}^2\operatorname{tr}(B^2).
\)
Hence
\(
\sqrt{\operatorname{Var}(T)}
\le
\frac3{\sqrt2}\|B\|_{\mathrm{op}}\sqrt V.
\)
Since $V\ge\sigma^2$, the second bound follows.
\end{proof}

\subsubsection{Proof of Theorem~\ref{thm:quadratic_ngrc_width}}

\begin{proof}
Let
\(
e=w_0-\widehat w.
\)
Conditionally on the training design $R$, Gaussianity of the random-effects
readout and training noise gives
\(
e\mid R\sim N(0,C_e),\) and
\(C_e
=
\frac{s_p^2\lambda^2}{p}A^2
+
\frac{\sigma^2}{n}ASA.
\)
Since $\|A\|_{\mathrm{op}}\le1/\lambda$ and
$\|ASA\|_{\mathrm{op}}\le1/(4\lambda)$,
\(
\|C_e\|_{\mathrm{op}}
\le
\frac{s_p^2}{p}
+
\frac{\sigma^2}{4n\lambda}
=:c_n
=
O(n^{-1}).
\)

Set
\(
v=\Sigma_p^{1/2}e,
\;\;
K_e=\Sigma_p^{1/2}C_e\Sigma_p^{1/2},\) and
\(
Q_n=\|v\|^2.
\)
Then
\(
\|K_e\|_{\mathrm{op}}\le Lc_n,\) and \\
\(\operatorname{tr}(K_e)=T_n=O(1).
\)
Thus
\(
\mathbb E(Q_n\mid R)=T_n,\) and
\(\operatorname{Var}(Q_n\mid R)
=
2\operatorname{tr}(K_e^2)
\le2Lc_nT_n.
\)
Therefore
\(
Q_n=T_n+O_{\mathbb P}(n^{-1/2}).
\)

Next decompose the coefficient vector $v$ according to the linear and quadratic
parts of $\phi_d$, writing
\(
v^\top\phi_d(G)
=
a^\top G+
\frac{G^\top BG-\operatorname{tr}(B)}{\sqrt2}.
\)
For a unit vector $u\in\mathbb R^d$, $u^\top Bu$ is conditionally Gaussian,
with variance at most $Lc_n$. Let $\mathcal N$ be a $1/4$-net of the unit
sphere with $|\mathcal N|\le9^d$. The standard net inequality gives
\(
\|B\|_{\mathrm{op}}
\le
2\max_{u\in\mathcal N}|u^\top Bu|.
\)
A Gaussian union bound therefore yields, for every $t>0$,
\(
\mathbb P\left(
\|B\|_{\mathrm{op}}
>
\sqrt{8Lc_n(d\log9+t)}
\,\middle|\,R
\right)
\le2e^{-t}.
\)
Because
$p=d(d+3)/2\asymp n$, we have $d\asymp n^{1/2}$ and hence
\(
\|B\|_{\mathrm{op}}
=
O_{\mathbb P}\left(
\sqrt{\frac dn}
\right)
=
O_{\mathbb P}(n^{-1/4}).
\)
Condition on $(w_0,R,y)$. A fresh signed prediction residual satisfies
\[
y_*-r_*^\top\widehat w
=
a^\top G_*
+
\frac{G_*^\top BG_*-\operatorname{tr}(B)}{\sqrt2}
+
\sigma Z_*,
\]
and has conditional variance $\sigma^2+Q_n$. By
Lemma~\ref{lem:quadratic_gaussian_approx},
\[
\sup_x
\left|
\mathbb P(
y_*-r_*^\top\widehat w\le x
\mid w_0,R,y)
-
\Phi\left(
\frac{x}{\sqrt{\sigma^2+Q_n}}
\right)
\right|
=
O_{\mathbb P}(n^{-1/4}).
\]
The corresponding absolute-residual CDF differs from the folded-Gaussian CDF
by at most twice this amount.

Conditionally on $(w_0,R,y)$, the calibration observations are independent.
The Dvoretzky--Kiefer--Wolfowitz--Massart inequality
\cite{Massart1990}, together with the preceding Gaussian approximation and
$Q_n=T_n+O_{\mathbb P}(n^{-1/2})$, therefore gives
\(
\widehat q_{1-\alpha}
=
z\sqrt{\sigma^2+T_n}
+
O_{\mathbb P}(n^{-1/4}+m^{-1/2}).
\)
Since $T_n=O(1)$ and $\sigma>0$, the relevant folded-Gaussian density is
uniformly bounded away from zero in a neighborhood of the target quantile.
Squaring the preceding display gives
\(
\frac{W_{C,n}^2}{4z^2}
=
\sigma^2+T_n
+
O_{\mathbb P}(n^{-1/4}+m^{-1/2}).
\)

For the Bayesian interval, conditionally on $R$,
\(
r_*^\top Ar_*
=
\phi_d(G_*)^\top
\Sigma_p^{1/2}A\Sigma_p^{1/2}
\phi_d(G_*).
\)
Because
\(
\|
\Sigma_p^{1/2}A\Sigma_p^{1/2}
\|_{\mathrm{op}}
\le
\frac{L}{\lambda},
\)
Lemma~\ref{lem:quadratic_concentration} gives
\[
\mathbb E\left[
\left\{
\frac{
r_*^\top Ar_*-\operatorname{tr}(\Sigma_pA)
}{n}
\right\}^2
\,\middle|\,R
\right]
\le
\frac{72L^2d^3}{\lambda^2n^2}.
\]
Since $d^2\asymp n$,
\(
\frac1n r_*^\top Ar_*
=
\frac1n\operatorname{tr}(\Sigma_pA)
+
O_{\mathbb P}(n^{-1/4}),
\)
and hence
\(
\frac{W_{B,n}(r_*)^2}{4z^2}
=
\sigma^2+P_n+O_{\mathbb P}(n^{-1/4}).
\)

Finally, using $ASA=A-\lambda A^2$,
\[
\begin{aligned}
P_n-T_n
&=
\frac{\sigma^2}{n}\operatorname{tr}(\Sigma_pA)
-
\frac{\sigma^2}{n}\operatorname{tr}(\Sigma_pASA)
-
\frac{s_p^2\lambda^2}{p}\operatorname{tr}(\Sigma_pA^2)\\
&=
\left(
\frac{\sigma^2\lambda}{n}
-
\frac{s_p^2\lambda^2}{p}
\right)
\operatorname{tr}(\Sigma_pA^2).
\end{aligned}
\]
This proves the finite-sample quadratic-NGRC width comparison.
\end{proof}

When $\Sigma_p=I_p$, Lemma~\ref{lem:quadratic_concentration} verifies the
quadratic-form concentration condition for the Marchenko--Pastur theorem for isotropic vectors with dependent coordinates \cite{Yaskov2016}. Therefore
\(
P_n\to\sigma^2\gamma I_1\) and 
\(T_n\to
s^2\lambda^2I_2+\sigma^2\gamma J,
\)
which yields the limits in Corollary~\ref{cor:hd_phase} despite the dependence
among the quadratic feature coordinates.

\subsubsection{Proof of Corollary~\ref{cor:quadratic_deterministic}}
\begin{proof}
Let
\(
e=w_0-\widehat w.
\)
Since
\(
\widehat w
=
A\left(
Sw_0+\frac1nR^\top\varepsilon
\right)
\)
and $AS=I-\lambda A$, then we have
\(
e
=
\lambda Aw_0-\xi
\) and \(\xi= \frac1nAR^\top\varepsilon.\)
Conditionally on $R$,
\(
\xi\sim N(0,C)\) and \(C=\frac{\sigma^2}{n}ASA.\)

Set
\(
b=\lambda Aw_0\) and \(Q_n=e^\top\Sigma_p e.\)
Then
\(
\mathbb E(Q_n\mid R)
=
b^\top\Sigma_pb
+
\operatorname{tr}(\Sigma_pC)
=
B_n+V_n.
\)
As in the proof of
Theorem~\ref{thm:finite_trace_widths},
the Gaussian quadratic-form identities give
\[
\operatorname{Var}(Q_n\mid R)
=
2\operatorname{tr}
\left[
\left(
\Sigma_p^{1/2}C\Sigma_p^{1/2}
\right)^2
\right]
+
4b^\top\Sigma_pC\Sigma_pb.
\]
Because
\(
\|A\|_{\mathrm{op}}\le\lambda^{-1}\), \(
\|ASA\|_{\mathrm{op}}\le 1 / 4\lambda,
\) and $\|\Sigma_p\|_{\mathrm{op}}$ and $\|w_0\|$ are uniformly bounded,
the right-hand side is $O(n^{-1})$. Hence
\[
Q_n
=
B_n+V_n+O_{\mathbb P}(n^{-1/2}).
\]
It remains to control the distribution of a new prediction residual.
Define
\(
v=\Sigma_p^{1/2}e=v_b-v_\xi,
\)
where
\(
v_b=\lambda\Sigma_p^{1/2}Aw_0
\)
and
\(
v_\xi=\Sigma_p^{1/2}\xi.
\)
Let
\(
D_n
=
\mathcal Q_d(v_b)
=
\mathcal Q_d\!\left(
\lambda\Sigma_p^{1/2}Aw_0
\right).
\)
By linearity of the map $\mathcal Q_d$,
\(
\mathcal Q_d(v)
=
D_n-\mathcal Q_d(v_\xi).
\)
By assumption,
\(
\|D_n\|_{\mathrm{op}}
=
O_{\mathbb P}(\rho_n).
\)
The first term satisfies
\(
\|\mathcal Q_d(v_b)\|_{\mathrm{op}}
=
O_{\mathbb P}(\rho_n)
\)
by assumption.

Conditionally on $R$, $v_\xi$ is Gaussian with covariance
\(
K_\xi
=
(\sigma^2 / n)
\Sigma_p^{1/2}ASA\Sigma_p^{1/2},
\)
and therefore
\(
\|K_\xi\|_{\mathrm{op}}
=
O(n^{-1}).
\)
For each fixed unit vector $u\in\mathbb R^d$,
\(
u^\top\mathcal Q_d(v_\xi)u
\)
is a centered Gaussian linear functional of $v_\xi$ with variance
$O(n^{-1})$. Applying the same $1/4$-net argument used in the
random-effects proof gives
\(
\|\mathcal Q_d(v_\xi)\|_{\mathrm{op}}
=
O_{\mathbb P}
\left(
\sqrt{d / n}
\right).
\)
Since $p=d(d+3)/2\asymp n$, we have $d\asymp n^{1/2}$ and hence
\(
\|\mathcal Q_d(v_\xi)\|_{\mathrm{op}}
=
O_{\mathbb P}(n^{-1/4}).
\)
Consequently,
\(
\|\mathcal Q_d(v)\|_{\mathrm{op}}
=
O_{\mathbb P}(\rho_n+n^{-1/4}).
\)

Conditionally on the fitted training sample, a fresh signed prediction
residual can be written as
\[
y_*-r_*^\top\widehat w
=
a^\top G_*
+
\frac{
G_*^\top\mathcal Q_d(v)G_*
-
\operatorname{tr}\{\mathcal Q_d(v)\}
}{\sqrt2}
+
\sigma Z_*,
\]
and its conditional variance is
\(
\sigma^2+\|v\|^2
=
\sigma^2+Q_n.
\)
Lemma~\ref{lem:quadratic_gaussian_approx} therefore yields
\[
\sup_x
\left|
\mathbb P(
y_*-r_*^\top\widehat w\le x
\mid R,y)
-
\Phi\left(
\frac{x}{\sqrt{\sigma^2+Q_n}}
\right)
\right|
=
O_{\mathbb P}(\rho_n+n^{-1/4}).
\]
The same bound, up to a constant factor, holds for the corresponding
absolute-residual distribution.

Conditionally on the training sample, the calibration observations are
independent. The Dvoretzky--Kiefer--Wolfowitz--Massart inequality and the
preceding Gaussian approximation therefore give
\[
\widehat q_{1-\alpha}
=
z\sqrt{\sigma^2+Q_n}
+
O_{\mathbb P}
\left(
\rho_n+n^{-1/4}+m^{-1/2}
\right).
\]
Since $Q_n=O_{\mathbb P}(1)$ and $\sigma>0$, squaring and using
\(
Q_n
=
B_n+V_n+O_{\mathbb P}(n^{-1/2})
\)
gives
\[
\frac{W_{C,n}^2}{4z^2}
=
\sigma^2+B_n+V_n
+
O_{\mathbb P}
\left(
\rho_n+n^{-1/4}+m^{-1/2}
\right).
\]

The Bayesian calculation does not require randomness of $w_0$. The
quadratic-form concentration argument already used in the random-effects
case gives
\(r_*^\top Ar_* / n
=
\operatorname{tr}(\Sigma_pA) / n
+
O_{\mathbb P}(n^{-1/4}),
\)
and hence
\(
W_{B,n}(r_*)^2 / 4z^2
=
\sigma^2+P_n+O_{\mathbb P}(n^{-1/4}).
\)
Finally,
\(
P_n-V_n
=
\sigma^2\lambda 
\operatorname{tr}(\Sigma_pA^2) / n,
\)
so
\[
\frac{W_{B,n}(r_*)^2-W_{C,n}^2}{4z^2}
=
\frac{\sigma^2\lambda}{n}
\operatorname{tr}(\Sigma_pA^2)
-
B_n
+
O_{\mathbb P}
\left(
\rho_n+n^{-1/4}+m^{-1/2}
\right),
\]
which gives the stated width-difference expansion. Finally, if
\(
\rho_n=O(n^{-1/4}),
\)
then
\(
\rho_n+n^{-1/4}=O(n^{-1/4}),
\)
so the deterministic-readout extension has the same approximation order as
the random-effects result.
\end{proof}

\subsection{Isotropic Marchenko-Pastur Specialization}
\label{app:mp_specialization}

\subsubsection{Proof of Corollary~\ref{cor:hd_phase}}
\begin{proof}
Under the isotropic Gaussian design, $\Sigma=I_p$. Write
\(
S=U\Lambda U^\top\) and \(A=(S+\lambda I_p)^{-1}
=
U(\Lambda+\lambda I_p)^{-1}U^\top.
\)
The Marchenko--Pastur theorem
\cite{MarchenkoPastur1967,DobribanWager2018}, together with bounded
continuity of the functions
\(
x\mapsto 1/(x+\lambda),\)
\(
x\mapsto 1 / (x+\lambda)^2,
\) and
\(x\mapsto x / (x+\lambda)^2,
\)
gives
\(
\operatorname{tr}(A)/p
\xrightarrow{\mathbb P}
I_1,
\)
\(\operatorname{tr}(A^2)/p
\xrightarrow{\mathbb P}
I_2,
\)
and
\(
\operatorname{tr}(ASA)/p
\xrightarrow{\mathbb P}
J.
\)
Since $p/n\to\gamma$,
\(
P_n
=
\sigma^2 / n \operatorname{tr}(A)
=
\sigma^2 / n
\operatorname{tr}(A)
\xrightarrow{\mathbb P}
\sigma^2\gamma I_1,
\)
and similarly,
\(
V_n
=
\sigma^2 / n \operatorname{tr}(ASA)
\xrightarrow{\mathbb P}
\sigma^2\gamma J.
\)

It remains to identify the limit of the bias term
\(
B_n
=
\lambda^2w_0^\top A^2w_0.
\)
Let
\(
D=(\Lambda+\lambda I_p)^{-2},
\)
so that
\(
A^2=UDU^\top.
\)
Because the isotropic Gaussian design is orthogonally invariant, conditional
on the eigenvalues $\Lambda$, the direction
\(
u= U^\top w_0/ \|w_0\|
\)
is uniformly distributed on the unit sphere in $\mathbb R^p$. Hence
\(
w_0^\top A^2w_0
=
\|w_0\|^2u^\top Du.
\)

For $u$ uniformly distributed on the unit sphere,
\(
\mathbb E(u_i^2)= 1/p,\)
\(\mathbb E(u_i^4)= 3 / p(p+2),
\) and
\( E(u_i^2u_j^2)= 1 / p(p+2) \) for \(i\neq j.
\)
Therefore, conditional on $\Lambda$,
\(
\mathbb E\left[
w_0^\top A^2w_0
\,\middle|\,
\Lambda
\right]
=
(\|w_0\|^2 / p)
\operatorname{tr}(A^2).
\)
The same sphere-moment identities give
\[
\operatorname{Var}\left(
w_0^\top A^2w_0
\,\middle|\,
\Lambda
\right)
=
\frac{2\|w_0\|^4}{p(p+2)}
\left\{
\operatorname{tr}(A^4)
-
\frac{\operatorname{tr}(A^2)^2}{p}
\right\}
\le
\frac{2\|w_0\|^4}{p(p+2)}
\operatorname{tr}(A^4).
\]
Since
\(
\|A\|_{\mathrm{op}}\le\lambda^{-1},
\)
we have
\(
\operatorname{tr}(A^4)
\le
p\|A\|_{\mathrm{op}}^4
\le
p / \lambda^4,
\)
and hence \\
\(
\operatorname{Var}\left(
w_0^\top A^2w_0
\,\middle|\,
\Lambda
\right)
\le
2\|w_0\|^4 / (p+2)\lambda^4.
\)
Because $\|w_0\|^2\to s^2<\infty$, the right-hand side converges to zero.
Conditional Chebyshev's inequality therefore yields
\[
w_0^\top A^2w_0
-
\frac{\|w_0\|^2}{p}\operatorname{tr}(A^2)
\xrightarrow{\mathbb P}
0.
\]
Together with
\(
\|w_0\|^2\to s^2
\)
and
\(
\operatorname{tr}(A^2) / p
\xrightarrow{\mathbb P}
I_2,
\)
this gives
\(
B_n
=
\lambda^2w_0^\top A^2w_0
\xrightarrow{\mathbb P}
s^2\lambda^2I_2.
\)

Theorem~\ref{thm:finite_trace_widths} now implies

\[
\frac{W_{B,n}(r_*)^2}{4z^2}\xrightarrow{\mathbb P}\sigma^2(1+\gamma I_1),
\qquad
\frac{W_{C,n}^2}{4z^2}\xrightarrow{\mathbb P}\sigma^2(1+\gamma J)+s^2\lambda^2I_2.
\]
Finally,
\(
ASA=A-\lambda A^2,
\)
so
\(
J=I_1-\lambda I_2.
\)
Therefore
\[
\begin{aligned}
&\sigma^2(1+\gamma I_1)
-
\left\{
\sigma^2(1+\gamma J)
+s^2\lambda^2I_2
\right\}\\
&\qquad
=
\sigma^2\gamma(I_1-J)
-s^2\lambda^2I_2\\
&\qquad
=
\sigma^2\gamma\lambda I_2
-s^2\lambda^2I_2\\
&\qquad
=
\lambda I_2
(\sigma^2\gamma-s^2\lambda).
\end{aligned}
\]
Multiplying by $4z^2$ proves
\eqref{eq:hd_phase}. Since $\lambda>0$ and $I_2>0$, the sign of the limiting
squared-width difference is determined by
$\sigma^2\gamma-s^2\lambda$, giving the three stated asymptotic ordering
regimes.
\end{proof}

\subsubsection{Bayesian coverage in the proportional regime.}
\label{app:hd_coverage}
Let \( v_B = \sigma^2(1+\gamma I_1), \) and \\ \( v_C = \sigma^2(1+\gamma J)+s^2\lambda^2I_2. \) The preceding proof shows that the Bayesian squared radius satisfies \(W_{B,n}^2/4 \xrightarrow{\mathbb P} z^2v_B, \) and hence \( W_{B,n} / 2 \xrightarrow{\mathbb P} z\sqrt{v_B}. \) 

To identify the limiting distribution of the prediction error, define \( E_n = y_*-r_*^\top\widehat w. \) Conditional on the fitted training sample $\mathcal D_{\mathrm{tr}}$, the Gaussian design gives \( E_n \mid\mathcal D_{\mathrm{tr}} \sim N(0,\sigma^2+Q_n), \) where \( Q_n = (\widehat w-w_0)^\top \Sigma (\widehat w-w_0). \) The proof of Theorem~\ref{thm:finite_trace_widths} and the preceding Marchenko--Pastur limits give \(\sigma^2+Q_n \xrightarrow{\mathbb P} v_C. \) For each fixed $t\in\mathbb R$, \[ \mathbb E[ e^{itE_n} \mid\mathcal D_{\mathrm{tr}}] = \exp\left\{ -\frac{t^2}{2}(\sigma^2+Q_n) \right\} \xrightarrow{\mathbb P} \exp\left(-\frac{t^2v_C}{2}\right). \] The conditional characteristic functions are bounded in absolute value by one. Hence bounded convergence along subsequences, or equivalently dominated convergence together with convergence in probability, gives \( \mathbb E e^{itE_n} \longrightarrow \exp\left(-t^2v_C/ 2 \right). \) By L\'evy's continuity theorem, \( E_n \xrightarrow{d} N(0,v_C). \) Since \(W_{B,n} / 2 \xrightarrow{\mathbb P} z\sqrt{v_B}, \) Slutsky's theorem gives \[ \mathbb P\left( |E_n| \le \frac{W_{B,n}}2 \right) \longrightarrow \mathbb P\left( |N(0,v_C)| \le z\sqrt{v_B} \right). \] Therefore \[ \mathbb P\{y_*\in C_B(r_*)\} \longrightarrow 2\Phi\left( z\sqrt{\frac{v_B}{v_C}} \right)-1. \]

\subsubsection{Optimal Ridge Penalty and the Equality Boundary}
\label{app:optimal_ridge}
For $s^2>0$, the limiting prediction risk is
\[
\mathcal E(\lambda)
=
\sigma^2+
\int
\frac{s^2\lambda^2+\sigma^2\gamma x}
{(x+\lambda)^2}
\,dF_\gamma^{\mathrm{MP}}(x).
\]
Differentiating under the integral gives
\(
\mathcal E'(\lambda)
=
2(s^2\lambda-\sigma^2\gamma)
\int
\frac{x}{(x+\lambda)^3}
\,dF_\gamma^{\mathrm{MP}}(x).
\)
The integral is positive, so the unique minimizer is
\(
\lambda_*=\frac{\sigma^2\gamma}{s^2}.
\)
This is the standard optimal-ridge scaling in the isotropic proportional
regime; see \cite{DobribanWager2018}. At this value,
$s^2\lambda_*=\sigma^2\gamma$, which is exactly the equality boundary in
Corollary~\ref{cor:hd_phase}.

\subsubsection{Relation to the Fixed-Dimensional Regime}
\label{app:hd_fixed_relation}

The proportional-growth result does not reduce automatically to
Theorem~\ref{thm:fixed_width_discrepancy} by taking $\gamma\downarrow0$.
The high-dimensional analysis holds the normalized ridge penalty $\lambda>0$
fixed. Under this scaling, a nonvanishing shrinkage bias may persist even as
the aspect ratio decreases. Recovering the fixed-dimensional result therefore
additionally requires $\lambda\to0$ on the normalized covariance scale, or
more generally requires the resulting prediction bias to vanish. The limits
$p/n\to0$ and $\lambda\to0$ therefore represent distinct asymptotic
operations.

\subsection{Extension to Temporally Dependent Observations}
\label{app:dependent_transfer}

The preceding high-dimensional results use an independent-sample benchmark.
We now give one sufficient condition under which probability statements for
independent NGRC observation windows can be transferred to a dependent time
series.

Suppose the underlying process is strictly stationary and absolutely regular
with coefficient $\beta(h)$, using the total-variation convention described in
Appendix~\ref{app:technical_tools}. A feature--response pair at forecast origin
$t$ depends on observations in the window
\(
[t-L_n,t+H],
\;\;
L_n=(k_n-1)s_n.
\)
Let $t_1<\cdots<t_N$ denote retained forecast origins and suppose
\begin{equation}
t_{j+1}-t_j
\ge
L_n+H+h_n,
\qquad
h_n\ge1.
\label{eq:window_separation}
\end{equation}

\begin{proposition}[Transfer under absolute regularity]
\label{prop:beta_transfer}
Under \eqref{eq:window_separation}, the $N$ retained observation windows can be
coupled to mutually independent copies with the same marginal window
distribution such that the probability of at least one coupling failure is at
most
\(
(N-1)\beta(h_n).
\)
Consequently, if an event has failure probability at most $\delta$ under the
corresponding independent-window model, then its failure probability under the
dependent process is at most
\(
\delta+(N-1)\beta(h_n).
\)
In particular, whenever the independent copies satisfy the assumptions required
for the standard split-conformal guarantee,
\(
\mathbb P\{y_*\in C_C(r_*)\}
\ge
1-\alpha-(N-1)\beta(h_n).
\)
\end{proposition}

\begin{proof}
The latest observation in the window associated with $t_j$ is $t_j+H$, while
the earliest observation in the next retained window is $t_{j+1}-L_n$.
Therefore \eqref{eq:window_separation} implies
\(
(t_{j+1}-L_n)-(t_j+H)\ge h_n.
\)
By Berbee's coupling principle \cite{Berbee1979}, the $(j+1)$st window can be
coupled to an independent copy having the same marginal distribution, with
failure probability at most $\beta(h_n)$. Applying the coupling successively
over the $N-1$ transitions and using the union bound gives total coupling
failure probability at most
\(
(N-1)\beta(h_n).
\)

On the event that all couplings succeed, every statistic computed from the
retained windows agrees with the corresponding statistic computed from the
independent copies. Hence, if $\mathcal E$ has failure probability at most
$\delta$ under the independent-window model,
\[
\mathbb P_{\mathrm{dep}}(\mathcal E^c)
\le
\mathbb P_{\mathrm{ind}}(\mathcal E^c)
+
\mathbb P(\text{coupling failure})
\le
\delta+(N-1)\beta(h_n).
\]
Taking $\mathcal E$ to be the split-conformal coverage event gives the stated
bound.
\end{proof}

\paragraph{Scope of the transfer result.}
The separation condition applies to every pair of successive retained forecast
windows, not merely to the boundaries between fitting, calibration, and test
blocks. Thus Proposition~\ref{prop:beta_transfer} describes a thinned sequence
of sufficiently separated NGRC forecast origins. It does not provide an
independence approximation for the ordinary overlapping delay windows used at
every consecutive forecast origin.

\paragraph{Sampling cost under quadratic features.}
The separation requirement can be costly in the proportional quadratic-NGRC
regime. With a full quadratic lag map, $p\asymp k_n^2$. If $p\asymp n$, then
$k_n\asymp n^{1/2}$. For fixed lag spacing,
$L_n=(k_n-1)s_n\asymp n^{1/2}$. Retaining $N\asymp n$ windows while enforcing
$t_{j+1}-t_j\ge L_n+H+h_n$ therefore requires an underlying time series of
length at least order $NL_n\asymp n^{3/2}$, even before accounting for any
additional mixing gap $h_n$. The proposition should therefore be viewed as a
sufficient theoretical bridge for heavily separated windows rather than as a
direct asymptotic description of ordinary overlapping NGRC forecasts.

The result is also only a transfer principle. It transfers probability
statements from an independent-window model to sufficiently separated
dependent windows, but it does not establish Gaussian feature marginals, the
linear readout model, quadratic-feature concentration, or a Marchenko--Pastur
law for an arbitrary dependent NGRC process. Those conditions must be verified
separately.

For example, if $\beta(h)\le Ce^{-ch}$, choosing
\(
h_n\ge \frac1c\log\{C(N-1)/\eta_n\}
\)
makes the additional dependence error at most $\eta_n$. Thus preventing
overlap between adjacent NGRC windows is not sufficient for approximate
independence: an additional gap is required relative to the decay of temporal
dependence. For related conformal methods for dependent data, see
\cite{ChernozhukovWuthrichZhu2018}.

\section{Additional Simulation Results}
\label{app:additional_simulations}

\subsection{Experiment 1: Fixed-Dimensional Simulation Details}
\label{app:sim_fixed_details}

For each replication of the experiment in Section~\ref{subsec:sim_fixed_residual_shape}, we
generate
\(
r_i\overset{\mathrm{i.i.d.}}{\sim}N(0,I_p),
\)
\(
y_i=r_i^\top w_0+e_i,
\)
with $p=5$, $w_0=(-0.4,-0.2,0,0.2,0.4)^\top$, and
$\mathbb E[e_i^2]=1$. We consider four unit-variance residual distributions:
Gaussian $N(0,1)$, Laplace with scale $1/\sqrt{2}$, standardized Student-$t_5$
errors $\sqrt{3/5}\,T_5$, and centered exponential errors $E-1$ with
$E\sim\operatorname{Exp}(1)$.

We use
\(
\alpha\in\{0.01,0.05,0.10,0.20\},
\)
$n_{\rm fit}=800$, $n_{\rm cal}=800$, $n_{\rm test}=3000$, and normalized
ridge penalty $\lambda=0.01$. Results are averaged over 300 independent
replications. For each configuration, we record the Bayesian--conformal width
difference and marginal coverage of both procedures. The empirical width
difference and Bayesian coverage are compared with the limits in
Theorem~\ref{thm:fixed_width_discrepancy} and
Corollary~\ref{cor:fixed_coverage}, while split-conformal coverage is compared
with $1-\alpha$.

\subsection{Experiment 2: Verification of Individual Width Approximations}
\label{app:sim_hd_width_verification}

As an additional check for the isotropic Gaussian experiment in Section~\ref{subsec:sim_hd_widths}, we compare the empirical Bayesian and split-conformal squared widths with their corresponding finite-trace
approximations,
\(
W_B^2 \approx 4z^2(\sigma^2+P_n) \) and
\(
W_C^2 \approx 4z^2(\sigma^2+B_n+V_n),
\)
where \(z=z_{1-\alpha/2}\). Figure~\ref{fig:sim_hd_width_verification} shows close agreement for both methods across the simulated aspect ratios, ridge penalties, and signal strengths, where applicable. Thus, the close agreement in \(W_B^2-W_C^2\) reflects accurate approximation of the underlying width components rather than merely offsetting errors.

\begin{figure}[t]
    \centering
    \includegraphics[width=0.7\textwidth]{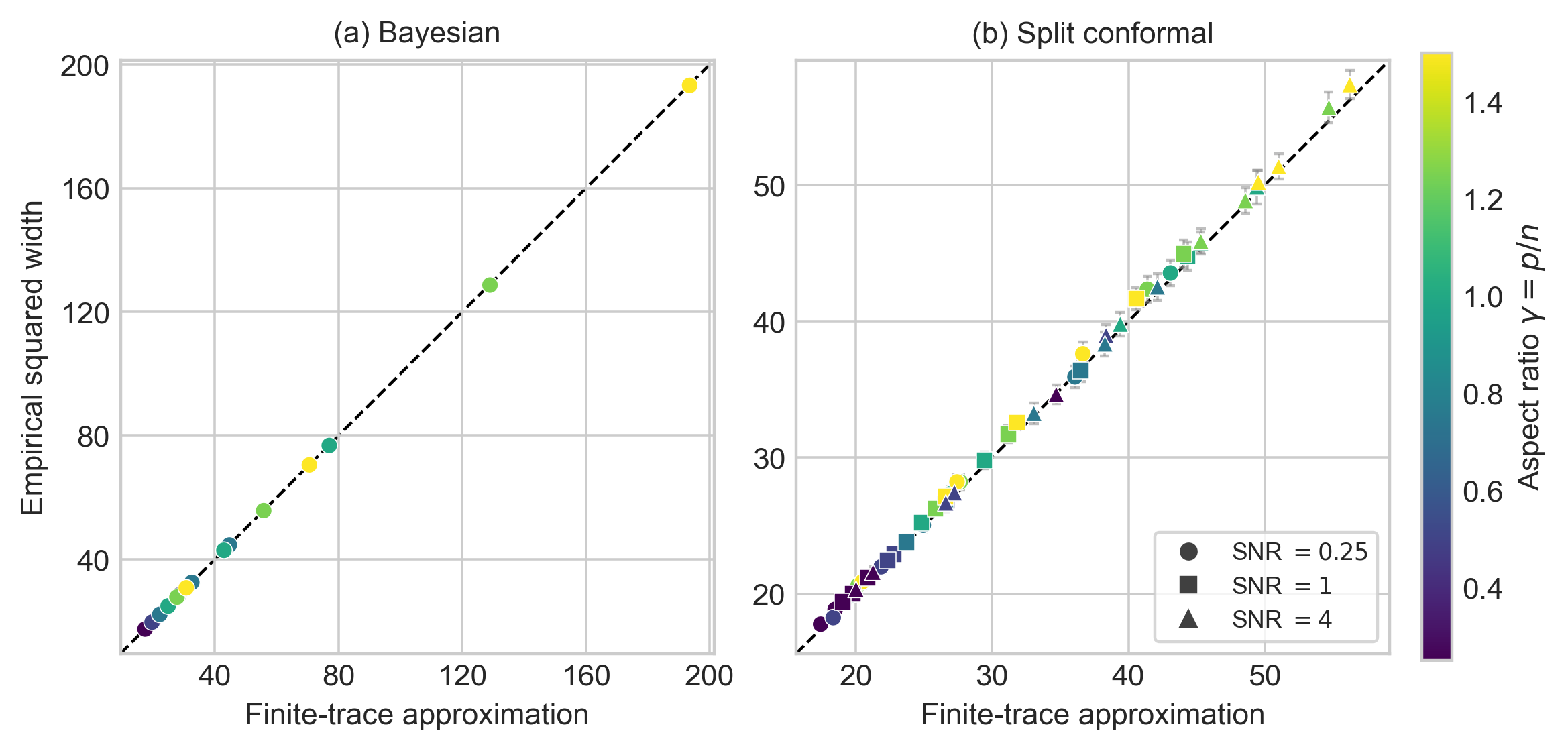}
    \caption{ Individual finite-trace width verification for Experiment~2. Points compare cell-average empirical squared widths with their finite-trace approximations for Bayesian intervals (left) and split-conformal intervals (right). Error bars are 95\%  confidence intervals for the cell means over 100 replications. Color denotes the aspect ratio \(\gamma=p/n\) and markers in the split-conformal panel denote signal-to-noise ratio. The cells span \(\lambda\in\{0.05,0.2,1\}\). The dashed line denotes equality.
    }
    \label{fig:sim_hd_width_verification}
\end{figure}

\subsection{Experiment 3: Ridge-Transformed Diffuseness Diagnostic}
\label{app:sim_quadratic_diagnostics}

The concentrated and diffuse constructions in
Section~\ref{subsec:sim_temporal_dependence} specify raw quadratic coefficient
matrices with equal Frobenius norm. Corollary~\ref{cor:quadratic_deterministic},
however, imposes its diffuseness condition on the ridge-transformed bias
direction rather than on the raw signal. We therefore record, in every
replication of the standardized independent quadratic experiment,
\(
\|D_n\|_{\mathrm{op}}
=
\left\|\mathcal Q_d\!\left(\lambda\Sigma_p^{1/2}A_nw_0\right)\right\|_{\mathrm{op}},
\)
where $A_n=(R^\top R/n_{\rm fit}+\lambda I_p)^{-1}$. For the centered and
variance-standardized independent quadratic feature map, $\Sigma_p=I_p$.

\begin{table}[t]
\centering
\small
\caption{Raw and ridge-transformed quadratic diffuseness diagnostics for
Experiment~3. Entries are means $\pm$ standard errors over 150 replications.
The generating-matrix norm refers to the unnormalized matrix $M_d$ in
$G^\top M_dG-\operatorname{tr}(M_d)$.}
\label{tab:sim_quadratic_diffuseness}
\begin{tabular}{clrrr}
\toprule
$d$ & Signal structure & $\|M_d\|_{\rm op}$ & $\|D_n\|_{\rm op}$ & Residual KS \\
\midrule
8  & Concentrated & 1.414 & $0.973\pm0.013$ & $0.140\pm0.001$ \\
8  & Diffuse      & 0.500 & $0.545\pm0.004$ & $0.079\pm0.001$ \\
12 & Concentrated & 1.414 & $0.941\pm0.010$ & $0.133\pm0.001$ \\
12 & Diffuse      & 0.408 & $0.444\pm0.002$ & $0.065\pm0.001$ \\
16 & Concentrated & 1.414 & $0.923\pm0.008$ & $0.130\pm0.001$ \\
16 & Diffuse      & 0.354 & $0.386\pm0.002$ & $0.059\pm0.001$ \\
20 & Concentrated & 1.414 & $0.901\pm0.006$ & $0.125\pm0.001$ \\
20 & Diffuse      & 0.316 & $0.346\pm0.001$ & $0.052\pm0.001$ \\
24 & Concentrated & 1.414 & $0.885\pm0.006$ & $0.125\pm0.001$ \\
24 & Diffuse      & 0.289 & $0.318\pm0.001$ & $0.050\pm0.001$ \\
\bottomrule
\end{tabular}
\end{table}

Table~\ref{tab:sim_quadratic_diffuseness} shows that the transformed direction
becomes increasingly diffuse under the diffuse construction: its mean operator
norm decreases from $0.545$ at $d=8$ to $0.318$ at $d=24$. The concentrated
construction remains substantially larger, decreasing only from $0.973$ to
$0.885$. Over the same dimensions, the diffuse residual-normality KS distance
decreases from $0.079$ to $0.050$, whereas the concentrated value remains
$0.125$ at $d=24$. These finite-dimensional diagnostics directly examine the
condition in Corollary~\ref{cor:quadratic_deterministic} and are consistent with
the Gaussian-approximation mechanism in
Lemma~\ref{lem:quadratic_gaussian_approx}; they are not presented as proof of
the asymptotic condition.

\begin{table*}[t]
\centering
\scriptsize
\caption{Width, coverage, and residual-normality results for Experiment~3 at
$d=24$. Entries are means $\pm$ standard errors over 150 replications. Trace
denotes the finite-trace approximation.}
\label{tab:sim_quadratic_widths}
\resizebox{\textwidth}{!}{%
\begin{tabular}{llrrrrrrr}
\toprule
Feature design & Signal & Empirical $W_B$ & Trace $W_B$ & Empirical $W_C$ & Trace $W_C$ & Bayesian cov. & SCP cov. & Residual KS \\
\midrule
Independent Gaussian & Concentrated & $1.165\pm0.000$ & $1.165\pm0.000$ & $3.732\pm0.016$ & $3.715\pm0.010$ & $0.462\pm0.001$ & $0.950\pm0.001$ & $0.019\pm0.000$ \\
Independent Gaussian & Diffuse      & $1.165\pm0.000$ & $1.166\pm0.000$ & $3.743\pm0.015$ & $3.747\pm0.011$ & $0.459\pm0.002$ & $0.949\pm0.001$ & $0.020\pm0.000$ \\
Quadratic Gaussian   & Concentrated & $1.169\pm0.000$ & $1.173\pm0.000$ & $3.843\pm0.035$ & $3.870\pm0.023$ & $0.542\pm0.002$ & $0.950\pm0.001$ & $0.125\pm0.001$ \\
Quadratic Gaussian   & Diffuse      & $1.169\pm0.000$ & $1.173\pm0.000$ & $3.830\pm0.015$ & $3.872\pm0.010$ & $0.448\pm0.001$ & $0.951\pm0.001$ & $0.050\pm0.001$ \\
Temporal quadratic   & Concentrated & $1.054\pm0.001$ & $1.055\pm0.001$ & $2.889\pm0.024$ & $2.886\pm0.013$ & $0.637\pm0.002$ & $0.949\pm0.001$ & $0.121\pm0.001$ \\
Temporal quadratic   & Diffuse      & $1.055\pm0.001$ & $1.057\pm0.001$ & $2.336\pm0.032$ & $2.373\pm0.025$ & $0.633\pm0.005$ & $0.942\pm0.002$ & $0.049\pm0.001$ \\
\bottomrule
\end{tabular}%
}
\end{table*}

The empirical and finite-trace Bayesian widths differ by at most $0.004$ at
$d=24$, and the conformal-width approximations are also close. Split-conformal
coverage remains near the nominal $0.95$ level. The substantially smaller
residual KS distance for the diffuse quadratic signal is consistent with the
corresponding reduction in $\|D_n\|_{\mathrm{op}}$.

\subsection{Experiment 3: Supplementary Temporal-Transfer Stress Test}
\label{app:sim_temporal_transfer}

To supplement the quadratic-feature experiment in Section~\ref{subsec:sim_temporal_dependence}, we examine whether a temporally dependent quadratic-window design approaches its same-marginal independent-window benchmark as the separation between forecast origins increases.

We generate a stationary Gaussian AR(1) process,
\(
x_t
=
\rho x_{t-1}
+
\sqrt{1-\rho^2}\,\varepsilon_t,\) with \(
\varepsilon_t\overset{\mathrm{i.i.d.}}{\sim}N(0,1),
\)
and \(\rho=0.8\). At each retained forecast origin, we form a lag vector of base dimension \(d=20\) and construct its centered, variance-standardized quadratic coordinates. The resulting feature dimension is
\(
p=(d(d+1)) / 2 =210.
\)
The signal is the diffuse diagonal quadratic form used in the main quadratic-feature experiment, with total signal magnitude \(2\). We use \(n=m=2p=420\) fitting and calibration observations, 2,000 test observations, noise standard deviation \(\sigma=0.25\), normalized ridge penalty
\(\lambda=0.5\), and miscoverage level \(\alpha=0.05\).

Forecast origins are retained at separations
\(
s\in\{1,4,16,64\}.
\)
For comparison, we also generate independent Gaussian lag windows with the same within-window AR(1) covariance. This benchmark therefore preserves the marginal feature-response distribution while removing dependence across forecast origins. Results are averaged over 150 independent simulation
replications.

We assess temporal transfer using three diagnostics. First, we compare the empirical split-conformal width \(W_C\) with the finite-trace approximation
\(
W_C^{\mathrm{tr}}
=
2z_{1-\alpha/2}
\sqrt{\sigma^2+B_n+V_n},
\)
where \(B_n\) and \(V_n\) are the finite-sample squared-bias and estimation-variance terms. Second, for each separation, we compare the distribution of \(W_C\) across replications with its independent-window counterpart using the two-sample Kolmogorov--Smirnov distance
\(
D_{\mathrm{KS}}
=
\sup_x
\left|
\widehat F_s(x)
-
\widehat F_{\mathrm{ind}}(x)
\right|.
\)
We additionally report the KS distance between the standardized prediction residuals and a standard Gaussian distribution. Third, we report empirical split-conformal coverage at every separation.

\begin{figure*}[t]
    \centering
    \includegraphics[width=\textwidth]
    {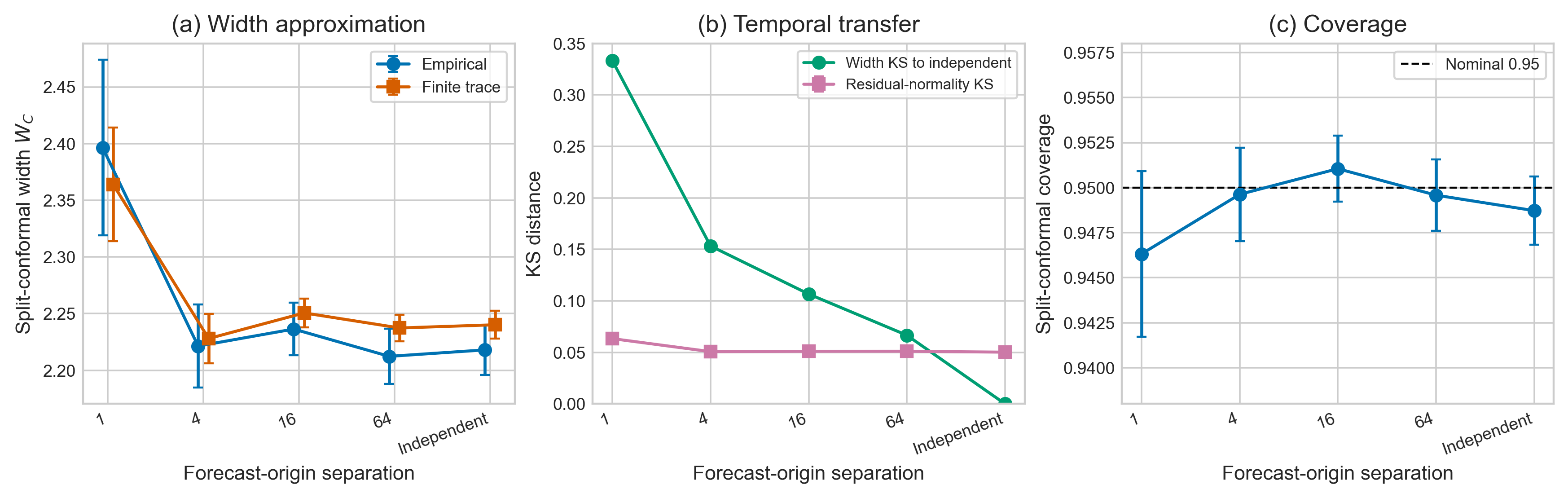}
    \caption{
        Temporal quadratic-window bridge for a stationary Gaussian AR(1) process with \(\rho=0.8\) and base dimension \(d=20\). Panel (a) compares the empirical split-conformal width with its finite-trace approximation. Panel (b) reports the KS distance between the conformal-width distribution and its independent-window benchmark, together with the residual-normality KS diagnostic. Panel (c) reports
        split-conformal coverage, with the dashed line denoting the nominal level \(0.95\). Error bars, where shown, are 95\% confidence intervals for cell means over 150 replications. }
    \label{fig:sim_temporal_dependence}
\end{figure*}

Figure~\ref{fig:sim_temporal_dependence} shows that the empirical and
finite-trace conformal widths remain close throughout the bridge. At
separation one, the empirical and approximated widths are \(2.397\) and
\(2.364\), respectively. At separations \(4\), \(16\), and \(64\), the
corresponding empirical widths are \(2.221\), \(2.236\), and \(2.212\),
compared with finite-trace approximations \(2.228\), \(2.251\), and \(2.237\).
The independent-window benchmark has empirical width \(2.218\) and
approximated width \(2.240\).

The KS distance between the conformal-width distribution and the
independent-window benchmark decreases from \(0.333\) at separation one to
\(0.153\), \(0.107\), and \(0.067\) at separations \(4\), \(16\), and \(64\),
respectively. The residual-normality KS diagnostic similarly decreases from
approximately \(0.063\) at separation one to approximately \(0.051\) at the
larger separations, close to the independent-window value \(0.050\).

Mean split-conformal coverage remains close to the nominal level throughout:
\(0.946\), \(0.950\), \(0.951\), and \(0.950\) at separations \(1\), \(4\),
\(16\), and \(64\), respectively, compared with \(0.949\) for independent
windows. Increasing forecast-origin separation substantially reduces the
cross-window dependence targeted by Proposition~\ref{prop:beta_transfer},
while the width comparison and split-conformal calibration are relatively
insensitive to the remaining dependence in this experiment. The smaller
separations include overlapping windows and should be interpreted as empirical
stress tests rather than configurations covered by
Proposition~\ref{prop:beta_transfer}.

\subsubsection{Absolute-Regularity Penalties Under Temporal Dependence}
\label{app:sim_beta_penalties}

We also examine the absolute-regularity correction terms appearing in Proposition~\ref{prop:beta_transfer}. For the stationary Gaussian AR(1) process with \(\rho=0.8\), the lag-\(h\) absolute-regularity coefficient is
computed numerically from the Gaussian transition distribution.

The current temporal bridge uses lag windows of base dimension \(d=20\).
If consecutive forecast origins are separated by \(\Delta\) time steps, the
distance between the latest observation in the earlier window and the
earliest observation in the later window is
\(
h(\Delta)=\Delta-d+1=\Delta-19.
\)
Consequently, the simulated separations
\(\Delta\in\{1,4,16\}\) produce overlapping lag windows, and an
absolute-regularity transfer bound between disjoint windows is not applicable
to those configurations. The first nonoverlapping configuration occurs at
\(\Delta=20\), for which \(h=1\). At the largest simulated separation,
\(\Delta=64\), the corresponding boundary lag is \(h=45\).

The empirical conformal-width statistic uses \(n=420\) fitting and \(m=420\)
calibration windows. Using the pooled count
\(
N_W=n+m=840,
\)
the corresponding distributional correction is
\(
(N_W-1)\beta(h)=839\beta(h).
\)
A single split-conformal coverage event additionally involves one test
window, giving
\(
N_C=n+m+1=841
\)
and
\(
\mathbb{P}\{Y_*\notin C_C(X_*)\}
\le
\alpha+(N_C-1)\beta(h)
=
0.05+840\beta(h).
\)
These pooled counts are conservative for the simulation implementation,
which generates the fitting, calibration, and test blocks from separate
AR(1) trajectories while retaining temporal dependence within each block.

\begin{figure*}[t]
    \centering
    \includegraphics[width=0.9\textwidth]
    {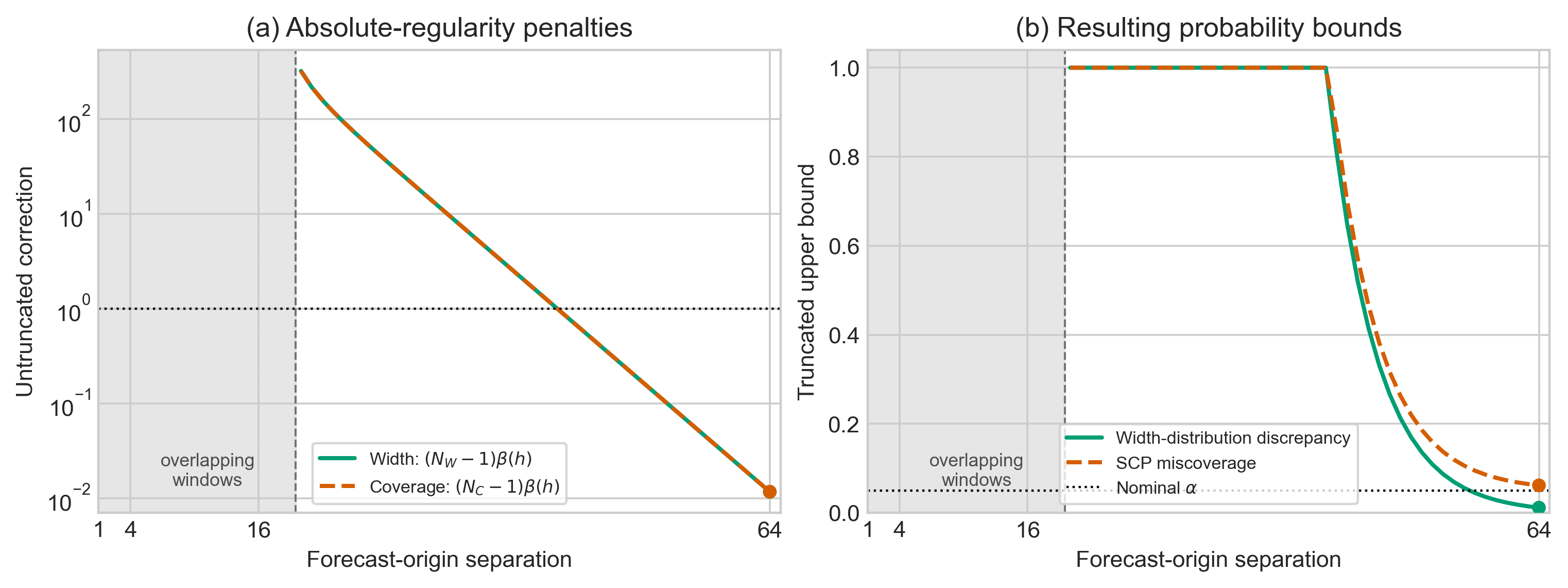}
    \caption{
        Absolute-regularity corrections for the temporal quadratic-window
        experiment with \(\rho=0.8\) and \(d=20\). The shaded region contains
        origin separations \(\Delta<20\), for which consecutive lag windows
        overlap and the disjoint-window transfer bound is not applied.
        Panel (a) shows the untruncated corrections
        \((N_W-1)\beta(h)\) and \((N_C-1)\beta(h)\); the horizontal dotted
        line marks one, above which the corresponding probability correction
        is vacuous. Panel (b) shows the resulting truncated upper bounds for
        the width-distribution discrepancy and split-conformal miscoverage.
        The horizontal dotted line in panel (b) marks the nominal
        miscoverage level \(\alpha=0.05\). Vertical dotted lines identify the
        four forecast-origin separations used in the simulation.
    }
    \label{fig:beta_penalty_decay}
\end{figure*}

Figure~\ref{fig:beta_penalty_decay} shows that the sufficient transfer bounds
are highly conservative at short nonoverlapping separations. The untruncated
corrections remain greater than one until approximately
\(\Delta=45\), corresponding to boundary lag \(h=26\). At the largest
simulated separation, \(\Delta=64\), numerical integration gives
\(
\beta(45)\approx 1.39\times 10^{-5}.
\)
The resulting width-distribution correction is approximately \(0.0117\), and
the split-conformal miscoverage upper bound is approximately
\(
0.05+840\beta(45)\approx 0.0617.
\)
The observed mean split-conformal miscoverage at this separation is
approximately \(0.0504\), consistent with this bound.

The empirical two-sample KS distance in
Figure~\ref{fig:sim_temporal_dependence} is an estimate based on 150
replications per distribution and therefore need not lie below the population
coupling bound pointwise. In particular, simulation variation can make the
empirical KS statistic larger than the bound on the underlying distributional
discrepancy. For the independent-window benchmark,
\(\beta(h)=0\), so the width-distribution correction is zero and the
miscoverage upper bound reduces to the nominal level \(\alpha=0.05\).

\subsection{Experiment 4: Additional Online-Robustness Details}
\label{app:sim_online_details}

This section provides implementation details for the sequential uncertainty
methods used in Section~\ref{subsec:sim_nonlinear_volterra}. All methods
use the same fixed NGRC ridge point predictor. After the initial fit, the
regression coefficients are not updated, so only the uncertainty estimates
are allowed to adapt as test outcomes become available.

\subsubsection{Updated Gaussian and Rolling SCP Baselines}
\label{app:online_uq_updates}

Let
\(
e_t = y_t-\widehat y_t
\)
denote the one-step-ahead forecast residual at test time $t$. When the
interval for $y_t$ is constructed, only residuals whose outcomes have
already been observed are available. The sequential methods use a rolling
window containing the $M=100$ most recent available residuals. At the
beginning of the test period, this window is initialized using the most recent
residuals from the pre-test calibration block.

For comparison, the frozen Bayesian ridge interval is
\[
C_t^{\mathrm{B}}
=
\left[
\widehat y_t
\pm
z_{1-\alpha/2}\widehat\tau
\sqrt{1+\ell_t}
\right],
\qquad
\ell_t
=
r_t^\top
\left(
R_{\mathrm{fit}}^\top R_{\mathrm{fit}}
+
n_{\mathrm{fit}}\lambda I
\right)^{-1}
r_t,
\]
where $\widehat\tau$ is the root-mean-square residual estimated from the
fitting sample. The scale $\widehat\tau$ remains fixed after deployment,
although the interval width may vary with the predictive-leverage term
$\ell_t$.

\paragraph{Updated Gaussian interval.}

Let $\mathcal I_t$ denote the indices of the $M$ residuals available
immediately before forecast $t$. The sequential forecast-error scale is
estimated as
\(
\widehat\tau_t^2
=
\frac{1}{M}
\sum_{j\in\mathcal I_t} e_j^2.
\)
The updated Gaussian interval is
\(
C_t^{\mathrm{UG}}
=
\left[
\widehat y_t
\pm
z_{1-\alpha/2}\widehat\tau_t
\right].
\)
Here $\widehat\tau_t$ estimates the total out-of-sample forecast-error scale,
so an additional Bayesian leverage factor is not applied. This diagnostic
therefore isolates the effect of replacing the frozen training-residual scale
by a sequential estimate based on recently observed forecast errors.

\paragraph{Rolling split conformal prediction.}

Rolling split conformal prediction uses the same fixed point predictor but
updates its residual quantile sequentially. For $j\in\mathcal I_t$, define
the conformity scores
\(
s_j=|e_j|.
\)
Let $s_{t,(1)}\le\cdots\le s_{t,(M)}$ denote the order statistics
of the $M$ scores in the current residual window and define
\(
k=\left\lceil (M+1)(1-\alpha)\right\rceil,
\)
with the convention $s_{t,(M+1)}=+\infty$. Let
$q_t=s_{t,(k)}$. The resulting interval is
\(
C_t^{\mathrm{RSCP}}
=
[\widehat y_t-q_t,\widehat y_t+q_t].
\)
For $M=100$ and $\alpha=0.05$, this gives $k=96$, so the infinite-score convention is not active in the reported experiment. After $y_t$ is
observed, the residual $e_t$ is appended to the rolling history and the
oldest residual is discarded. Thus, the interval for $y_t$ never uses its
own outcome or any future response.

The updated Gaussian and rolling SCP procedures are diagnostic baselines.
Comparing the updated Gaussian interval with frozen Bayesian ridge isolates
the effect of sequential scale estimation, while comparing rolling SCP with
frozen SCP isolates the effect of updating the empirical residual
distribution.

\subsubsection{Numerical Simulation of the Nonlinear Volterra System}
\label{app:volterra_numerics}

The exponentially damped oscillatory memory kernel in
Experiment~\ref{subsec:sim_nonlinear_volterra} admits a recursive
state-space representation. Define
\[
u(t)=\int_0^t e^{-0.2(t-s)}\cos\{2\pi(t-s)\}x(s)\,ds,
\qquad
v(t)=\int_0^t e^{-0.2(t-s)}\sin\{2\pi(t-s)\}x(s)\,ds.
\]
Then $M(t)=u(t)$, and differentiation of the two convolution states gives
\(
u'(t)=x(t)-0.2u(t)-2\pi v(t),\) and 
\(
v'(t)=2\pi u(t)-0.2v(t).
\)
Hence the Volterra integro-differential equation can be written as the
three-dimensional system
\[
\begin{aligned}
x'(t)
&=
\sin(t)-0.1x(t)+0.5\tanh\{u(t)\},\\
u'(t)
&=
x(t)-0.2u(t)-2\pi v(t),\\
v'(t)
&=
2\pi u(t)-0.2v(t),
\end{aligned}
\]
with $x(0)=u(0)=v(0)=0$.

We integrate all three states using forward Euler with internal step
$h=0.005$. Writing $t_j=jh$, the internal recursion is
\[
\begin{aligned}
x_{j+1}
&=
x_j+h\left[\sin(t_j)-0.1x_j+0.5\tanh(u_j)\right],\\
u_{j+1}
&=
u_j+h\left[x_j-0.2u_j-2\pi v_j\right],\\
v_{j+1}
&=
v_j+h\left[2\pi u_j-0.2v_j\right].
\end{aligned}
\]
The homogeneous Euler update for the two memory states has eigenvalues
\(
1-0.2h\pm 2\pi h\,\mathrm{i},
\)
and therefore amplification factor
\(
r(h)
=
\sqrt{(1-0.2h)^2+(2\pi h)^2}.
\)
Stability of this update requires
\[
r(h)<1
\quad\Longleftrightarrow\quad
h<\frac{2(0.2)}{0.2^2+(2\pi)^2}
\approx 0.01012.
\]
Thus, applying the printed Euler recursion directly at the observation
spacing $0.1$ would be unstable because $r(0.1)\approx1.164$. The selected
internal step $h=0.005$ satisfies the stability condition
($r(0.005)\approx0.99949$).

We retain every 20th internal state, so the observation spacing remains
$\Delta t=20h=0.1$. The first 500 retained states (equivalently, the first
10,000 internal Euler steps) are discarded as burn-in before constructing the
fitting, calibration, and test samples. Observation noise is added only after
generating the latent trajectory, so the constant-variance and
variance-change conditions use the same underlying path. Within each paired
replication, the two conditions also use the same standardized Gaussian
noise innovations. They differ only because the noise standard deviation for
test observations 160 through 799 is $0.35$ in the variance-change condition
and remains $0.15$ in the constant-variance condition. As defined in
Section~\ref{subsec:sim_nonlinear_volterra}, observations 160 through 799 in
the variance-change condition form the high-noise period.

As a numerical-accuracy check, we repeated the complete 100-replication
experiment with the internal step halved to $h=0.0025$, retaining every 40th
state so that the observation grid remained unchanged. Across the 3,221
retained latent states, the maximum absolute trajectory difference was
$1.26\times10^{-3}$ and the root-mean-square difference was
$8.86\times10^{-4}$, or $0.044\%$ of the trajectory range. Across the
downstream summaries, the largest absolute change in a coverage quantity was
$3.6\times10^{-4}$ and the largest change in an interval-width quantity was
$6.3\times10^{-5}$. All variance-change recovery fractions were unchanged,
and the largest change in a conditional mean recovery delay was 2.24
observations. The independently selected ResCP configurations were also
identical. Thus, the retained trajectories and
downstream uncertainty results are insensitive to this further refinement of
the integration step.

\subsubsection{ResCP Implementation Details}
\label{app:sim_rescp_details}

We use the ResCP construction described in
Appendix~\ref{app:rescp_details}. For
Experiment 4 in Section~\ref{subsec:sim_nonlinear_volterra}, the principal difference is
the tuning protocol. The common reservoir and sampling hyperparameters are
selected once on an independent constant-variance pilot sequence by minimizing
validation Winkler score, and are then held fixed across all constant-variance
and variance-change evaluation replications. We report an all-history version
that retains every revealed residual and the distribution-shift treatment
described for ResCP, which initializes a calibration set with the $N=800$
pre-test residuals and keeps its size fixed by FIFO replacement during testing.
Both versions apply ResCP's linear recency weighting. The two implementations
therefore differ only in whether the calibration history grows or remains fixed
at 800 residuals. As for the other online methods, a forecast residual is
incorporated only after its corresponding outcome has been observed.

\subsubsection{Recovery Criterion and Summary Statistics}
\label{app:volterra_recovery}

The post-shift high-noise period consists of test observations 160--799. We
report coverage over the first 50 high-noise observations, coverage and mean
interval width over the complete high-noise period, and sustained-recovery
diagnostics.

Recovery delay is measured from the noise increase at test step 160. It is the
number of subsequent test observations until coverage over the trailing window
of 100 test observations first enters $[0.93,0.97]$ and remains in that
interval for every remaining test observation. Recovery fraction is the number
of replications satisfying this sustained-recovery criterion. This is a
stringent retrospective stability criterion rather than a direct estimate of
an intrinsic adaptation time: sampling variation can cause a calibrated
procedure to leave the band, and the reported recovery delay is defined only
among replications that eventually satisfy the criterion. We therefore
interpret recovery delay jointly with the corresponding recovery fraction.

Because the change occurs at test index 160, an absolute recovery index can be
obtained by adding 160 to the reported delay. For example, Adaptive CP's mean
delay of $546.6$ corresponds to mean test index $706.6$ among recovering
replications, while all-history ResCP's mean delay of $610.3$ corresponds to
index $770.3$.

\subsection{Experiment 5: Unequal-Budget Sensitivity Analysis}
\label{app:sim_budget_unequal}

As a sensitivity analysis, we fix the split-conformal budget at $N_C=400$ and
vary the Bayesian budget over
\(
N_B/N_C\in\{0.5,0.75,1,1.5,2\}.
\)
Split conformal is evaluated at
$\rho\in\{0.5,0.6,0.7,0.8\}$, while Bayesian samples are nested across budget
ratios within each replication. All procedures share the same test set.

Figure~\ref{fig:sim_budget_unequal} shows that the equal-budget Bayesian
advantage does not persist uniformly as its fitting budget changes. In the
high-dimensional setting with $p=300$, the first three Bayesian budgets give
$p/N_B=1.5$, $1$, and $0.75$, respectively. Thus the two smallest budgets
place the estimator at or beyond the interpolation threshold, producing wide,
conservative intervals and substantially worse prediction performance.
Performance improves sharply once the fitting sample becomes sufficiently
large relative to the feature dimension. The unequal-budget comparison is
therefore primarily an aspect-ratio sensitivity analysis rather than evidence
of uniform Bayesian dominance or inferiority.

\begin{figure*}[t]
    \centering
    \includegraphics[width=\textwidth]{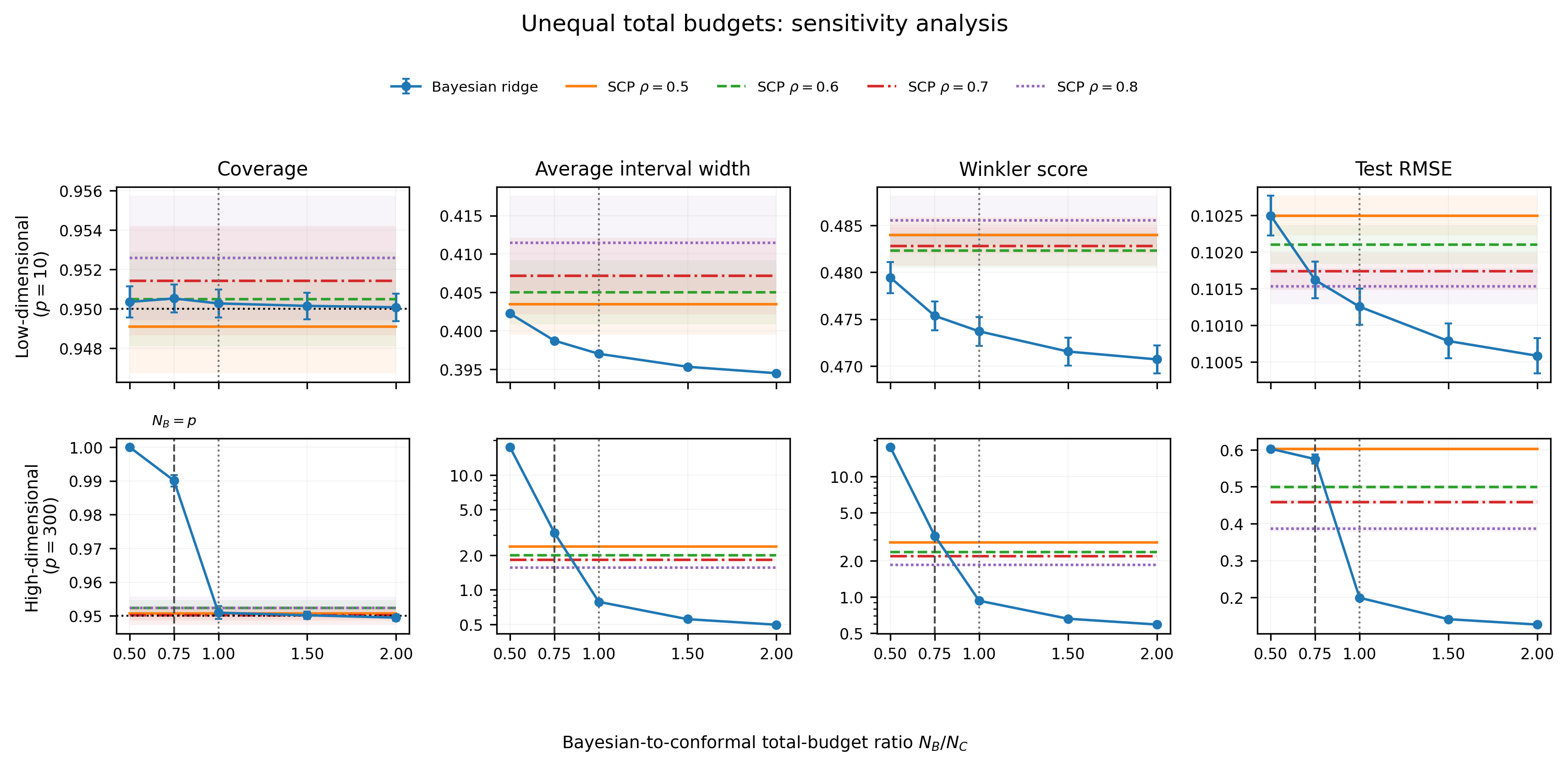}
    \caption{
    Sensitivity to unequal total data budgets. Split conformal uses
    $N_C=400$, while Bayesian ridge varies over
    $N_B/N_C\in\{0.5,0.75,1,1.5,2\}$. The vertical dotted line marks equal
    budgets, and the dashed line in the high-dimensional row marks
    $N_B=p=300$. Results are means over 200 paired replications with
    95\% confidence intervals.
    }
    \label{fig:sim_budget_unequal}
\end{figure*}

\section{Additional Real-Data Details}
\label{app:real_details}

This section provides implementation details and supplementary analyses for Section~\ref{sec:real_data}. We use the same chronological, leakage-safe evaluation protocol described in the main text. The material below records dataset-specific feature construction, tuning procedures, sequential update rules, and robustness analyses that are supplemental to the main text.

\subsection{Datasets, Preprocessing, and Evaluation Protocol}
\label{app:real_protocol}

\paragraph{LA ozone.}
We analyze the \texttt{LA\_O3} object contained in
\texttt{LA\_O3.RData}. The exact data file used in the experiments is included
with the supplementary materials. It provides daily ozone observations for a
location near Los Angeles at longitude $-118.742064$ and latitude
$34.0206066$. The forecasting target is the \texttt{value} field. Observations
run from January 1, 1980, through December 31, 2023, yielding 16,071 daily
values. If multiple records share a date, their \texttt{value} measurements
are averaged before construction of the regular daily grid. The resulting
series contains no missing dates.

\paragraph{Solar.}
We use the Solar-Energy benchmark distributed with the multivariate time-series
repository of Lai et al. The source matrix contains 137 solar-power series from
Alabama sampled every 10 minutes during 2006. The forecasting target is the
cross-sectional mean across the 137 series at each time point, yielding 52,560
observations. The source matrix used in the experiments contains no missing
entries. Because there are 144 observations per day and 1,008 per week, these
periods are also used for deterministic seasonal features.

\paragraph{Beijing PM$_{10}$.}
We use the Beijing Multi-Site Air Quality dataset from the UCI Machine Learning
Repository. The dataset contains hourly air-pollution measurements from 12
monitoring sites in Beijing from March 1, 2013, through February 28, 2017. The
target is the mean PM$_{10}$ concentration across the sites reporting at a
given hour, in $\mu\mathrm{g}\,\mathrm{m}^{-3}$. The regular hourly grid
contains 35,064 observations. After spatial averaging, 83 hours have no
available PM$_{10}$ measurement at any site; these values are filled causally
using the most recent earlier hourly value.

\paragraph{Exchange returns.}
We use the Exchange-Rate benchmark distributed with the multivariate
time-series repository of Lai et al. The dataset contains eight daily
exchange-rate series from 1990 to 2016, listed as Australia, Britain, Canada,
Switzerland, China, Japan, New Zealand, and Singapore. Following this ordering,
we analyze the first series, corresponding to Australia. Let $P_t$ denote the
observed exchange-rate level. We transform the positive level series to scaled
log returns,
\(
r_t=100\{\log(P_t)-\log(P_{t-1})\}.
\)
The 7,588 level observations therefore produce 7,587 daily returns. The
selected series contains no missing values.

\subsubsection{Chronological Splits and Leakage Controls}
\label{app:real_splits}

The primary fixed-origin analysis uses a chronological 40\% fitting, 40\% calibration, and 20\% test split. No observations are shuffled. Standardization parameters are estimated using the fitting block only and are
then applied unchanged to the calibration and test blocks.

Missing observations are filled using past information only. Specifically, an available observation is carried forward to subsequent missing entries, while an unavailable leading prefix is discarded rather than imputed from future values. Seasonal Fourier covariates are deterministic functions of the time index and therefore do not require estimation from future observations.

A feature-response row is assigned to a split according to its target time. For a direct $H$-step forecast with target at time $t$, the corresponding feature vector contains observations only through time $t-H$. We consider direct forecast horizons
\(
H\in\{1,3,7\}.
\)
The resulting processed series lengths and fixed-origin split sizes are shown
in Table~\ref{tab:real_split_sizes}. Counts vary slightly with $H$ because
direct $H$-step feature construction removes the final $H$ forecast origins.

\begin{table*}[t]
\centering
\small
\caption{Processed series lengths and chronological fitting/calibration/test
sample sizes. Each triplet is reported as
$n_{\rm fit}/n_{\rm cal}/n_{\rm test}$.}
\label{tab:real_split_sizes}
\begin{tabular}{lrrrr}
\toprule
Dataset & Observations & $H=1$ & $H=3$ & $H=7$ \\
\midrule
LA ozone & 16,071 & 6420/6420/3210 & 6419/6419/3210 & 6417/6417/3210 \\
Solar & 52,560 & 21009/21009/10506 & 21008/21008/10506 & 21007/21007/10504 \\
Beijing PM$_{10}$ & 35,064 & 14016/14016/7008 & 14015/14015/7008 & 14013/14013/7008 \\
Exchange returns & 7,587 & 3029/3029/1515 & 3028/3028/1515 & 3026/3026/1515 \\
\bottomrule
\end{tabular}
\end{table*}

For sequential uncertainty procedures, a forecast residual becomes available
to the online calibration state only after its corresponding target has been
observed. Thus, for $H>1$, online updating is delayed by the forecast horizon.

To assess sensitivity to the choice of forecast period, we additionally use three shifted chronological evaluation windows with 40\% fitting, 20\% calibration, and 20\% test blocks beginning at 0\%, 10\%, and 20\% of the
available feature rows. These windows overlap, so variability across them is interpreted as forecast-period sensitivity rather than sampling uncertainty.

\subsubsection{NGRC Feature Construction and Ridge Tuning}
\label{app:real_ngrc_details}

The NGRC ridge parameter is selected from
\(
\Lambda=
\{0.001,0.003,0.01,0.03,0.1,0.3,1,3,10\}.
\)
Selection uses a chronological 80/20 split within the fitting block. The
penalty minimizing validation RMSE is retained, after which the ridge readout
is refitted using the complete fitting block. The intercept is not penalized.
Accordingly, the empirical Bayesian interpretation corresponds to a flat prior
on the intercept and independent Gaussian priors on the remaining coefficients.
The theoretical derivations instead use the proper isotropic Gaussian prior on
all coordinates; the unpenalized intercept is an empirical modeling convention.

The dataset-specific lag lengths are 21 for LA ozone, 36 for Solar, 24 for
Beijing PM$_{10}$, and 14 for Exchange returns. The default NGRC feature map
contains the lag coordinates together with their unique quadratic products.
Solar additionally uses Fourier periods 144 and 1008, while Beijing
PM$_{10}$ uses periods 24 and 168, with two harmonics per period. LA ozone
and Exchange returns do not use explicit Fourier covariates.

For each horizon $H$, the NGRC readout is fitted directly to the corresponding
future target rather than recursively iterating a one-step predictor. Thus the
point model and ridge penalty are re-estimated separately for each forecast
horizon.

\subsubsection{Additional Point-Forecasting Baselines}
\label{app:real_point_baselines}

The exploratory architecture comparison uses the same standardized target,
chronological fitting/calibration/test blocks, and one-step forecast targets as
the NGRC analysis. Sequence-model lookback lengths equal the dataset-specific
NGRC delay lengths: 21 observations for LA ozone, 36 for Solar, 24 for Beijing
PM$_{10}$, and 14 for Exchange returns.

The ARIMA baseline is an ARIMA$(3,1,3)$ model with no deterministic trend. Its
parameters are estimated from the fitting prefix. During calibration and
testing, the fitted state is updated sequentially after each newly observed
response, without re-estimating the model order or using future responses.

The GRU consists of one univariate GRU layer with hidden dimension 32, followed
by a linear readout from the final hidden state. The causal Transformer first
embeds each scalar observation into dimension 32 and adds a learned positional
representation. It uses three Transformer encoder layers, two attention heads,
feed-forward dimension 64, GELU activations, dropout probability $0.1$, and an
upper-triangular causal attention mask. A linear readout is applied to the
representation at the final sequence position.

Both neural models are trained with the mean absolute error loss and the Adam
optimizer with learning rate $10^{-3}$. Training uses batch size 32, no
within-epoch shuffling, and at most 200 epochs. The final 20\% of the fitting
block is reserved as a chronological early-stopping set; the calibration block
is not used for neural-model selection. Training stops after 20 epochs without
an improvement exceeding $10^{-7}$ in validation loss. NumPy and PyTorch are
initialized with seed 123. Once training and early stopping are complete, all
point-model parameters are frozen before uncertainty quantification.

\subsubsection{Sequential Updating and Multi-Step Forecasts}
\label{app:real_multistep}

For a direct $H$-step forecast, a forecast residual becomes available only
when its target is observed. Accordingly, all sequential uncertainty methods
use an availability-safe delayed-update protocol.

The first $H-1$ calibration targets are treated as boundary guards because
their forecasts would have been issued before the fitting block closed. The
final $H-1$ calibration targets are not yet observable when the first test
interval is issued. Consequently, the initially available calibration-score
set contains $n_{\rm cal}-2(H-1)$ scores. Frozen split conformal prediction
uses this initial set and remains fixed throughout testing. Hyperparameter
selection uses the same availability-safe sequence.

For adaptive CP, time-weighted CP, and ResCP, the initially unavailable
calibration-tail residuals and subsequent test residuals enter the online
history only when their targets become observable. Ridge validation likewise
excludes the first $H-1$ validation targets following its inner fitting
boundary. Thus no fitting, tuning, or uncertainty procedure uses a response
before its observation time.

\subsection{Uncertainty-Quantification Implementation}
\label{app:real_uq_details}

\subsubsection{Bayesian Ridge and Split Conformal Prediction}
\label{app:real_bayes_scp}

The Bayesian predictive interval estimates its residual scale from the fitting
block only:
\[
\widehat{\tau}_{\mathrm{fit}}
=
\left[
\frac{1}{n_{\mathrm{fit}}}
\sum_{i\in\mathrm{fit}}
(y_i-\widehat y_i)^2
\right]^{1/2}.
\]
Neither calibration nor test observations are used to re-estimate this scale.
The resulting interval therefore retains the fitting-period noise estimate
throughout evaluation.

Ordinary split conformal prediction uses the absolute residuals from the
availability-safe calibration-score set available when the first test forecast
is issued. For $H=1$, this is the complete calibration block. If
$S_{(1)}\le\cdots\le S_{(m)}$ are the ordered available scores, we use
$k_\alpha=\lceil(m+1)(1-\alpha)\rceil$ with the convention
$S_{(m+1)}=+\infty$, and set $\widehat q_{1-\alpha}=S_{(k_\alpha)}$. The
split-conformal quantile is computed before testing and remains fixed during
the test period. Consequently, ordinary SCP does not adapt to test-period
changes in the residual distribution.

The asymmetric SCP baseline instead uses signed calibration residuals and
constructs separate lower and upper residual quantiles with equal nominal tail
probabilities. It uses the same fitted NGRC point predictor and calibration
sample as symmetric SCP. Formally, define the signed calibration residuals
\(
e_i = y_i-\widehat y_i.
\)
Let $\widehat q_{\alpha/2}$ and
$\widehat q_{1-\alpha/2}$ denote the corresponding empirical lower and upper
quantiles, using the same finite-sample conformal quantile convention as the
symmetric SCP interval. The asymmetric interval is
\(
C_{\mathrm{ASCP}}(r_*)
=
\left[
r_*^\top\widehat w+\widehat q_{\alpha/2},
\;
r_*^\top\widehat w+\widehat q_{1-\alpha/2}
\right].
\)
This baseline is included to isolate residual
asymmetry from the additional localization, temporal weighting, and sequential adaptation used by ResCP.

\subsubsection{Adaptive and Time-Weighted Conformal Prediction}
\label{app:real_adaptive_details}

The adaptive conformal learning rate is selected from
\(
\mathcal E=
\{0.001,0.003,0.005,0.01,0.02,0.05\},
\)
and the time-weighted conformal decay parameter is selected from
\(
\mathcal R=
\{0.95,0.98,0.99,0.995,0.997,0.999\}.
\)

For both procedures, the first 60\% of the calibration block supplies the
initial sequential residual history and the final 40\% is used for
hyperparameter validation. The candidate minimizing validation Winkler score
is retained. The selected hyperparameter is fixed before test evaluation,
while the residual history continues to update sequentially using only
outcomes that have already become observable.

Adaptive CP and time-weighted CP are used here as adaptive empirical baselines.
The projection of $\alpha_t$ onto $[0.001,0.999]$ differs from the unprojected
adaptive-conformal recursion for which a telescoping long-run coverage argument
is available, so we do not claim that this argument transfers unchanged to the
projected implementation. Likewise, the time-weighted method normalizes
weights over revealed historical residuals only and does not assign test-point
mass at infinity. It is a weighted residual-quantile baseline rather than the
finite-sample weighted-conformal construction, and no corresponding
finite-sample coverage guarantee is claimed.

For direct $H$-step forecasting, neither adaptive nor time-weighted conformal
prediction uses a score immediately after its forecast is issued. A test
residual and, for adaptive conformal prediction, its coverage error enter the
online history only after the corresponding target becomes observable, giving
an $H$-step feedback delay.

\subsubsection{Reservoir Conformal Prediction}
\label{app:rescp_details}

We additionally evaluate Reservoir Conformal Prediction (ResCP)
\cite{Neglia2026ResCP}. ResCP encodes the revealed signed-residual history in a
random reservoir and compares the current reservoir state with previously
recorded states. Sampling probabilities combine a temperature-scaled softmax
of cosine similarities with the official implementation's linear recency
weights. Residuals are sampled with replacement according to these
probabilities, and empirical lower and upper quantiles of the sampled
residuals define a potentially asymmetric interval around the same point
forecast used by the other uncertainty-quantification methods.

We use the official stochastic sampler and reservoir implementation through a
local residual-bundle adapter. All primary real-data experiments use a fixed
reservoir configuration: reservoir dimension 512, spectral radius $1.2$, leak
rate $0.9$, input scaling $0.25$, connectivity $0.2$, and zero reservoir
noise. These reservoir parameters are not tuned. Residual inputs to the
reservoir are standardized using only calibration residuals available at the
corresponding forecast origin, while sampled residuals remain on the original
standardized-target scale. We use cosine similarity, linear recency decay, and
100 candidate quantile pairs. The adaptive-miscoverage parameter internal to
ResCP is set to zero.

Only the sampling temperature and residual-history policy are selected.
The temperature grid is
$T\in\{0.05,0.10,0.15,0.25\}$.
The candidate numerical history windows are \\
$W\in\{100,250,500,1000,1500,3200,3800,7600\}$.
An all-history policy is included as a separate candidate.

For a direct $H$-step forecast, the first $H-1$ calibration targets are
boundary guards and the final $H-1$ calibration outcomes are excluded from
hyperparameter selection because they are not yet observable when the first
test interval is issued. Temperature and history policy are selected on the
final 10\% of this availability-safe calibration sequence by first minimizing
validation Winkler score. Ties in Winkler score are broken by smaller absolute
validation coverage error, with any remaining ties resolved by candidate-grid
order. The selected configuration is frozen before test evaluation.

During testing, a forecast residual and its associated forecast-origin
reservoir state enter the ResCP history only after the corresponding target
has become observable. Each primary dataset--horizon configuration uses one
prespecified deterministic seed for the reservoir and stochastic sampler. We
separately assess stochastic sensitivity by repeating the final ResCP
evaluation over five seeds while holding the point forecasts, data splits, and
selected hyperparameters fixed; see
Appendix~\ref{app:real_rescp_seed_sensitivity}.

\subsection{Hyperparameter Selection}
\label{app:real_hyperparameters}

All hyperparameters are selected chronologically without using test responses.
For NGRC, the ridge penalty is chosen by validation RMSE on the final 20\% of
the fitting block. For adaptive and time-weighted conformal prediction, the
first 60\% of the availability-safe calibration sequence supplies the initial
residual history and the final 40\% is used for validation. The candidate
minimizing validation Winkler score is retained, with ties resolved by
candidate-grid order. ResCP uses the first 90\% of the availability-safe
calibration sequence as its initial history and the final 10\% for
chronological validation. ResCP candidates are ordered first by validation
Winkler score, then by absolute validation coverage error, and finally by
candidate-grid order.

The ridge grid is
$\lambda\in\{0.001,0.003,0.01,0.03,0.1,0.3,1,3,10\}$,
and the adaptive-CP grid is
$\eta\in\{0.001,0.003,0.005,0.01,0.02,0.05\}$.
For time-weighted CP, the candidate decay parameters are
$\rho\in\{0.95,0.98,0.99,0.995,0.997,0.999\}$.
For ResCP,
$T\in\{0.05,0.10,0.15,0.25\}$ and
$W\in\{100,250,500,1000,1500,3200,3800,7600,\mathrm{all}\}$.
A numerical ResCP window is omitted whenever it exceeds the residual history
available in the chronological tuning prefix.

\begin{table*}[t]
\centering
\small
\caption{Hyperparameters selected for the primary fixed-origin experiments at
95\% nominal coverage. Here $T$ is the ResCP sampling temperature and $W$ is
its residual-history window.}
\label{tab:real_selected_hyperparameters}
\begin{tabular}{lrrrrr}
\toprule
Dataset and horizon & $\lambda$ & Adaptive $\eta$ & TWCP $\rho$ & ResCP $T$ & ResCP $W$ \\
\midrule
Beijing PM$_{10}$, $H=1$ & 0.003 & 0.05  & 0.95  & 0.10 & all \\
Beijing PM$_{10}$, $H=3$ & 0.001 & 0.02  & 0.95  & 0.15 & all \\
Beijing PM$_{10}$, $H=7$ & 0.001 & 0.01  & 0.99  & 0.15 & all \\
Exchange returns, $H=1$  & 10    & 0.01  & 0.95  & 0.15 & 1500 \\
Exchange returns, $H=3$  & 10    & 0.02  & 0.95  & 0.25 & 100 \\
Exchange returns, $H=7$  & 10    & 0.01  & 0.95  & 0.15 & 1000 \\
LA ozone, $H=1$           & 0.03  & 0.001 & 0.999 & 0.10 & 3200 \\
LA ozone, $H=3$           & 1     & 0.001 & 0.999 & 0.10 & all \\
LA ozone, $H=7$           & 0.3   & 0.001 & 0.999 & 0.25 & 3800 \\
Solar, $H=1$              & 0.001 & 0.05  & 0.95  & 0.05 & all \\
Solar, $H=3$              & 0.001 & 0.001 & 0.999 & 0.05 & 7600 \\
Solar, $H=7$              & 0.001 & 0.001 & 0.999 & 0.05 & 7600 \\
\bottomrule
\end{tabular}
\end{table*}

\subsection{Evaluation Metrics}
\label{app:real_metrics}

Let $\{(y_i,\widehat y_i,L_i,U_i)\}_{i=1}^n$ denote the test responses, point
forecasts, and lower and upper prediction limits. Point accuracy is measured by
\[
\operatorname{RMSE}
=
\left\{
\frac1n\sum_{i=1}^n(y_i-\widehat y_i)^2
\right\}^{1/2}.
\]
Empirical coverage is
\(
\widehat{\operatorname{Cov}}
=
\frac1n\sum_{i=1}^n\mathbf 1\{L_i\le y_i\le U_i\},
\)
and mean interval width is
\(
\widehat{\operatorname{Width}}
=
\frac1n\sum_{i=1}^n(U_i-L_i).
\)
We report signed coverage error in percentage points as
\(
\Delta_{\rm cov}=100\{\widehat{\operatorname{Cov}}-(1-\alpha)\}.
\)
Positive values indicate overcoverage and negative values indicate
undercoverage.

For one observation, the Winkler interval score is
\[
W_{\alpha,i}
=
\begin{cases}
U_i-L_i, & L_i\le y_i\le U_i,\\
(U_i-L_i)+\frac{2}{\alpha}(L_i-y_i), & y_i<L_i,\\
(U_i-L_i)+\frac{2}{\alpha}(y_i-U_i), & y_i>U_i.
\end{cases}
\]
The reported Winkler score is $n^{-1}\sum_{i=1}^n W_{\alpha,i}$; smaller
values are better because the score penalizes both wide intervals and missed
observations.

\subsection{Supplementary Theory-Facing Diagnostics}
\label{app:real_theory_details}

The main text reports the quantities most directly connected to the
fixed-dimensional Bayesian--conformal comparison, including
$p/n_{\mathrm{fit}}$, $d_{\mathrm{eff}}/n_{\mathrm{fit}}$, blockwise residual
RMS values, and
\(
\rho_q
=
\widehat q_{1-\alpha,\mathrm{cal}} / \{z_{1-\alpha/2}\widehat\tau_{\mathrm{fit}}\}.
\)

For reference, the one-step fitting and calibration sample sizes and feature
dimensions are
\[
\begin{array}{lrrr}
\toprule
\text{Dataset} & n_{\mathrm{fit}} & n_{\mathrm{cal}} & p\\
\midrule
\text{LA ozone}          & 6420  & 6420  & 253\\
\text{Solar}             & 21009 & 21009 & 711\\
\text{Beijing PM}_{10}   & 14016 & 14016 & 333\\
\text{Exchange returns}  & 3029  & 3029  & 120\\
\bottomrule
\end{array}
\]
where the feature dimension includes the intercept and any deterministic
seasonal covariates used by the corresponding NGRC specification.

These dimensions confirm that the empirical applications are much closer to
the fixed- or effective-low-dimensional regime than to the proportional-growth
benchmark developed in Section~\ref{sec:theory}. The proportional theory is
therefore used to clarify possible finite-sample mechanisms rather than as a
literal asymptotic model for these four datasets.

\subsection{Repeated chronological evaluation}
\label{app:real_rolling}

Table~\ref{tab:real_main} provides the detailed comparison for one fixed
forecast period. To assess whether its conclusions depend on that particular
period, we separately repeat the comparison over three shifted 40\% fit/20\%
calibration/20\% test windows. Table~\ref{tab:real_rolling} reports the mean
and standard deviation across these forecast periods. Because this robustness
analysis uses a shorter calibration block, its values should not be compared
row-for-row with Table~\ref{tab:real_main}; its purpose is to quantify
sensitivity to temporal shifts.

\begin{table*}[t]
\centering
\small
\caption{Repeated chronological one-step results at 95\% nominal coverage.
Entries are means $\pm$ standard deviations over three shifted forecast
periods. RMSE is reported once per dataset because all UQ methods use the
same NGRC point predictor within each origin. Because the windows overlap,
the standard deviations quantify forecast-period sensitivity rather than
sampling uncertainty or standard errors.}
\label{tab:real_rolling}
\begin{tabular}{llrrrr}
\toprule
Dataset & Method & RMSE & Coverage & Width & Winkler \\
\midrule

\multirow{6}{*}{LA ozone}
& Bayesian ridge
& \multirow{6}{*}{$0.383 \pm 0.003$}
& $0.981 \pm 0.006$ & $1.927 \pm 0.123$ & $2.054 \pm 0.087$ \\
& SCP
& & $0.959 \pm 0.008$ & $1.639 \pm 0.057$ & $1.929 \pm 0.008$ \\
& Asymmetric SCP
& & $0.958 \pm 0.006$ & $1.630 \pm 0.047$ & $1.927 \pm 0.011$ \\
& Adaptive CP
& & $0.953 \pm 0.001$ & $1.587 \pm 0.005$ & $1.921 \pm 0.011$ \\
& Time-weighted CP
& & $0.953 \pm 0.001$ & $1.586 \pm 0.002$ & $1.921 \pm 0.011$ \\
& ResCP
& & $0.947 \pm 0.005$ & $1.542 \pm 0.032$ & $1.936 \pm 0.010$ \\
\midrule

\multirow{6}{*}{Solar}
& Bayesian ridge
& \multirow{6}{*}{$0.034 \pm 0.008$}
& $0.905 \pm 0.064$ & $0.111 \pm 0.004$ & $0.233 \pm 0.081$ \\
& SCP
& & $0.920 \pm 0.038$ & $0.117 \pm 0.007$ & $0.221 \pm 0.061$ \\
& Asymmetric SCP
& & $0.922 \pm 0.030$ & $0.116 \pm 0.008$ & $0.217 \pm 0.055$ \\
& Adaptive CP
& & $0.943 \pm 0.003$ & $0.119 \pm 0.022$ & $0.168 \pm 0.009$ \\
& Time-weighted CP
& & $0.875 \pm 0.015$ & $0.123 \pm 0.026$ & $0.182 \pm 0.032$ \\
& ResCP
& & $0.922 \pm 0.025$ & $0.057 \pm 0.006$ & $0.086 \pm 0.005$ \\
\midrule

\multirow{6}{*}{Beijing PM$_{10}$}
& Bayesian ridge
& \multirow{6}{*}{$0.203 \pm 0.028$}
& $0.970 \pm 0.003$ & $0.879 \pm 0.046$ & $1.260 \pm 0.127$ \\
& SCP
& & $0.953 \pm 0.007$ & $0.723 \pm 0.097$ & $1.239 \pm 0.157$ \\
& Asymmetric SCP
& & $0.953 \pm 0.007$ & $0.723 \pm 0.101$ & $1.239 \pm 0.158$ \\
& Adaptive CP
& & $0.945 \pm 0.000$ & $0.711 \pm 0.027$ & $1.059 \pm 0.075$ \\
& Time-weighted CP
& & $0.936 \pm 0.002$ & $0.684 \pm 0.047$ & $1.073 \pm 0.103$ \\
& ResCP
& & $0.948 \pm 0.007$ & $0.614 \pm 0.013$ & $0.966 \pm 0.097$ \\
\midrule

\multirow{6}{*}{Exchange returns}
& Bayesian ridge
& \multirow{6}{*}{$1.318 \pm 0.268$}
& $0.930 \pm 0.062$ & $4.106 \pm 0.124$ & $7.653 \pm 2.801$ \\
& SCP
& & $0.957 \pm 0.060$ & $5.416 \pm 0.814$ & $8.276 \pm 2.010$ \\
& Asymmetric SCP
& & $0.956 \pm 0.059$ & $5.415 \pm 0.823$ & $8.280 \pm 1.993$ \\
& Adaptive CP
& & $0.951 \pm 0.011$ & $4.780 \pm 1.606$ & $6.997 \pm 1.948$ \\
& Time-weighted CP
& & $0.946 \pm 0.003$ & $4.591 \pm 1.468$ & $6.541 \pm 1.548$ \\
& ResCP
& & $0.925 \pm 0.050$ & $4.168 \pm 0.865$ & $7.355 \pm 2.268$ \\
\bottomrule
\end{tabular}
\end{table*}

The repeated evaluation changes the emphasis relative to a single favorable test period. Adaptive CP is comparatively stable across forecast periods. Mean 95\% coverage ranges from 0.943 to 0.953, with small between-origin variation for ozone, Solar, and Beijing. Bayesian ridge is conservative for ozone and Beijing but undercovers on average for Solar and Exchange returns, consistent
with the block-to-block RMS changes in Table~\ref{tab:real_theory_diagnostics}. SCP and asymmetric SCP are nearly indistinguishable for these series, so asymmetry alone does not explain the larger differences observed for ResCP.
The Exchange results have much larger between-origin variation for every nonadaptive method, illustrating the practical importance of temporal shift.

\subsection{Forecast-horizon robustness}
\label{app:real_horizon}

Using the same fixed-origin split, we additionally evaluate direct forecast
horizons $H\in\{1,3,7\}$ to assess whether the one-step conclusions persist as
the prediction horizon increases. For $H>1$, every sequential method uses the
availability-safe delayed-update protocol described above: a residual enters
the online calibration state only after its target has been observed.

Under this protocol, Adaptive CP remains close to nominal coverage across
horizons. ResCP has relatively low Winkler scores for LA ozone and Beijing
PM$_{10}$, although its Exchange-return coverage is $0.894$ at $H=3$ and
$0.931$ at $H=7$. For Solar, uncertainty grows rapidly with the horizon;
ResCP coverage is $0.915$ at $H=3$ and $0.919$ at $H=7$, while ordinary SCP
also undercovers. ResCP hyperparameter selection considers only numerical
history windows that do not exceed the residual history available in the
chronological tuning prefix. Under this protocol, the LA-ozone $H=7$
configuration selects $T=0.25$ and $W=3800$. Thus good one-step performance
does not imply uniform multi-horizon calibration.

\begin{table}[t]
\centering
\caption{Complete availability-safe multi-horizon NGRC
uncertainty-quantification results at nominal 95\% coverage.}
\label{tab:real_horizon_full}

\begin{minipage}[t]{0.48\textwidth}
\centering
\textbf{LA ozone}

\vspace{2pt}
\small
\begin{tabular}{clrrr}
\toprule
$H$ & Method & Cov. & Width & Winkler \\
\midrule
1 & Bayesian ridge   & 0.993 & 2.047 & 2.083 \\
  & SCP              & 0.963 & 1.573 & 1.795 \\
  & Adaptive CP      & 0.954 & 1.485 & 1.784 \\
  & Time-weighted CP & 0.954 & 1.477 & 1.782 \\
  & ResCP            & 0.929 & 1.369 & 1.797 \\
\addlinespace[1pt]
3 & Bayesian ridge   & 0.997 & 2.669 & 2.691 \\
  & SCP              & 0.967 & 1.874 & 2.102 \\
  & Adaptive CP      & 0.957 & 1.778 & 2.083 \\
  & Time-weighted CP & 0.956 & 1.775 & 2.082 \\
  & ResCP            & 0.953 & 1.761 & 2.102 \\
\addlinespace[1pt]
7 & Bayesian ridge   & 0.996 & 2.726 & 2.752 \\
  & SCP              & 0.966 & 1.916 & 2.171 \\
  & Adaptive CP      & 0.956 & 1.816 & 2.151 \\
  & Time-weighted CP & 0.955 & 1.812 & 2.152 \\
  & ResCP            & 0.950 & 1.765 & 2.123 \\
\bottomrule
\end{tabular}
\end{minipage}
\hfill
\begin{minipage}[t]{0.48\textwidth}
\centering
\textbf{Solar}

\vspace{2pt}
\small
\begin{tabular}{clrrr}
\toprule
$H$ & Method & Cov. & Width & Winkler \\
\midrule
1 & Bayesian ridge   & 0.941 & 0.115 & 0.205 \\
  & SCP              & 0.940 & 0.115 & 0.205 \\
  & Adaptive CP      & 0.943 & 0.109 & 0.167 \\
  & Time-weighted CP & 0.857 & 0.113 & 0.176 \\
  & ResCP            & 0.961 & 0.058 & 0.075 \\
\addlinespace[1pt]
3 & Bayesian ridge   & 0.932 & 0.281 & 0.513 \\
  & SCP              & 0.923 & 0.267 & 0.520 \\
  & Adaptive CP      & 0.948 & 0.323 & 0.509 \\
  & Time-weighted CP & 0.949 & 0.326 & 0.509 \\
  & ResCP            & 0.915 & 0.168 & 0.256 \\
\addlinespace[1pt]
7 & Bayesian ridge   & 0.917 & 0.600 & 1.035 \\
  & SCP              & 0.903 & 0.560 & 1.067 \\
  & Adaptive CP      & 0.948 & 0.711 & 1.018 \\
  & Time-weighted CP & 0.947 & 0.698 & 1.014 \\
  & ResCP            & 0.919 & 0.417 & 0.578 \\
\bottomrule
\end{tabular}
\end{minipage}

\vspace{0.8em}

\begin{minipage}[t]{0.48\textwidth}
\centering
\textbf{Beijing PM$_{10}$}

\vspace{2pt}
\small
\begin{tabular}{clrrr}
\toprule
$H$ & Method & Cov. & Width & Winkler \\
\midrule
1 & Bayesian ridge   & 0.971 & 0.909 & 1.174 \\
  & SCP              & 0.960 & 0.776 & 1.138 \\
  & Adaptive CP      & 0.947 & 0.726 & 1.048 \\
  & Time-weighted CP & 0.936 & 0.677 & 1.030 \\
  & ResCP            & 0.954 & 0.656 & 0.933 \\
\addlinespace[1pt]
3 & Bayesian ridge   & 0.962 & 1.851 & 2.622 \\
  & SCP              & 0.961 & 1.856 & 2.645 \\
  & Adaptive CP      & 0.946 & 1.709 & 2.429 \\
  & Time-weighted CP & 0.920 & 1.515 & 2.459 \\
  & ResCP            & 0.951 & 1.496 & 2.208 \\
\addlinespace[1pt]
7 & Bayesian ridge   & 0.949 & 2.725 & 4.208 \\
  & SCP              & 0.961 & 3.095 & 4.288 \\
  & Adaptive CP      & 0.944 & 2.748 & 3.790 \\
  & Time-weighted CP & 0.938 & 2.614 & 3.871 \\
  & ResCP            & 0.953 & 2.436 & 3.322 \\
\bottomrule
\end{tabular}
\end{minipage}
\hfill
\begin{minipage}[t]{0.48\textwidth}
\centering
\textbf{Exchange returns}

\vspace{2pt}
\small
\begin{tabular}{clrrr}
\toprule
$H$ & Method & Cov. & Width & Winkler \\
\midrule
1 & Bayesian ridge   & 0.946 & 3.981 & 6.468 \\
  & SCP              & 0.983 & 5.407 & 6.997 \\
  & Adaptive CP      & 0.954 & 4.015 & 6.396 \\
  & Time-weighted CP & 0.939 & 4.281 & 6.242 \\
  & ResCP            & 0.934 & 3.793 & 6.640 \\
\addlinespace[1pt]
3 & Bayesian ridge   & 0.946 & 3.976 & 6.491 \\
  & SCP              & 0.983 & 5.423 & 6.998 \\
  & Adaptive CP      & 0.952 & 4.037 & 6.364 \\
  & Time-weighted CP & 0.938 & 4.274 & 6.830 \\
  & ResCP            & 0.894 & 3.335 & 7.202 \\
\addlinespace[1pt]
7 & Bayesian ridge   & 0.947 & 3.980 & 6.490 \\
  & SCP              & 0.983 & 5.423 & 6.999 \\
  & Adaptive CP      & 0.954 & 4.007 & 6.403 \\
  & Time-weighted CP & 0.937 & 4.282 & 6.879 \\
  & ResCP            & 0.931 & 3.565 & 6.585 \\
\bottomrule
\end{tabular}
\end{minipage}
\end{table}

\subsection{Calibration across nominal levels}
\label{app:real_calibration}
Figure~\ref{fig:real_calibration_all} compares empirical and nominal coverage
at nominal levels 80\%, 85\%, 90\%, 95\%, and 97.5\%. A method whose 95\% result
appears satisfactory can still show systematic miscalibration elsewhere on
the curve. No procedure is uniformly calibrated across datasets and levels.

\begin{figure*}[t]
  \centering
  \includegraphics[width=0.88\textwidth]{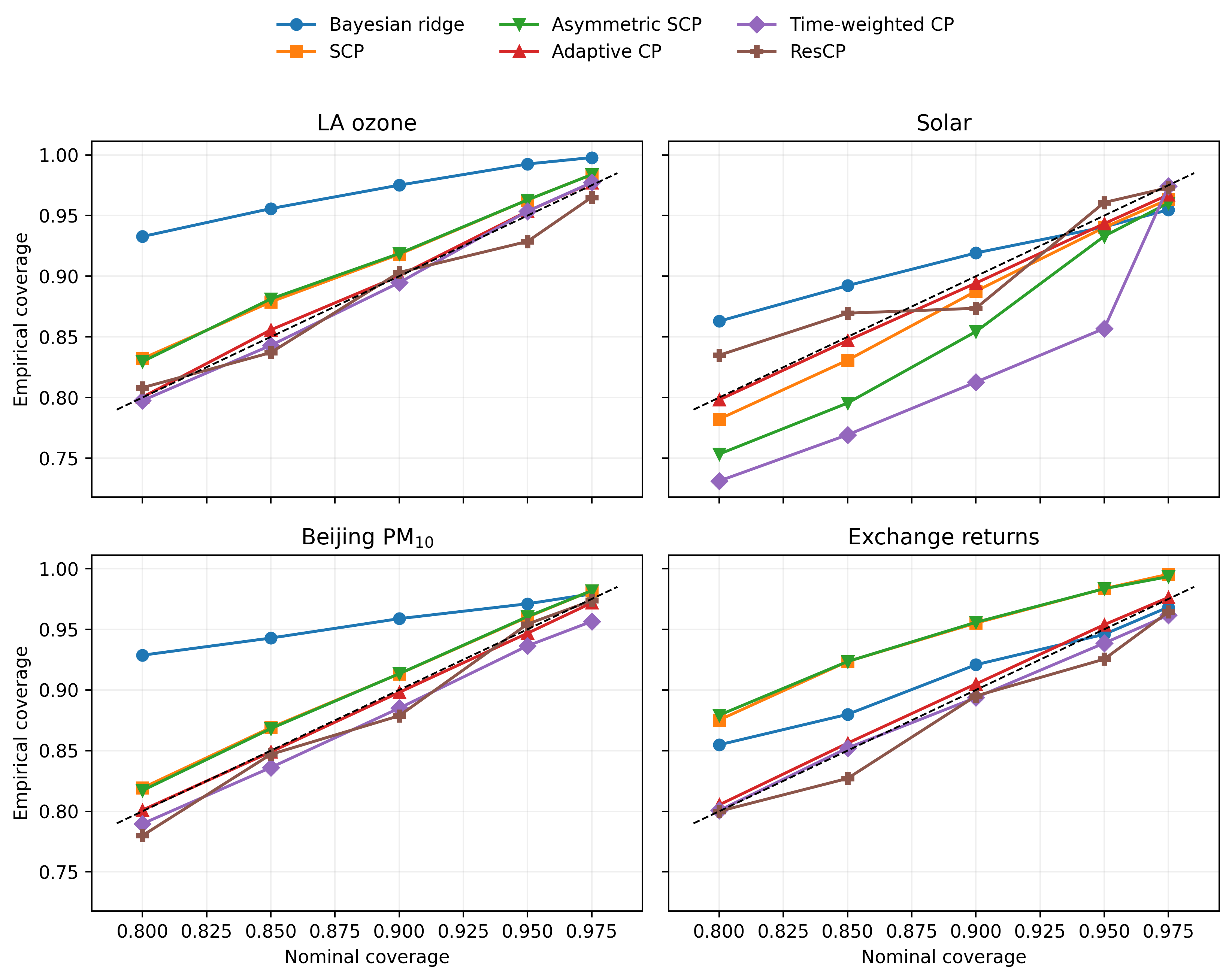}
  \caption{Fixed-origin empirical versus nominal coverage for all six NGRC
  UQ procedures. The dashed diagonal denotes exact marginal calibration.}
  \label{fig:real_calibration_all}
\end{figure*}

\subsection{Feature-Construction Robustness}
\label{app:real_feature_ablation}

We examine whether the main Bayesian--SCP conclusions depend strongly on the
default quadratic NGRC representation. The ablation compares linear lags,
quadratic NGRC features, three single-scale memory augmentations, and a
multi-scale memory representation.

Table~\ref{tab:real_feature_uq_ablation} reports point-prediction RMSE together
with coverage, interval width, and Winkler score for Bayesian ridge and
ordinary SCP. This analysis is intended as a robustness check rather than as a
search for an optimal feature representation.

\begin{table*}[t]
\centering
\scriptsize
\setlength{\tabcolsep}{4pt}
\renewcommand{\arraystretch}{0.95}
\caption{
Feature-construction robustness at $H=1$ and 95\% nominal coverage.
RMSE is reported once per feature construction because Bayesian ridge and
SCP use the same fitted point predictor.
}
\label{tab:real_feature_uq_ablation}
\begin{tabular}{llrccc ccc}
\toprule
& & & \multicolumn{3}{c}{Bayesian ridge}
& \multicolumn{3}{c}{SCP} \\
\cmidrule(lr){4-6}\cmidrule(lr){7-9}
Dataset & Features & RMSE
& Coverage & Width & Winkler
& Coverage & Width & Winkler \\
\midrule

\multirow{6}{*}{LA ozone}
& Linear lags        & 0.357 & 0.993 & 2.098 & 2.132 & 0.964 & 1.583 & 1.806 \\
& Quadratic NGRC     & 0.360 & 0.993 & 2.047 & 2.083 & 0.963 & 1.573 & 1.795 \\
& Short-memory NGRC  & 0.360 & 0.993 & 2.047 & 2.082 & 0.963 & 1.574 & 1.796 \\
& Medium-memory NGRC & 0.360 & 0.993 & 2.047 & 2.082 & 0.963 & 1.572 & 1.795 \\
& Long-memory NGRC   & 0.355 & 0.992 & 2.040 & 2.079 & 0.963 & 1.555 & 1.783 \\
& Multi-scale NGRC   & 0.354 & 0.992 & 2.036 & 2.076 & 0.962 & 1.558 & 1.783 \\
\midrule

\multirow{6}{*}{Solar}
& Linear lags        & 0.033 & 0.936 & 0.120 & 0.227 & 0.933 & 0.116 & 0.228 \\
& Quadratic NGRC     & 0.031 & 0.941 & 0.115 & 0.205 & 0.940 & 0.115 & 0.205 \\
& Short-memory NGRC  & 0.031 & 0.941 & 0.114 & 0.205 & 0.940 & 0.114 & 0.205 \\
& Medium-memory NGRC & 0.031 & 0.941 & 0.115 & 0.205 & 0.940 & 0.114 & 0.205 \\
& Long-memory NGRC   & 0.031 & 0.941 & 0.115 & 0.205 & 0.940 & 0.114 & 0.205 \\
& Multi-scale NGRC   & 0.031 & 0.941 & 0.114 & 0.205 & 0.940 & 0.114 & 0.205 \\
\midrule

\multirow{6}{*}{Beijing PM$_{10}$}
& Linear lags        & 0.179 & 0.973 & 0.931 & 1.185 & 0.958 & 0.751 & 1.126 \\
& Quadratic NGRC     & 0.182 & 0.971 & 0.909 & 1.174 & 0.960 & 0.776 & 1.138 \\
& Short-memory NGRC  & 0.182 & 0.971 & 0.909 & 1.174 & 0.960 & 0.775 & 1.138 \\
& Medium-memory NGRC & 0.182 & 0.971 & 0.909 & 1.174 & 0.960 & 0.773 & 1.137 \\
& Long-memory NGRC   & 0.182 & 0.971 & 0.909 & 1.174 & 0.960 & 0.772 & 1.137 \\
& Multi-scale NGRC   & 0.182 & 0.971 & 0.909 & 1.174 & 0.960 & 0.773 & 1.137 \\
\midrule

\multirow{6}{*}{Exchange returns}
& Linear lags        & 1.276 & 0.948 & 3.912 & 6.401 & 0.983 & 5.405 & 6.983 \\
& Quadratic NGRC     & 1.284 & 0.946 & 3.981 & 6.468 & 0.983 & 5.407 & 6.997 \\
& Short-memory NGRC  & 1.284 & 0.946 & 3.981 & 6.466 & 0.983 & 5.404 & 6.994 \\
& Medium-memory NGRC & 1.284 & 0.946 & 3.981 & 6.467 & 0.983 & 5.407 & 6.997 \\
& Long-memory NGRC   & 1.284 & 0.946 & 3.981 & 6.467 & 0.983 & 5.403 & 6.994 \\
& Multi-scale NGRC   & 1.283 & 0.946 & 3.980 & 6.466 & 0.983 & 5.405 & 6.994 \\
\bottomrule
\end{tabular}
\end{table*}

The qualitative Bayesian--SCP comparison is largely stable across these
representations. Changes in feature construction can alter point error and
interval metrics, but the differences among the uncertainty procedures are
not explained solely by fine tuning of the NGRC representation. Similar point
RMSE also does not imply identical interval behavior.

\subsection{ResCP Stochastic-Seed Sensitivity}
\label{app:real_rescp_seed_sensitivity}

The primary ResCP results use one prespecified seed for the random reservoir
and stochastic residual sampler in each dataset--horizon configuration. To
assess whether the qualitative conclusions depend on that computational
randomness, we repeat the final ResCP evaluation over five seeds while holding
the NGRC point forecasts, chronological split, selected temperature, and
selected history policy fixed. Thus this check isolates seed sensitivity rather
than combining computational randomness with hyperparameter reselection.

\begin{table*}[t]
\centering
\small
\caption{Sensitivity of ResCP to its stochastic seed. Entries are means
$\pm$ standard deviations over five seeds; the coverage range gives the
minimum and maximum empirical coverage across those seeds. Standard deviations
describe computational seed sensitivity, not sampling standard errors.}
\label{tab:real_rescp_seed_sensitivity}
\begin{tabular}{llrrrr}
\toprule
Dataset & Horizon & Coverage & Coverage range & Width & Winkler score \\
\midrule
LA ozone & $H=1$ & $0.933\pm0.003$ & $[0.929,0.936]$ & $1.375\pm0.006$ & $1.818\pm0.017$ \\
& $H=3$ & $0.954\pm0.001$ & $[0.953,0.955]$ & $1.755\pm0.012$ & $2.095\pm0.006$ \\
& $H=7$ & $0.952\pm0.001$ & $[0.950,0.953]$ & $1.772\pm0.008$ & $2.111\pm0.012$ \\
\midrule
Solar & $H=1$ & $0.960\pm0.003$ & $[0.955,0.963]$ & $0.059\pm0.001$ & $0.078\pm0.002$ \\
& $H=3$ & $0.928\pm0.010$ & $[0.915,0.937]$ & $0.183\pm0.012$ & $0.258\pm0.009$ \\
& $H=7$ & $0.928\pm0.009$ & $[0.919,0.936]$ & $0.419\pm0.008$ & $0.576\pm0.010$ \\
\midrule
Beijing PM$_{10}$ & $H=1$ & $0.954\pm0.004$ & $[0.947,0.958]$ & $0.648\pm0.024$ & $0.933\pm0.011$ \\
& $H=3$ & $0.954\pm0.005$ & $[0.948,0.960]$ & $1.532\pm0.062$ & $2.146\pm0.056$ \\
& $H=7$ & $0.957\pm0.003$ & $[0.953,0.961]$ & $2.525\pm0.081$ & $3.339\pm0.060$ \\
\midrule
Exchange returns & $H=1$ & $0.937\pm0.003$ & $[0.934,0.941]$ & $3.830\pm0.070$ & $6.654\pm0.037$ \\
& $H=3$ & $0.891\pm0.005$ & $[0.884,0.898]$ & $3.408\pm0.083$ & $7.233\pm0.213$ \\
& $H=7$ & $0.932\pm0.004$ & $[0.927,0.937]$ & $3.654\pm0.054$ & $6.587\pm0.100$ \\
\bottomrule
\end{tabular}
\end{table*}

Across the 12 dataset--horizon configurations, the largest across-seed coverage
standard deviation is $0.010$ for Solar at $H=3$. The largest relative standard
deviations of mean interval width and Winkler score are 6.8\% and 3.3\%,
respectively. Thus the qualitative comparisons are not driven by the particular
stochastic seed used in the primary results.

\subsection{Runtime and Computational Details}
\label{app:real_runtime}

We distinguish two runtime quantities throughout the real-data analysis.
UQ-only runtime measures method-specific calibration, hyperparameter
selection, sequential updating, and interval construction after the common
NGRC point predictor has been fitted. End-to-end runtime additionally
includes preprocessing, point-model fitting, and forecasting.

The runtimes reported for the primary UQ comparison are UQ-only measurements.
The architecture comparison instead reports end-to-end pipeline runtime.
These quantities should therefore not be compared directly.

For the multi-horizon experiments, wall-clock time is recorded separately at
each forecast horizon. Bayesian ridge and ordinary SCP require little
additional computation after fitting the common point predictor, whereas
adaptive CP, time-weighted CP, and ResCP require sequential processing and,
where applicable, pre-test hyperparameter selection.

\begin{table}[t]
\centering
\caption{Uncertainty-quantification runtime in seconds for forecast horizons
\(H=3\) and \(H=7\). Times include pre-test uncertainty-quantification
hyperparameter selection where applicable and exclude fitting of the common
NGRC point predictor.}
\label{tab:real_runtime_horizon}
\begin{tabular}{llrr}
\toprule
Dataset & Method & \(H=3\) & \(H=7\) \\
\midrule
LA ozone
& Bayesian ridge   & 0.0125 & 0.0107 \\
& SCP              & 0.0003 & 0.0007 \\
& Adaptive CP      & 6.4812 & 6.4685 \\
& Time-weighted CP & 9.5753 & 9.2434 \\
& ResCP            & 22.2484 & 31.4085 \\
\midrule
Solar
& Bayesian ridge   & 0.1008 & 0.0988 \\
& SCP              & 0.0012 & 0.0012 \\
& Adaptive CP      & 71.8993 & 73.2512 \\
& Time-weighted CP & 99.6658 & 100.6878 \\
& ResCP            & 126.1974 & 122.0124 \\
\midrule
Beijing PM$_{10}$
& Bayesian ridge   & 0.0267 & 0.0249 \\
& SCP              & 0.0008 & 0.0007 \\
& Adaptive CP      & 32.0766 & 31.9235 \\
& Time-weighted CP & 44.4136 & 45.1172 \\
& ResCP            & 131.2923 & 131.9915 \\
\midrule
Exchange returns
& Bayesian ridge   & 0.0041 & 0.0015 \\
& SCP              & 0.0002 & 0.0001 \\
& Adaptive CP      & 1.4273 & 1.4068 \\
& Time-weighted CP & 1.8807 & 1.8806 \\
& ResCP            & 5.3240 & 5.4499 \\
\bottomrule
\end{tabular}
\end{table}

\end{document}